\documentclass[11pt]{article}

\usepackage{microtype} 
\usepackage{graphicx}
\usepackage{booktabs} 
\usepackage{tikz}
\usepackage{hyperref,graphicx,fullpage,amsmath,amsfonts,subcaption,amssymb,bm,url,breakurl,epsfig,epsf,color,MnSymbol,mathbbol,fmtcount,semtrans,caption,multirow,comment}

\usepackage{color}
\usepackage[utf8]{inputenc} 
\usepackage[T1]{fontenc}    
\usepackage{hyperref}       
\usepackage{url}            
\usepackage{booktabs}       
\usepackage{amsfonts}       
\usepackage{nicefrac}       
\usepackage{microtype}      
\usepackage{relsize}
\usepackage{enumitem}
\usepackage{hyperref,graphicx,amsmath,amsfonts,amssymb,bm,url,breakurl,epsfig,epsf,color,MnSymbol,mathbbol,fmtcount,semtrans,caption,subcaption,multirow,comment, boldline}
\usepackage[noend]{algpseudocode}
\usepackage{wrapfig}
\usepackage{amssymb}

\usepackage[utf8]{inputenc} 
\usepackage[T1]{fontenc}    
\usepackage{url}            
\usepackage{booktabs}       
\usepackage{amsfonts}       
\usepackage{nicefrac}       
\usepackage{microtype}      

\makeatletter
\providecommand*{\boxast}{%
  \mathbin{
    \mathpalette\@boxit{*}%
  }%
}
\newcommand*{\@boxit}[2]{%
  \sbox0{$\m@th#1\Box$}%
  \ifx#1\displaystyle \ht0=\dimexpr\ht0+.05ex\relax \fi
  \ifx#1\textstyle \ht0=\dimexpr\ht0+.05ex\relax \fi
  \ifx#1\scriptstyle \ht0=\dimexpr\ht0+.04ex\relax \fi
  \ifx#1\scriptscriptstyle \ht0=\dimexpr\ht0+.065ex\relax \fi
  \sbox2{$#1\vcenter{}$}
  \rlap{%
    \hbox to \wd0{%
      \hfill
      \raisebox{%
        \dimexpr.5\dimexpr\ht0+\dp0\relax-\ht2\relax
      }{$\m@th#1#2$}%
      \hfill
    }%
  }%
  \Box
}
\makeatother

  \makeatletter
\def\BState{\State\hskip-\ALG@thistlm}
\makeatother

  \usepackage{mathtools}

\usepackage{titlesec}

\usepackage{tikz}
\usepackage{pgfplots}
\usetikzlibrary{pgfplots.groupplots}

\titleformat{\paragraph}
{\normalfont\normalsize\bfseries}{\theparagraph}{1em}{}
\titlespacing*{\paragraph}
{0pt}{3.25ex plus 1ex minus .2ex}{1.5ex plus .2ex}

\usepackage{movie15}

\usepackage{caption}
\usepackage[bottom,hang,flushmargin]{footmisc} 

\newcommand{\tsn}[1]{{\left\vert\kern-0.25ex\left\vert\kern-0.25ex\left\vert #1 
    \right\vert\kern-0.25ex\right\vert\kern-0.25ex\right\vert}}

\definecolor{darkred}{RGB}{150,0,0}
\definecolor{darkgreen}{RGB}{0,150,0}
\definecolor{darkblue}{RGB}{0,0,200}
\hypersetup{colorlinks=true, linkcolor=darkred, citecolor=darkgreen, urlcolor=darkblue}

\newtheorem{theorem}{Theorem}[section]

\newtheorem{assumption}{Assumption}

\newtheorem{lemma}[theorem]{Lemma}
\newtheorem{corollary}[theorem]{Corollary}
\newtheorem{propo}[theorem]{Proposition}
\newtheorem{definition}[theorem]{Definition}

\newtheorem{remark}[subsection]{Remark}

\newcommand{\eps}{\varepsilon}

\newcommand{\beq}{\begin{equation}}

\newcommand{\eeq}{\end{equation}}

\newcommand{\nn}{\nonumber}

\newcommand{\diag}[1]{\text{diag}(#1)}

\newcommand{\bmu}{{\boldsymbol{{\mu}}}}

\newcommand{\Iden}{{\mtx{I}}}

\newcommand{\z}{{\vct{z}}}

\newcommand{\opnorm}[1]{\left\|#1\right\|}

\newcommand{\onenorm}[1]{\left\|#1\right\|_{\ell_1}}
\newcommand{\twonorm}[1]{\left\|#1\right\|_{\ell_2}}
\newcommand{\Sigmanorm}[1]{\left\|#1\right\|_{\mtx{\Sigma}}}
\newcommand{\pnorm}[1]{\left\|#1\right\|_{\ell_p}}
\newcommand{\qnorm}[1]{\left\|#1\right\|_{\ell_q}}

\newcommand{\infnorm}[1]{\left\|#1\right\|_{\ell_\infty}}

\newcommand{\abs}[1]{\left|#1\right|}

\newcommand{\x}{\vct{x}}

\definecolor{emmanuel}{RGB}{255,127,0}

\newcommand{\R}{\mathbb{R}}

\newcommand{\<}{\langle}
\renewcommand{\>}{\rangle}
\newcommand{\Var}{\textrm{Var}}
\newcommand{\sgn}[1]{\textrm{sgn}(#1)}

\newcommand{\E}{\operatorname{\mathbb{E}}}

\newcommand{\vct}[1]{\bm{#1}}
\newcommand{\mtx}[1]{\bm{#1}}

\numberwithin{equation}{section} 

\def \endprf{\hfill {\vrule height6pt width6pt depth0pt}\medskip}
\newenvironment{proof}{\noindent {\bf Proof} }{\endprf\par}

\newcommand{\hth}{{\widehat{\boldsymbol{\theta}}}}
\newcommand{\tth}{{\widetilde{\boldsymbol{\theta}}}}
\newcommand{\bth}{{\boldsymbol{\theta}}}
\newcommand\cL{\mathcal{L}}

\newcommand\reals{\mathbb{R}}
\newcommand\bSigma{\boldsymbol{\Sigma}}
\newcommand\normal{{\sf N}}
\newcommand\bdelta{\boldsymbol{\delta}}
\newcommand\sign{{\rm sign}}
\newcommand\sT{{\sf T}}

\newcommand\bv{\mtx{v}}

\newcommand\bz{\boldsymbol{z}}

\def\de{{\rm d}}

\def\bu{\boldsymbol{u}}

\def\bw{\boldsymbol{w}}
\def\bX{\boldsymbol{X}}

\def\tbth{\tilde{\boldsymbol{\theta}}}

\def\ind{\mathbb{1}}
\def\cS{\mathcal{S}}
\def\prob{\mathbb{P}}

\def\tbth{\vct{\theta}_\perp}
\def\reals{\mathbb{R}}
\def\bx{{\vct{x}}}

\def\ST{{\sf ST}}

\def\SR{{\sf SA}}
\def\AR{{\sf RA}}

\def\tw{\tilde{\vct{w}}}

\def\proj{{\sf P}}
\def\pproj{{\sf P}^\perp}

\def\sE{{\sf{E}}}
\def\sF{{\sf{F}}}
\def\naturals{\mathbb{N}}
\def\tmu{\widetilde{\vct{\mu}}}

\def\erf{{\rm erf}}

\def\EJ{\mathcal{J}}

\def\teta{\tilde{\eta}}
\def\LDA{\widehat{\bmu}^{{\rm LDA}}}
\def\bW{\mtx{W}}


\author{Adel Javanmard\thanks{Data Science and Operations Department, Marshall School of Business, University of Southern California}   \, and \,  Mahdi Soltanolkotabi\thanks{Ming Hsieh Department of Electrical and Computer Engineering, University of Southern California} }

\begin{document}
\title{Precise Statistical Analysis of Classification\\ Accuracies for Adversarial Training}
\maketitle

\begin{abstract}
Despite the wide empirical success of modern machine learning algorithms and models in a multitude of applications, they are known to be highly susceptible to seemingly small indiscernible perturbations to the input data known as \emph{adversarial attacks}. A variety of recent adversarial training procedures have been proposed to remedy this issue. Despite the success of such procedures at increasing accuracy on adversarially perturbed inputs or \emph{robust accuracy}, these techniques often reduce accuracy on natural unperturbed inputs or \emph{standard accuracy}. Complicating matters further, the effect and trend of adversarial training procedures on standard and robust accuracy is rather counter intuitive and radically dependent on a variety of factors including the perceived form of the perturbation during training, size/quality of data, model overparameterization, etc. In this paper we focus on binary classification problems where the data is generated according to the mixture of two Gaussians with general anisotropic covariance matrices and derive a precise characterization of the standard and robust accuracy for a class of minimax adversarially trained models. We consider a general norm-based adversarial model, where the adversary can add perturbations of bounded $\ell_p$ norm to each input data, for an arbitrary $p\ge 1$. 
Our comprehensive analysis allows us to theoretically explain several intriguing empirical phenomena and provide a precise understanding of the role of different problem parameters on standard and robust accuracies.

\end{abstract}

\section{Introduction}
Over the past decade there has been a tremendous increase in the use of machine learning models, and deep learning in particular, in a myriad of domains spanning computer vision and speech recognition, to robotics, healthcare and e-commerce. Despite wide empirical success in these and related domains, these modern learning models are known to be highly fragile and susceptible to \emph{adversarial attacks}; even seemingly small imperceptible perturbations to the input data can significantly compromise their performance. As machine learning systems are increasingly being used in applications involving human subjects including healthcare and autonomous driving, such vulnerability can have catastrophic consequences. As a result there has been significant research over the past few years focused on proposing various \emph{adversarial training} methods aimed at mitigating the effect of adversarial perturbations~\cite{DBLP:journals/corr/GoodfellowSS14, kurakin2016adversarial, DBLP:conf/iclr/MadryMSTV18, DBLP:conf/iclr/RaghunathanSL18, DBLP:conf/icml/WongK18}.



While adversarial training procedures have been successful in making machine learning models robust to adversarial attacks, their full effect on machine learning systems is not understood. Indeed, adversarial training procedures often behave in mysterious and somewhat counter intuitive ways. For instance, while they improve performance on adversarially perturbed inputs, this benefit often comes at the cost of decreasing accuracy on natural unperturbed inputs. This suggests that the two performance measures,  \emph{robust accuracy} --the accuracy on adversarially
perturbed inputs--  and the \emph{standard accuracy} --accuracy on benign unperturbed inputs-- may be fundamentally at conflict. Even more surprising, the performance of adversarial training procedure varies significantly in different settings. For instance, while adversarial trained models yield lower standard accuracy in comparison with non-adversarially trained counterparts, this behavior is completely reversed when there are very few training data with the standard accuracy of adversarially trained models outperforming that of non-adversarial models ~\cite{tsipras2018robustness}. We refer the reader to Section \ref{related} for a through discussion of recent empirical results that demonstrate how a variety of factors such as the adversary's power, the size of training data, and model over-parameterization affect the performance of adversarially trained models.

To clearly demonstrate the surprising and counterintuitive behavior of adversarially trained models, we plot the behavior of such an approach in Figure \ref{fig:SA-RA-pinf-intro}. We consider a simple binary classification problem with the data generated according to a mixture of two isotropic Gaussians and depict the performance of a commonly used adversarial training procedure. In particular, in this figure, we plot the standard and robust accuracy of an adversarially trained linear classifier for different values of the adversary's perceived power (measured in $\ell_\infty$ perturbations) and different sampling ratios (size of the training data divided by the number of parameters denoted by $\delta$). We would like to highlight the highly non-trivial behavior of the standard and robust accuracy curves with respect to the adversary's power and the sampling ratio. For instance, the standard accuracy first decreases, then increases and again decreases as a function of the adversary's power. Furthermore, the exact nature of this curve is highly reliant on the sampling ratio $\delta$. Similarly, for robust accuracy, we first observe a decreasing trend for all $\delta$, but after some threshold depending on $\delta$, robust accuracy increases and then decreases or stays constant. Even more surprising, as we will see in the forth-coming sections the behavior of these curves vary drastically for different forms of $\ell_p$ perturbations. This simple experiment clearly demonstrates the importance of having a precise theory for characterizing the rather nuanced performance of adversarial training procedures and demystify their behavior. Developing such a precise theoretical analysis is exactly the goal of this paper. Indeed, the solid curves in Figure \ref{fig:SA-RA-pinf-intro} are based on our theoretical predictions!

\begin{figure}[]
\centering
\begin{minipage}{.485\textwidth}
  \centering
    \includegraphics[scale=0.7]{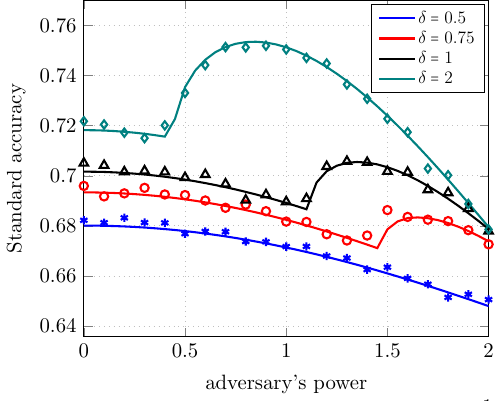}
   \caption*{(a) Standard accuracy}
\end{minipage}
\begin{minipage}{.485\textwidth}
  \centering
  \includegraphics[scale=0.7]{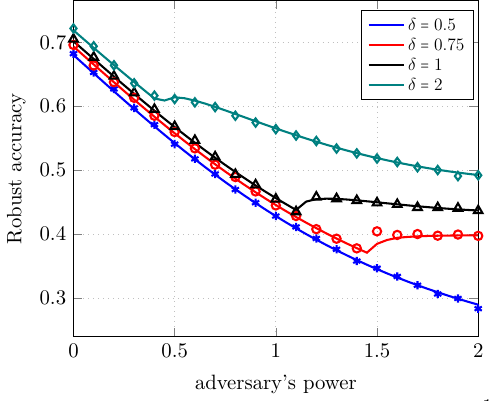}
   \caption*{(b) Robust accuracy}
\end{minipage}
\caption{Depiction of standard and robust accuracies as a function of the adversary's  power with $\ell_\infty$ ($p=\infty$) perturbation for different values of $\delta$ (ratio of the size of the training data to the number of parameters in the model). Solid curves are theoretical predictions and dots are the empirical results. We refer to Figure~\ref{fig:SA-RA-pinf} and Section~\ref{secinf} for further details.}
\label{fig:SA-RA-pinf-intro}
\end{figure}

\subsection{Contributions}
In this paper we focus on binary classification problems where the data is generated according to the mixture of two Gaussians with general anisotropic covariance matrices and derive a precise characterization of the standard and robust accuracy for a class of minimax adversarially trained models. We consider a general norm-based adversarial model, where the adversary can add perturbations of bounded $\ell_p$ norm to each input data, for an arbitrary $p\ge 1$. We would like to emphasize that our theory provides a \emph{precise characterization} of the performance of this class of adversarially trained models, rather than just upper bounds on the standard and robust accuracies.
Our analysis for such a broad setting allows us to capture several intriguing phenomena that we discuss next.

We show and theoretically prove an interesting phase transition phenomena holds for adversarial classification applied to the Gaussian mixture model. Specifically, we characterize a threshold $\delta_*$ for the ratio of size of training data to feature dimension, $\delta$ so that when $\delta< \delta_*$, the data is \emph{robustly separable} with high probability, and for $\delta>\delta_*$ it is non-separable, with high probability. Here, \emph{robust separability} is a generalization of the classical linear separability condition for data and roughly speaking means that there is  a linear separator that correctly separates the two label classes with a positive margin that depends on the adversary's power. We precisely characterize the threshold $\delta_*$ in terms of various problem parameters including the mean and covariance of the mixture components, the adversary's power, and the $\ell_p$ perturbation norm. Interestingly, $\delta_*$ is related to the spherical width of a set defined in terms of the dual $\ell_q$ norm ($1/p+1/q =1$) conforming with classical notions of prior knowledge and complexity used in the compressive sensing literature.

Our precise theoretical characterization of standard and robust accuracies provides a precise understanding of the role that different problem parameters such as size/quality of the training data, feature covariates and means, 
model overparameterization $(1/\delta)$, and the adversary's perceived power have during training on these performance measures. Surprisingly, our analysis reveals that the effects of these factors very much depend on the choice of perturbations norm $\ell_p$. For example, in the robustly separable regime, we observe that for $p=2$ adversarial training has no effect on standard accuracy, while for $p= 1$ and $p=\infty$ it hurts the standard accuracy. In the non-separable regime, we observe that for $p=2$ adversarial training helps with improving the standard accuracy. However, for $p=\infty$ the adversarial training first improves the standard accuracy but as the training procedure hedges against stronger adversary, after some threshold on the adversary's power, we start to see a decrease in the standard accuracy of the resulting model. Interestingly, this threshold on the adversary's power varies with model overparameterization.  

Lastly, a key ingredient of our analysis is a powerful extension of Gordon's Gaussian process inequality \cite{gordon1988milman} known as the Convex Gaussian Minimax Theorem (CGMT) developed in \cite{thrampoulidis2015regularized} and further extended in \cite{thrampoulidis2018precise,deng2019model} for various learning settings. Using this technique we provide a precise prediction of the performance of adversarial training in terms of the optimal solutions to a convex-concave problem with a small number of scalar variables that can be easily solved by a low-dimensional gradient descent/ascent rather fast and accurately. In addition, this low-dimensional optimization problem can be significantly simplified for special cases of $p$ (see Section~\ref{specialcase} for details). While  CGMT has been used to study the behavior of  regularized M-estimators, using this framework for the broad class of minimax adversarially trained models studied in this paper (including general anisotropic covariance matrices and general choice of $\ell_p$ norm for adversarial perturbations)
poses significant technical challenges. Specifically, the intrinsic differences between $\ell_p$ geometries and the interaction between the class means the feature covariance matrix in the model requires a rather intricate and technical analysis.  


\subsection{Related work}
\label{related}
We briefly discuss the related literature along two lines.
\vspace{-0.1cm}

\noindent{\bf Other models of adversarial perturbations.} 
Another popular model for adversarial attacks on the models is the so-called distribution shifts, wherein the adversary can shift
the test data distribution, making it different from the
training distribution. The adversary is assumed to have limited manipulative power in terms of the Wasserstein distance between the test and the training distributions~\cite{staib2017distributionally,pydi2020adversarial,Mehrabi-fund}. The articles~\cite{bartl2020robust,Mehrabi-fund} study the robust loss $L(\bth;\eps) = \sup_{\nu\in B_{\eps}(\mu)} \E_{\nu}[\ell(\bz,\bth)]$, where $B_\eps(\mu)$ is the $\eps$ ball around $\mu$ in the Wasserstein ($W_p$) distance for some $p\in[1,\infty)$, and the data $\bz = (\bx,y)\sim \mu$. A first order approximation of the robust loss $L(\bth;\eps)$ is given for small $\eps$, in terms of a variation measure of the original loss $\ell$. Such characterization is used in~\cite{Mehrabi-fund} to investigate the tradeoff between the standard and robust accuracies for various learning problems. Note that these work are focused on the population loss ($n\to\infty$, with $d$ fixed). In comparison, in this paper we study norm bounded adversarial perturbations and work with empirical loss in asymptotic regime ($n,d\to\infty$, with $n/d = \delta$ fixed).    

In adversarial training it is assumed that the modeler has access to clean (unperturbed) data and strives to construct a model that is resilient to potential adversarial perturbations of the test data. The article~\cite{lai2020adversarial} considers a different adversarial setup in which an attacker can
observe and modify all training data samples in an
adversarial manner so as to maximize the estimation error caused
by his attack. This work introduces the notion of adversarial influence function (AIF) to quantify the sensitivity of estimators to such adversarial attacks, and further derive the optimal estimator, among a certain class of estimator, that minimizes AIF.

\smallskip

\noindent{\bf Standard accuracy and robust accuracy tradeoffs.} 
 Several recent papers contain empirical results suggesting a potential trade-off between standard accuracy and robust accuracy. A few papers have started to shed light on the theoretical foundations of such tradeoffs~\cite{DBLP:conf/iclr/MadryMSTV18, DBLP:conf/nips/SchmidtSTTM18,tsipras2018robustness,raghunathan2019adversarial, DBLP:conf/icml/ZhangYJXGJ19, pmlr-v125-javanmard20a, min2020curious,dobriban2020provable} often focusing on very specific models or settings. However, a comprehensive quantitative understanding of such tradeoffs is largely underdeveloped. 


A central question we wish to address in this paper is whether there exists a fundamental conflict between robust accuracy and standard accuracy. We briefly mention a few papers that take a step towards addressing this question. In \cite{tsipras2018robustness,DBLP:conf/icml/ZhangYJXGJ19}, the authors provide examples of learning problems where no predictor can achieve both optimal standard accuracy and robust accuracy in the infinite data limit, pointing to such fundamental tradeoff. By contrast, \cite{raghunathan2019adversarial} provides examples where there is no such tradeoff in the infinite data limit, in the sense that the optimal predictor performs well on both objectives, however a tradeoff is still observed with finite data. Despite this interesting progress a quantitive understanding of fundamental and algorithmic tradeoffs between standard and robust accuracies and how they are affected by various factors, such as overparameterization, adversary's power and the data model is still missing. Such a result requires novel perspectives and analytical tools to precisely characterize the behavior of robust and standard accuracies, which is one of the motivating factors behind our current paper.

More closely related to this paper, in~\cite{pmlr-v125-javanmard20a} the current authors used the convex Gaussian minimax framework to provide a precise characterization of standard and robust accuracies for linear regression, studying the fundamental conflict between these objectives along with algorithmic tradeoffs for specific minimax estimators. For classification problems, a recent paper~\cite{dobriban2020provable} focuses on characterizing the optimal $\ell_2$ and $\ell_\infty$ robust linear classifiers assuming access to the class means. This paper also studies some tradeoffs between standard and robust accuracies by contrasting this optimal robust classifier with the 
 Bayes optimal classifier in a non-adversarial setting. This paper however does not directly study the tradeoffs of adversarial training procedures except for linear losses. A related publication~\cite{min2020curious} studies the generalization property of an adversarially trained model for classification on a Gaussian mixture model with a diagonal covariance matrix and a linear loss. In this setting, this work discusses the different effects that more training data can have on generalization based on the strength of the adversary. Using a linear loss in the above two classification papers is convenient as in this case the adversarially trained model admits a simple closed form representation. We also note that these two papers do not seem to focus on the high-dimensional regime where the number of training data grow in proportion to the number of parameters. In contrast, in this paper we focus on developing a comprehensive theory that provides a precise characterization of standard and robust accuracies and their tradeoffs in the high dimensional regime for a broad class of loss functions and covariance matrices. Such a comprehensive analysis allows us to better understand the role of the loss function in adversarial training. Indeed, as we demonstrate, the behavior of standard and robust accuracy for nonlinear loss functions can be very different from linear losses. We also note that such a theoretical result requires much more intricate techniques as the adversarially trained model does not admit a simple closed form. Finally, we would like to note while in this paper we provide a precise understanding of the tradeoffs between standard and robust accuracies for commonly used adversarial training algorithms our work still does not address two tantalizing open questions: What is the optimal standard-robust accuracy tradeoff for a fixed ratio of sample size to dimension? Are there adversarial training approaches that achieve the optimal tradeoff between standard and robust accuracies universally over the range of adversary's power.

\section{Problem formulation}\label{sec:formulation}
In this section we discuss the problem setting and formulation of this paper in greater detail. After adopting some notations, we describe the adversarial training for binary classification in Section \ref{sec:adversarial}. Next, we discuss the data model and asymptotic setting studied in this paper in Section \ref{sec:data}. Finally, in Section \ref{perf} we formally define the standard and robust classification accuracies in this model. 
\smallskip
 
\noindent{\bf Notations.} For a vector $\bv\in\reals^d$, we write $\pnorm{\bv}$ for the standard 
 $\ell_p$ norm of $\bv$, i.e., $\pnorm{\bv} = (\sum_i |v_i|^p)^{1/p}$. For a matrix $\bSigma$, $\|\bSigma\|$ indicates the spectrum norm of $\bSigma$. Throughout, we say a probabilistic event holds `with high probability', when its probability converges to one as $n\to\infty$. In addition, for a sequence of random variables $\{X_n\}_{n\in\naturals}$ and a constant $c$ (independent of $n$) we write $\lim_{n\to\infty} X_n = c$, `in probability' if $\forall \eps>0$ we have $\lim_{n\to\infty} \prob(|X_n-c|>\eps) = 0$.

\subsection{Adversarial training for binary classification} \label{sec:adversarial}
 In binary classification we have access to a training data set of $n$ input-output pairs $\{(\vct{x}_i,y_i)\}_{i=1}^n$ with $\vct{x}_i\in\R^d$ representing the input features and $y_i\in\{-1,+1\}$ representing the binary class label associated to each data point. Throughout we assume the data points $(\vct{x}_i, y_i)$ are generated i.i.d.~according to a distribution $\mathbb{P}$. To find a classifier that predicts the labels, one typically fits a function $f_{\vct{\theta}}$, parameterized by $\vct{\theta}\in\R^d$ to the training data via empirical risk minimization. In this paper we focus on linear classifiers of the form $f_{\vct{\theta}}(\vct{x})=\langle \vct{x},\vct{\theta}\rangle$ in which case the training problem takes the form 
 \begin{align}\label{eq:EMP2}
\widehat{\bth} :=\arg\min_{\bth\in \reals^d} \frac{1}{n}\sum_{i=1}^n \ell(y_if_{\vct{\theta}}(\vct{x}_i))= \arg\min_{\bth^\in \reals^d} \frac{1}{n}\sum_{i=1}^n \ell(y_i\<\x_i,\bth\>)\,.
\end{align}
Here, $\ell$ is a loss and $\ell(y_i\<\x_i,\bth\>)$ approximately measuring the missclassification between the labels $y_i$ and the output of the model $\<\x_i,\bth\>$. Some common choices include logistic loss $ \ell(t) = \log(1+e^{-t})$, exponential loss $ \ell(t) = e^{-t}$, and hinge loss $\ell(t) = \max\left(0,1-t\right)$. Once the parameter $\widehat{\bth}$ is estimated one can find the predicted label by simply calculating the sign of the model output $\widehat{y}=\sgn{f_{\widehat{\bth}}(\vct{x})}=\sgn{\<\vct{x},\widehat{\bth}\>}$.

Despite the widespread of empirical risk minimizers in supervised learning, these estimators are known to be highly vulnerable to even minute perturbations in the input features $\x_i$. In particular, it is known that even small, norm-bounded perturbations to the features that are imperceptible to the human eye, can lead to surprising miss-classification errors. These observations have spurred a surge of interest in adversarial training where the goal is to learn models that are robust against such adversarial perturbation. In this paper we focus on an adversarial training approach that is based on using a robust minimax loss \cite{tsipras2018robustness,DBLP:conf/iclr/MadryMSTV18}. In our linear binary classification setting the robust minimax estimator takes the form
\begin{align}\label{eq:minimaxEst}
\hth^\eps := \arg\min_{\bth\in \reals^d} \frac{1}{n} \sum_{i=1}^n \underset{\pnorm{\bdelta_i}\le\eps}{\max} \ell\left(y_i \<\x_i+\bdelta_i,\bth\>\right).  
\end{align}
The main intuition behind such an estimator is that although the learner has access to unperturbed training data, instead of fitting to that data she imitates potential adversarial perturbations to test data in the training data and aims to learn a model that performs well in the presence of such perturbations. One can also view this adversarial training approach as an implicit smoothing that tries to fit the same label $y_i$ to all the features in the $\eps$-neighborhood of $\x_i$ simultaneously. 

In this paper we focus on convex and decreasing losses such as the aforementioned logistic, exponential, and hinge losses. In such cases the inner maximization in~\eqref{eq:minimaxEst} can be solved in closed form. In particular, the worst perturbation $\bdelta_i$ in terms of loss value is given by
$\bdelta^*_i = \arg\min\{y_i\<\bdelta_i,\vct{\theta}\>:\, \pnorm{\bdelta_i}\le \eps\}$, which by using Holder's inequality results in $y_i\<\bdelta^*_i,\bth\> = -\eps\qnorm{\bth}$. Therefore the adversarially trained model $\hth^\eps$ can be equivalently written as
\begin{align}\label{eq:inner}
\hth^\eps := \arg\min_{\bth\in \reals^d} \frac{1}{n} \sum_{i=1}^n  \ell\left(y_i \<\x_i,\bth\> - \eps \qnorm{\bth}\right) \,. 
\end{align} 

\subsection{Data model and asymptotic setting}
\label{sec:data}
We consider supervised binary classification under a Gaussian Mixture data Model (GMM). Concretely, each data point belongs to one of two classes $\{\pm 1\}$ with corresponding probabilities $\pi_+$, $\pi_-$, so that $\pi_+ + \pi_- = 1$. Given the label $y_i\in \{-1,+1\}$ for data point $i$, the associated input/feature vectors $\bx_i\in \reals^d$ are generated independently according to the distribution $\x_i\sim \normal(y_i\vct{\mu}, \mtx{\Sigma})$, conditioned on $y_i$, where $\vct{\mu}\in \reals^d$ and $\mtx{\Sigma}\in \reals^{d\times d}$. In other words the mean of feature vectors are $\pm \vct{\mu}$ depending on its class, and $\mtx{\Sigma}$ is the covariance of features. We depict this mixture model in Figure \ref{GMMfig}.
\begin{figure}
\begin{center}
\includegraphics[scale=0.45]{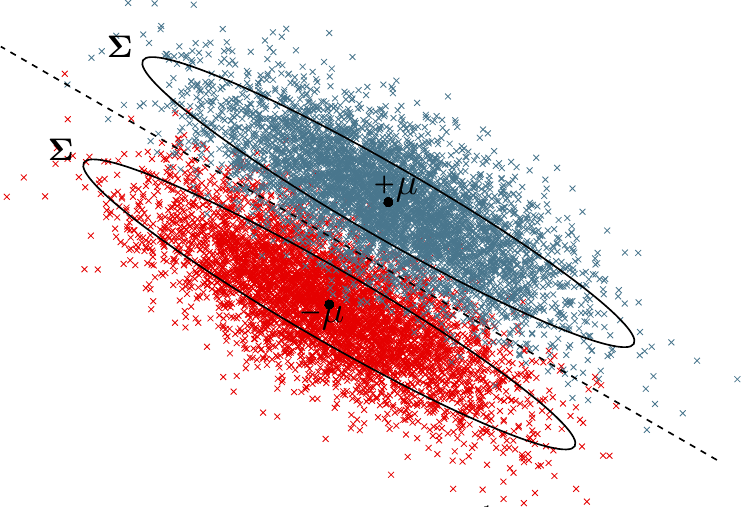}
\end{center}
\caption{Depiction of the Mixture of Gaussian data model.}
\label{GMMfig}
\end{figure}
We next describe the asymptotic regime of interest and our assumptions in this paper.
\begin{assumption}[Asymptotic Setting]\label{ass:asymptotic} We focus on the following asymptotic regime:
\begin{itemize}
\item[(a)] (Scaling of dimensions) $n\to \infty$ and $\frac{n}{d}\to\delta\in(0,\infty)$.
\item[(b)] (Scaling of signal to noise ratio) We have $C_{\min} \le \frac{\twonorm{\vct{\mu}}}{\opnorm{\mtx{\Sigma}}}\le C_{\max}$ for some positive constants $C_{\min}$ and $C_{\max}$, which are independent of $n$ and $d$.
\item[(c)] (Scaling of adversary's power) We have $\eps = \eps_0 \pnorm{\vct{\mu}}$ for a constant $\eps_0\ge0$ which we refer to as adversary's normalized power.
\end{itemize}
\end{assumption}
Assumption \ref{ass:asymptotic} (a) details our high-dimensional regime where the size of the training data $n$ and the dimension of the features $d$ grow proportionally with their ratio fixed at $\delta$. We would like to note that while we focus on this asymptotic regime our theoretical technique can also demonstrate very accurate concentration around this asymptotic behavior. Assumption \ref{ass:asymptotic} (b) demonstrates the scaling of the signal to noise ratio and ensures that the distance between the centers of the two components $2\twonorm{\vct{\mu}}$ (`signal') is comparable to the projection of noise in any direction (noise). Finally, Assumption \ref{ass:asymptotic} details our scaling of the adversary's power. This scaling is justified as if the adversary could perturb data points $\x_i$ by $2\vct{\mu}$, she can flip the label of every data point, so that the leaner cannot do better than random guessing. Since the perturbations can be chosen arbitrary from an $\ell_p$ ball of radius $\eps$, we require $\eps$ to be comparable to $\pnorm{\vct{\mu}}$.

\subsection{Standard and robust accuracies}
\label{perf}
Our goal is this paper is to precisely characterize performance of the estimator $\hth^\eps$ in terms of two accuracies and understand the interplay between them. The two accuracies are \emph{standard accuracy} which is the accuracy on unperturbed test data, and \emph{robust accuracy} which is the accuracy on adversarially perturbed test data. More formally \emph{standard accuracy} quantifies the accuracy of an estimator on an unperturbed test data that is generated from the same distribution as the training data:
\begin{align}\label{eq:SRdef}
\SR(\hth) &:= \mathbb{P}\{\widehat{y}= y\}\,, \quad\text{where}\quad (\x,y)\sim \mathbb{P}
\end{align}
Our second accuracy, called \emph{robust accuracy} quantifies robustness of an estimator to adversarial perturbations in the test data. Specifically,
\begin{align}\label{eq:ARdef}
\AR(\hth) &:= \E\Big[\min_{\pnorm{\bdelta}\le\eps} {\mathlarger{\ind}}_{\big\{y \<\x+\bdelta,\hth\> \ge0\big\}}\Big]\,, \quad\text{where}\quad (\x,y)\sim \mathbb{P}.
\end{align}
We end this section by stating a lemma that characterizes $\SR(\hth)$ and $\AR(\hth)$ 
under the Gaussian mixture model. We defer the proof to Appendix~\ref{proof:SR-AR}.
\begin{lemma}\label{lem:SR-AR}
Consider mixtures of Gaussian data model where $y_i\in \{-1,+1\}$ with corresponding probabilities $\pi_-, \pi_+$ and the feature vector distributed as $\x\sim \normal(y\vct{\mu}, \mtx{\Sigma})$, conditioned on $y$, where $\vct{\mu}\in \reals^d$ and $\mtx{\Sigma}\in \reals^{d\times d}$.  Then,
\begin{align}
\SR(\hth)&:= \Phi\left(\frac{\<\vct{\mu},\hth\>}{\twonorm{\mtx{\Sigma}^{1/2}\hth}}\right)\,,\\
\AR(\hth)&:= \Phi\left(\frac{\<\vct{\mu},\hth\>-\eps \qnorm{\hth}}{\twonorm{\mtx{\Sigma}^{1/2}\hth}}\right).
\end{align}
Here, $\Phi(x) = \frac{1}{\sqrt{2\pi}} \int_{-\infty}^x e^{-\frac{t^2}{2}} \de t$ is the cdf of a standard Gaussian distribution and q is such that $\frac{1}{p}+\frac{1}{q} = 1$.
\end{lemma} 
By Lemma~\ref{lem:SR-AR}, characterizing $\SR(\hth^\eps)$ and $\AR(\hth^\eps)$ amounts to characterizing $\<\vct{\mu}, \hth^\eps\>$, $\twonorm{\mtx{\Sigma}^{1/2}\hth}$, $\qnorm{\hth^\eps}$, which constitutes the bulk of our analysis.

\section{Prelude: two regimes for adversarial training}
Similar to normal classification, an interesting phenomena that arises in adversarial classification is that depending on the size of the training data there are two different regimes of operation: Robustly separable and non-separable. In the \emph{robustly separable} regime there is a robust classifier that perfectly separates the training data, with a positive margin that depends on the adversary's power, while this is not possible in the non-separable case. We formally define this notion of robust separability below.
\begin{definition}[Robust linear separability]
Given $\eps>0$ and $q\ge 1$, we call a training data $\{(\x_i,y_i)\}_{i=1}^n$, $(\eps,q)$-separable if
\begin{align}\label{eq:serp-def}
\exists \vct{\theta}\in\reals^d:\quad \forall i\in[n],\;\; y_i\<\vct{x}_i,\vct{\theta}\> -\eps \qnorm{\vct{\theta}} > 0. 
\end{align}
\end{definition}
We note that our notion of robust separability is closely related to the standard notion of separability by a linear classifier. In particular, using a simple rescaling argument\footnote{\eqref{eq:serp-def}$\Rightarrow$\eqref{eq:serp-def2}: Scaling by $\frac{1}{\eps\qnorm{\bth}}$ we see that $y_i\<x_i,\tilde{\bth}\> >1$ for $\tilde{\bth} = \frac{\bth}{\eps\qnorm{\bth}}$, and $\qnorm{\tilde{\bth}}  = \frac{1}{\eps}$ by definition. \eqref{eq:serp-def2}$\Rightarrow$\eqref{eq:serp-def}: Letting $c=\frac{1}{\eps\qnorm{\bth}}\ge 1$, we also have $y_i\<x_i, c\bth\>>1$. Substituting for $c$ and rearranging the terms we get \eqref{eq:serp-def}.} one can rewrite condition \ref{eq:serp-def} as follows
\begin{align}\label{eq:serp-def2}
\exists \vct{\theta}, \;\;\qnorm{\bth} \le \frac{1}{\eps}:\quad \forall i\in[n],\;\; y_i\<\vct{x}_i,\vct{\theta}\> > 1. 
\end{align}
Therefore, robust separability is akin to linear separability of the data but with a budget constraint on the $\ell_q$ norm of the coefficients of the classifier.

When the training data is $(\eps, q)$-separable (with $\ell_q$ the dual norm of $\ell_p$), then the minimax estimator $\hth^\eps$ becomes unbounded and achieves zero adversarial training loss in~\eqref{eq:inner}. In other words, one can completely interpolate the data. This is due to the fact that if $\bth$ is an ($\epsilon,q$)-separator, then $c\bth$ with $c\to \infty$ leads to zero adversarial training loss and since the loss is nonnegative it is optimal. 
Although the norm of $\hth^\eps$ tends to infinity in the separable regime, what matters for {our linear classifier} is the direction of $\hth^\eps$. However, in this separable regime even the direction of the optimal solution ($\frac{\hth^\eps}{\twonorm{\hth^\eps}}$) may not be unique. Even though there may be multiple optimal directions it is possible to show that the direction that gradient descent converges to is a specific maximum margin classifier. We formally state this result which is essentially a direct consequence of~\cite{lyu2019gradient, ji2018risk} below. 
\begin{propo}\label{pro:SVM}
Consider the adversarial training loss
\[
\cL(\bth) := \frac{1}{n} \sum_{i=1}^n  \ell\left(y_i \<\x_i,\bth\> - \eps \qnorm{\bth}\right)\,,
\]
with the loss $\ell(t)$ obeying certain technical assumptions\footnote{See \cite[Assumption S3 in Appendix F]{lyu2019gradient}. We list these assumptions in Appendix~\ref{app:pro:SVM} for readers' convenience.} which are satisfied for common classification losses such as logistic, exponential, and hinge losses. Then, the gradient descent iterates 
\begin{align*}
\bth_{\tau+1}=\bth_{\tau}-\mu\nabla\mathcal{L}(\bth_{\tau})
\end{align*}
with a sufficiently small step size $\mu$ obey
\begin{align}
\lim_{t\to\infty} \twonorm{\frac{\bth_t}{\twonorm{\bth_t}} - \frac{\tth^\eps}{\twonorm{\tth^\eps}}} = 0\,,
\end{align}
where $\tth^\eps$ is the solution to the following max-margin problem
\begin{align}
\tth^\eps = &\arg\min_{\bth\in\reals^d} \quad  \twonorm{\bth}^2\nn  \\
&{\rm subject}\;{\rm to}\;\; y_i\<\vct{x}_i,\vct{\theta}\> - \eps \qnorm{\bth} \ge1\,.\label{eq:MM}
\end{align}  
\end{propo}

In the non-separable regime, as we show in the proof of Theorem~\ref{thm:isotropic-nonseparable} the minimizer $\hth^\eps$ is bounded. Moreover, the loss~\eqref{eq:inner} is convex as it is pointwise maximum of a set of convex functions (see~\eqref{eq:minimaxEst} and recall convexity of loss $\ell$). Therefore, a variety of  iterative methods (including gradient descent) can be used to converge to a global minimizer of \eqref{eq:minimaxEst}. Theorem \ref{thm:isotropic-nonseparable}  also shows that all global minimizers of ~\eqref{eq:minimaxEst} have the same standard and robust accuracy.
\section{Main results for isotropic features}\label{sec:main-isotorpic}
In this section we present our main results. For the sake of exposition, in this section we state our results for the case where the features are isotropic (i.e.~$\mtx{\Sigma}=\mtx{I}$). We discuss our more general results with anisotropic features in Section~\ref{sec:extension-anisotropic}. In this paper, we establish a sharp phase-transition characterizing the separability of the training data generated according to a Gaussian mixture model. Specifically, in our asymptotic regime (see Section~\ref{sec:adversarial}) we characterize a threshold $\delta_*$ such that for $\delta<\delta_*$ the data is $(\eps,q)$-separable, with high probability, and for $\delta>\delta_*$ it is non-separable, with high probability. This phase transition for robust separability is discussed in Section \ref{PTsec}. We also precisely characterize the standard accuracy $\SR(\hth^\eps)$ and the robust accuracy $\AR(\hth^\eps)$ of the point that gradient descent converges to in both the separable and non-separable data  regimes which are the subject of Sections \ref{sepsec} and \ref{nonsepsec}, respectively. We then discuss the implications of our main results for the special cases of  $\ell_p$ perturbations with $p=1$, $p=2$, and $p=\infty$ in Section \ref{specialcase}.


\subsection{Phase transition for robust data separability}
\label{PTsec}
In this section we discuss our results for characterizing the phase transition for $(\eps,q)$-separability under the Gaussian mixtures model. As detailed earlier in Section \ref{sec:data}, in our asymptotic setting the dimension of the mean vector $\vct{\mu}$ ($d$) as well as the size of the training data ($n$) grow to infinity in proportion with each other $n/d=\delta$. To state our main result we need a few technical assumptions on the limiting behavior of the mean vector. We begin with a simple assumption on the convergence of the Euclidean norm of the mean vector.

%
\begin{assumption}[Convergence of Euclidean norm of $\vct{\mu}$]\label{ass:norm-mu} We assume the Euclidean norm of the mean vector converges to a bounded quantity, that is $\twonorm{\vct{\mu}} \to V <\infty$, as $n\to\infty$ and $n/d\to\delta$.
\end{assumption}
We note that for the isotropic case, the boundedness condition in Assumption~\ref{ass:norm-mu} is already implied by Assumption~\ref{ass:asymptotic}(b).

Naturally, the separability threshold depends on the mean vector and the adversary's power. For instance, intuitively, one expects the separability threshold to decrease as the adversary's power or the length of the mean vector increases. We also expect the direction of the mean vector $\frac{\vct{\mu}}{\twonorm{\vct{\mu}}}$ to play a role. We capture these effects via the spherical width of a suitable set. Recall that the \emph{spherical width} of a set $\mathcal{S}\subset\R^d$ is a measure of its complexity and is defined as 
$$\omega_s\left(\mathcal{S}\right)=\E\Big[\sup_{\vct{z}\in\mathcal{S}} \vct{z}^T\vct{u}\Big]\,,$$ 
where $\vct{u}\in\mathcal{S}^{d-1}$ is a vector chosen uniformly at random from the unit sphere. In particular, the appropriate set for characterizing the separability threshold takes the form
\begin{align}\label{eq:S}
\cS(\alpha, \theta, \eps_0, \vct{\mu}) : = \left\{\vct{z}\in \reals^d :\quad  \bz^T\vct{\mu} =0,\; \twonorm{\vct{z}} \le \alpha,\; \qnorm{\vct{z} +\theta \frac{\vct{\mu}}{\twonorm{\vct{\mu}}}} \le \frac{1}{\eps_0\pnorm{\vct{\mu}}} \right\}\,,
\end{align}
where $\eps_0$ is the adversary's scaled power per Assumption~\ref{ass:asymptotic}(c). Next assumption focuses on   the spherical width convergence in our asymptotic regime.
\begin{assumption}[Convergence of spherical width]\label{ass:omegas} We assume the following limit exists
\begin{align}
 \omega\left(\alpha, \theta, \eps_0\right):=  \lim_{n\to\infty} \omega_s\left(\cS(\alpha, \theta, \eps_0, \vct{\mu})\right)\,.
\end{align}
\end{assumption}
As it will become clear later on in this section Assumptions~\ref{ass:norm-mu} and \ref{ass:omegas} are trivially satisfied in various settings. With these assumptions  in place we are ready to state our result precisely characterizing the separability threshold.
\begin{theorem}\label{thm:sep-thresh}
Consider a data set generated i.i.d.~according to an isotropic Gaussian mixture data model per Section \ref{sec:data} and suppose the mean vector $\vct{\mu}$ obeys Assumptions \ref{ass:norm-mu} and \ref{ass:omegas}. Also define
\begin{align}\label{eq:threshold}
\delta_* :=  \min_{\alpha\ge0, \theta} \frac{ \omega\left(\alpha, \theta, \eps_0\right)^2}{\E\left[\left(1-V\theta+\sqrt{\alpha^2+\theta^2}g\right)_{+}^2\right]}\,,
\end{align}
where the expectation is taken with respect to $g\sim \normal(0,1)$. Then, under the asymptotic setting of Assumption \ref{ass:asymptotic}, for $\delta < \delta_*$ the data are $(\eps,q)$-separable with high probability and for $\delta>\delta_*$, the data are non-separable, with high probability. Namely,
\begin{align*}
\delta < \delta_* &\Rightarrow  \lim_{n\to\infty} \prob(\text{data is $(\eps,q)$-separable}) =1\,,\\
\delta > \delta_* &\Rightarrow  \lim_{n\to\infty} \prob(\text{data is $(\eps,q)$-separable}) = 0\,.
\end{align*}
\end{theorem}

Theorem \ref{thm:sep-thresh} above precisely characterizes the separability threshold as a function of the adversary's power as well as properties of the mean vector. In particular since $\omega$ decreases with the increase in $\eps_0$, this theorem indicates that the separability threshold decreases as the adversary's power  increases. This of course conforms with our natural intuition and is consistent with characterization~\eqref{eq:serp-def2}. To better understand the implications of Theorem~\ref{thm:sep-thresh} we now consider some special cases.
\begin{itemize}[leftmargin=*]
\item {\bf Example 1 (Non-adversarial setting).} Our first example focuses on the non-adversarial setting where $\eps_0=0$. In this case the $\ell_q$ constraint in definition of $\cS$, given by~\eqref{eq:S}, is void and the set $\cS$ becomes the intersection of $\ell_2$ ball of radius $\alpha$ with the hyperplane of dimension $d-1$ that is orthogonal to $\vct{\mu}$. Therefore $\omega(\alpha,\theta,\eps_0) = \omega_s(\cS) = \alpha$ and the separability threshold reduces to 
\[
\delta_* :=  \max_{\alpha\ge0, \theta} \frac{\alpha^2}{\E\left[\left(1-V\theta+\sqrt{\alpha^2+\theta^2}g\right)_{+}^2\right]}\,.
\]
By the change of variables $(\alpha, \frac{\theta}{\alpha}) \to (\alpha, \theta)$, it is straightforward to see that optimal $\alpha$ is at $+\infty$ and the separability condition reduces to
\[
\delta_* :=  \left(\min_{\theta} {\E\left[\left(-V\theta+\sqrt{1+\theta^2}g\right)_{+}^2\right]}\right)^{-1}\,.
\]
\item {\bf Example 2 ($\ell_2$ perturbation).} When $p=q=2$, the set $\cS$ becomes the intersection of $\ell_2$ ball of radius $$R:=\min\left(\alpha,\sqrt{\tfrac{1}{\eps_0^2\twonorm{\vct{\mu}}^2} - \theta^2} \right)\,,$$ with the hyperplane of dimension $d-1$ that is orthogonal to $\vct{\mu}$. Therefore $\omega(\alpha,\theta,\eps_0) = \omega_s(\cS) = \min\left(\alpha,\sqrt{\tfrac{1}{\eps_0^2V^2} - \theta^2} \right)$ and the separability threshold reduces to
\[
 \delta_* = \max_{\alpha\ge0,\text{ }\theta \le \frac{1}{\eps_0V}}\;\;  \frac{\min\left(\alpha^2,\tfrac{1}{\eps_0^2V^2} - \theta^2 \right)}{\E\left[\left(1-V\theta+\sqrt{\alpha^2+\theta^2}g\right)_{+}^2\right]}\,.
\]
Note that the above ratio is decreasing in $\alpha$ over the range of $\alpha\ge \sqrt{\tfrac{1}{\eps_0^2\twonorm{\vct{\mu}}^2} - \theta^2}$. Therefore, the maximizer $\alpha$ should satisfy $\alpha \le \sqrt{\tfrac{1}{\eps_0^2\twonorm{\vct{\mu}}^2} - \theta^2}$ and this further simplifies the expression for $\delta_*$ as follows
\[
 \delta_* = \max_{\alpha\ge0,\, \alpha^2+\theta^2 \le \frac{1}{\eps_0^2V^2}}\;\;  \frac{\alpha^2}{\E\left[\left(1-V\theta+\sqrt{\alpha^2+\theta^2}g\right)_{+}^2\right]}\,.
\]
By the change of variable $(\alpha, \frac{\theta}{\alpha}) \to (\alpha, \theta)$, this can be written as:
\[
 \delta_* = \left(\min_{\alpha\ge0, \theta, \alpha^2(1+\theta^2)\le \frac{1}{\eps_0^2V^2}}\;\;  {\E\left[\left(\frac{1}{\alpha}-V\theta+\sqrt{1+\theta^2}g\right)_{+}^2\right]}\right)^{-1}\,.
\] 
Since the inner function is decreasing in $\alpha$ it is minimized at $\alpha_* = \frac{1}{\eps_0V\sqrt{1+\theta^2}}$ which simplifies the separability threshold to the following:
\begin{align}\label{eq:thresh-L2}
 \delta_* = \left(\min_{\theta}\;\;  {\E\left[\left((\eps_0\sqrt{1+\theta^2}-\theta)V+\sqrt{1+\theta^2}g\right)_{+}^2\right]}\right)^{-1}\,.
\end{align}
\end{itemize}

To the best of our knowledge, our paper is the first work that shows such a phase transition for robust separability in the adversarial setting. In the non-adversarial case, similar phase transitions have been shown for data separability (a.k.a interpolation threshold) \cite{candes2020phase, montanari2019generalization, deng2019model}. More specifically, \cite{candes2020phase} derived separability threshold for a logistic link regression model. Similar phenomenon extends to other link functions, as characterized by \cite{montanari2019generalization}, and also to Gaussian mixtures model \cite{deng2019model}. Interestingly, our result specialized to the case where the adversary has no power (cf. Example 1) recovers the existing thresholds for Gaussian mixtures model.

We end this section by demonstrating that in addition to the examples above Assumption \ref{ass:omegas} holds for a fairly broad family of mean vectors. This is the subject of the next lemma. We defer the proof of this lemma to Appendix~\ref{proof:lem:justify}.
\begin{assumption}\label{ass:alternative}
Suppose that the empirical distribution of the entries of $\sqrt{d} \vct{\mu}$ converges weakly to a distribution $\prob_{M}$ on real line, with bounded $2^{nd}$ and $p^{th}$ moment ($\int x^2 \de\prob_M(x) = \sigma_{M,2}^2<\infty$, $\int |x|^p \de\prob_M(x) = \sigma_{M,p}^p<\infty$).
\end{assumption}
\begin{lemma}\label{lem:justify}
Consider the asymptotic regime of $n\to\infty$ and $n/d\to\delta$, for some $\delta\in(0,\infty)$.
Also, consider the function $J_q(\cdot;\cdot):\reals\times\reals_{\ge0}\mapsto \reals_{\ge0}$ defined by
\begin{align}\label{eq:Jq}
J_q(x;\lambda) = \min_{u}\; \frac{1}{2}(x-u)^2+\lambda |u|^q\,.
\end{align}
Then Assumption~\ref{ass:alternative} implies Assumption~\ref{ass:omegas} with
\begin{align}
 \omega\left(\alpha, \theta, \eps_0\right) = \min_{\lambda_0,\eta\ge0, \nu}  \;\; &\sqrt{\delta} \left\{\frac{\nu^2}{2\eta}+\frac{1}{2\eta \delta} + \frac{\eta}{2}\alpha^2 + {\lambda_0}(\eps_0\sigma_{M,p})^{-q}\right\}\nn\\
 &-\eta\sqrt{\delta} \E\bigg[J_q\bigg(\frac{h}{\eta\sqrt{\delta}}
 - \left(\frac{\nu}{\eta}-\theta\right)\frac{M}{\sigma_{M,2}};\frac{\lambda_0}{\eta}\bigg) \bigg] \,,
\end{align}
where the expectation in the last line is taken with respect to the independent random variables $h\sim\normal(0,1)$ and $M\sim\prob_M$.
\end{lemma} 

\subsection{Precise characterization of SA and RA in the separable regime}
\label{sepsec}

In this section we precisely characterize the SA and RA of the classifier obtained as the limiting point of gradient descent on the loss \eqref{eq:inner} in the separable regime. As discussed in Proposition~\ref{pro:SVM}, the normalized  iterations of gradient descent for the loss~\eqref{eq:inner} converge to the max-margin classifier~\eqref{eq:MM}. Since $\SR(\bth)$ and $\AR(\bth)$ are only functions of the direction $\frac{\bth}{\twonorm{\bth}}$, instead of studying the classifier obtained via GD iterations directly, we study the classification performance of the max-margin classifier. 

Recall the function $J_q$ is given by~\eqref{eq:Jq}, and define
\begin{align}\label{eq:EJ}
\EJ(c_0,c_1;\lambda_0) = \E\left[J_q\left(\frac{c_0}{\sqrt{\delta}} h - c_1 \frac{M}{\sigma_{M,2}}; \lambda_0 \sigma_{M,p}^q\right)\right]\,,
\end{align}
where the expectation in the last line is taken with respect to the independent random variables $h\sim\normal(0,1)$ and $M\sim\prob_M$, per the setting of Assumption~\ref{ass:alternative}. Our characterization of $\SR$ and $\AR$ will be in terms of the function $\EJ$ as formalized in the next theorem.

\begin{theorem}\label{thm:isotropic-separable}
Consider a data set generated i.i.d.~according to an isotropic Gaussian mixture data model per Section \ref{sec:data} and suppose the mean vector $\vct{\mu}$ obeys Assumptions~\ref{ass:asymptotic} and \ref{ass:alternative}. Also let $\tth^\eps$ be the max margin solution per~\eqref{eq:MM}. If $\delta<\delta_*$, with $\delta_*$ given by~\eqref{eq:threshold}, then in the asymptotic setting of Assumption \ref{ass:asymptotic} we have:
\begin{itemize}
\item[(a)] The following convex-concave minimax scalar optimization has a  bounded solution $(\alpha_*,\gamma_{0*},\theta_*,\beta_*,\lambda_{0*},\eta_*,\tilde{\eta}_*)$ with the minimization components $(\alpha_*,\gamma_{0*},\theta_*)$ unique:
\begin{align}\label{eq:sep-AO9}
\min_{\alpha,\gamma_0\ge 0,\theta} \max_{\beta,\lambda_0,\eta\ge0, \tilde{\eta}} \quad &D_{\rm s}(\alpha,\gamma_{0},\theta,\beta,\lambda_{0},\eta,\tilde{\eta}), \quad \text{where}\nn
\end{align}
\begin{align}
D_{\rm s}(\alpha,\gamma_{0},\theta,\beta,\lambda_{0},\eta,\tilde{\eta}) \;=\; &2  \left(1+\frac{\eta}{2\alpha}\right)^{-1} 
\EJ\left(\frac{\beta}{2}, \frac{\tilde{\eta}}{2}; \frac{\lambda_0}{q\gamma_0^{q-1}}\left(1+\frac{\eta}{2\alpha}\right)^{1-q}\right)\nn\\
&-\left(\frac{\beta^2}{\delta}+\tilde{\eta}^2\right) \frac{1}{4(1+\frac{\eta}{2\alpha})}-\frac{2\lambda_0}{q}  \gamma_0 -\frac{\eta\alpha}{2} -\tilde{\eta}\theta \nn\\
&  +\beta\sqrt{\E\Bigg[\left(\left(1+\eps_0\gamma_0-\theta\sigma_{M,2}\right)+\alpha g\right)_{+}^2\Bigg]}\,,
\end{align}
where the expectation in the last part is taken with respect to $g\sim\normal(0,1)$.
\item[(b)] It holds in probability that
\begin{align}
\lim_{n\to\infty} \frac{1}{\twonorm{\vct{\mu}}}\<\vct{\mu},\tth^\eps\> &= \theta_*\,,\\
\lim_{n\to\infty} \twonorm{\tth^\eps} &= \alpha_*\,,\\
\lim_{n\to\infty} \pnorm{\vct{\mu}}\qnorm{\tth^\eps} &= \gamma_{0*}\,.
\end{align}
\item[(c)] Furthermore, part part (b) combined with Lemma~\ref{lem:SR-AR} imply the following limits hold in probability:
\begin{align}
\lim_{n\to\infty} \SR(\tth^\eps) &= \Phi\left(\sigma_{M,2} \frac{\theta_*}{\alpha_*}\right)\,,\\
\lim_{n\to\infty} \AR(\tth^\eps) &= \Phi\left(-\frac{\eps_0\gamma_{0*}}{\alpha_*} + \sigma_{M,2}\frac{\theta_*}{\alpha_*}\right)\,.\label{eq:AR-sep}
\end{align}
\end{itemize}
\end{theorem}
Theorem \ref{thm:isotropic-separable} above provides us with a precise characterization of SA and RA and allows us to rigorously quantify the effect of adversary’s manipulative power $\eps_0$, mean vector $\vct{\mu}$, and scaling of dimensions $\delta$ on SA and RA. In particular, this theorem precisely characterizes the performance of the max margin classifier (and in turn the classifier GD converges to) in terms of the optimal solutions to a low-dimensional optimization problem, namely \eqref{eq:sep-AO9}. It is worth noting that by part (b), $\theta_*$ is the asymptotic value of the projection of the estimator $\tth^\eps$ along the direction of the class averages $\bmu$, and $\alpha_*$ represents the asymptotic value of the $\ell_2$ norm of the estimator. Therefore, the $\theta_*/\alpha_*$ term appearing in the $\SR$ and $\AR$ formulae corresponds to the correlation coefficient between the estimator $\tth^\eps$ and the class averages $\bmu$.

While the optimization problem \eqref{eq:sep-AO9} may look quite complicated, we note that it is a convex-concave problem in a handful number of scalar variables and hence can be easily solved by a low-dimensional gradient descent/ascent rather fast and accurately.  In addition, this low-dimensional optimization problem significantly simplifies for special cases of $p$. We discuss some of these cases, which are also of particular practical interest, in Sections \ref{sectwo} and \ref{secinf}. 

\subsection{Precise characterization of SA and RA in non-separable regime}
\label{nonsepsec}
In this section we precisely characterize the SA and RA of the classifier obtained by running gradient descent on the loss \eqref{eq:inner} in the non-separable regime. Before we can state our main result we need the definition of the Moreau envelop.
\begin{definition}[Moreau envelope and expected Moreau envelope] The \emph{Moreau envelope}  or \emph{Moreau-Yosida regularization} of a function $\ell$ is given by
\begin{align}
e_\ell(x;\mu) := \min_{t} \frac{1}{2\mu} (x-t)^2 + \ell(t)\,. 
\end{align}
We also define the \emph{expected} Moreau envelope
\begin{align}
L(a,b,\mu) = \E[e_{\ell}(ag+b;\mu)]\,,
\end{align}
where the expectation is taken with respect to $g\sim\normal(0,1)$. 
\end{definition}
We this definition in place we are now ready to state our main result in the non-separable regime.
\begin{theorem}\label{thm:isotropic-nonseparable}
Consider a data set generated i.i.d.~according to an isotropic Gaussian mixture data model per Section \ref{sec:data} and suppose the mean vector $\vct{\mu}$ obeys Assumption~\ref{ass:alternative}. Also let $\hth^\eps$ be the solution to optimization~\eqref{eq:inner}. If $\delta>\delta_*$, with $\delta_*$ given by~\eqref{eq:threshold}, then in the asymptotic setting of Assumption \ref{ass:asymptotic} we have:
\begin{itemize}
\item[(a)] The following convex-concave minimax scalar optimization has a bounded solution $(\theta_*,\alpha_*,\gamma_{0*},\tau_{g*},\beta_*,\tau_{h*})$ with the minimization components $(\alpha_*,\gamma_{0*},\theta_*)$ unique:
\begin{align}
&\min_{\theta, 0\le \alpha, \gamma_0, \tau_g}\;\;\max_{0\le\beta, \tau_h} \;\; D_{\rm ns}(\alpha, \gamma_0, \theta, \tau_g, \beta,\tau_h)\nn\\
& D_{\rm ns}(\alpha, \gamma_0, \theta, \tau_g, \beta,\tau_h) = \frac{\beta \tau_g}{2} +L\left(\sqrt{\alpha^2+\theta^2},\sigma_{M,2} \theta-\eps_0 \gamma_0,\frac{\tau_g}{\beta}\right) \nn \\
&\quad\quad\quad\quad\quad- \min_{\lambda_0\ge0,\nu} \left[\frac{\alpha}{\tau_h}\left\{\frac{\beta^2}{2\delta} + \lambda_0 \left(\frac{\gamma_0\tau_h}{\alpha}\right)^q  + \frac{\nu^2}{2}
-\EJ\left(\beta,\left(\frac{\tau_h\theta}{\alpha} + {\nu}\right) ;\lambda_0\right)\right\} +\frac{\alpha\tau_h}{2} \right]\,.\label{eq:main-thm}
\end{align}
\item[(b)] It holds in probability that
\begin{align}
\lim_{n\to\infty} \frac{1}{\twonorm{\vct{\mu}}}\<\vct{\mu},\hth^\eps\> &= \theta_*\,,\\
\lim_{n\to\infty} \twonorm{\pproj_{\vct{\mu}}  \hth^\eps} &= \alpha_*\,,\\
\lim_{n\to\infty} \pnorm{\vct{\mu}}\qnorm{\hth^\eps} &= \gamma_{0*}\,.
\end{align}
\item[(c)] As a corollary of part (b) and Lemma~\ref{lem:SR-AR}, the following limits hold in probability:
\begin{align}
\lim_{n\to\infty} \SR(\hth^\eps) &= \Phi\left(\frac{\sigma_{M,2}\theta_*}{\sqrt{\alpha_*^2+\theta_*^2}}\right)\,,\\
\lim_{n\to\infty} \AR(\hth^\eps) &= \Phi\left(\frac{-\eps_0\gamma_{0*} + \sigma_{M,2}\theta_*}{\sqrt{\alpha_*^2+\theta_*^2}}\right)\,. \label{eq:AR-nonsep}
\end{align}
\end{itemize}
\end{theorem}
It is worth noting that by part (b), $\theta_*$ is the asymptotic value of the projection of the estimator $\hth^\eps$ along the direction of the class averages $\bmu$. In addition, \[\lim_{n\to\infty} \twonorm{\pproj_{\vct{\mu}}  \hth^\eps}^2+ \twonorm{\proj_{\vct{\mu}}  \hth^\eps}^2= \alpha_*^2+\theta_*^2\] represents the asymptotic value of the squared $\ell_2$ norm of the estimator. Therefore, the $\theta_*/\sqrt{\alpha_*^2+\theta_*^2}$ term appearing in the $\SR$ and $\AR$ formulae corresponds to the correlation coefficient between the estimator $\hth^\eps$ and the class averages $\bmu$.

Theorem~\ref{thm:isotropic-nonseparable} complements the result of Theorem~\ref{thm:isotropic-separable} by providing a precise characterization of SA and RA measures in the non-separable regime.
In the remaining part of this section and also in the next section, we specialize our results to several specific choices of $p$ that are of particular practical interest.

\begin{remark}\label{rem:2eps}
In stating our results (Theorems~\ref{thm:isotropic-separable} and \ref{thm:isotropic-nonseparable}), we are implicitly assuming the same variable $\eps_0$ for both the perturbation level to the test data as well as the `perceived' perturbation level used in the robust minimax estimator $\hth^\eps$. In principle, we can use different variable for the test perturbation level, say $\eps_{0,{\rm test}}$. The same results applies to this setting with minimal modifications; only in the \AR formalue, cf. equations~\eqref{eq:AR-sep}, \eqref{eq:AR-nonsep} the variable $\eps_0$ should be replaced by $\eps_{0,\rm{test}}$. 
\end{remark}
\section{Results for special cases of $p$}
\label{specialcase}
In this section we discuss the implications of our main results for the special cases of  $\ell_p$ perturbations with $p=2$ in Section \ref{sectwo}, $p=\infty$ in Section \ref{secinf}, and $p=1$ in Section \ref{extentsec}.  We refer to Appendix~\ref{proof:sec5} for the proofs of theorems and corollaries stated in this section.
\subsection{Results for $\ell_2$ perturbation}
\label{sectwo}

We begin with stating our results for $\ell_p$ perturbation with $p=2$. This result can be viewed as a corollary of Theorem \ref{thm:sep-thresh}, Theorem~\ref{thm:isotropic-separable}, and Theorem \ref{thm:isotropic-nonseparable} specializing our main result for $p=2$.

\begin{corollary}\label{cor:p2-sep}
Consider a data set generated i.i.d.~according to an isotropic Gaussian mixture data model per Section \ref{sec:data} and suppose the mean vector $\vct{\mu}$ obeys Assumptions~\ref{ass:alternative}.  Then in the asymptotic setting of Assumption \ref{ass:asymptotic} we have:
\begin{itemize}
\item [(a)] The separability threshold $\delta_*$ is given by
\begin{align}\label{eq:thresh-L2-repeat}
 \delta_* = \left(\min_{\theta}\;\;  {\E\left[\left((\eps_0\sqrt{1+\theta^2}-\theta)V+\sqrt{1+\theta^2}g\right)_{+}^2\right]}\right)^{-1}\,.
\end{align} 
\item [(b)] In the separable regime where $\delta<\delta_*$, the followings hold in probability for the max margin solution $\tth^\eps$  (see ~\eqref{eq:MM}): 
\begin{align}
\label{SRAR1}
\lim_{n\to\infty} \SR(\tth^\eps) = \Phi\left(\sigma_{M,2} \frac{\theta_*}{\alpha_*}\right)\,,\quad
\lim_{n\to\infty} \AR(\tth^\eps) = \Phi\left(\frac{\theta_*}{\alpha_*}\sigma_{M,2} -\eps_0\sigma_{M,2}\right)\,,
\end{align}
where 
\begin{align}
\alpha_* = \left(\tilde{\alpha}_*^{-1} - \eps_0 \sigma_{M,2}\right)^{-1}\,, \quad \theta_* = u_*\alpha_*\,.
\end{align}
Here, $(\tilde{\alpha}_*, u_*)$ the solution to the following problem:
\begin{align}
\label{coropt1}
&\min_{\tilde{\alpha} \ge 0,u} \;\; \tilde{\alpha}^2
\nn\\
&{\rm subject\;\; to}\quad 1\ge u^2+\delta\E\left[\left(\frac{1}{\tilde{\alpha}}- u\sigma_{M,2}+ g\right)_{+}^2\right]\,,
\end{align}
with expectation taken with respect to $g\sim\normal(0,1)$.
\item[(c)]  In the non-separable regime where $\delta>\delta_*$, the followings hold in probability for the optimal solution $\hth^\eps$ of~\eqref{eq:inner}: 
\begin{align}
\label{SRAR2}
\lim_{n\to\infty} \SR(\hth^\eps) &= \Phi\left(\frac{\sigma_{M,2}\theta_*}{\sqrt{\alpha_*^2+\theta_*^2}}\right)\,,\\
\lim_{n\to\infty} \AR(\hth^\eps) &= \Phi\left(\frac{\theta_*}{\sqrt{\alpha_*^2+\theta_*^2}} \sigma_{M,2} - \eps_0\sigma_{M,2}  \right)\,.
\end{align}
where $(\alpha_*,\theta_*,\beta_*)$ is the bounded solution of the following convex-concave minimax scalar optimization problem with the minimization components $(\alpha_*,\theta_*)$ unique:
\begin{align}
&\max_{0\le\beta}\;\; \min_{\theta, 0\le \alpha}\;\; D_{\rm ns}(\alpha, \theta, \beta)\nn\\
& D(\alpha, \theta, \beta) = L\left(\sqrt{\alpha^2+\theta^2},\sigma_{M,2} \theta - \eps_0 \sqrt{\alpha^2+\theta^2},\frac{\alpha}{\beta\sqrt{\delta}}\right) 
- \frac{\alpha\beta}{2\sqrt{\delta}}\,.\label{eq:main-thm-iso}
\end{align}
\end{itemize}

\end{corollary}
The corollary above precisely characterizes the behavior of the classifier that gradient descent converges to in terms of low-dimensional optimization problems (\eqref{coropt1} in the separable regime and \eqref{eq:main-thm-iso} in the non-separable regime). 

Recall that the term $\theta_*/\alpha_*$ in the separable regime and the term $\theta_*/\sqrt{\alpha_*^2+\theta_*^2}$ in the non-separable regime correspond to the correlation coefficient between the robust minimax estimator and the classes average $\bmu$. As we will see in Figure~\ref{fig:SA-RA-p2-f}, the standard accuracy is decreasing in $1/\delta$, for any fixed $\eps_0$, which equivalently indicates that the correlation between the estimator and $\bmu$ is monotone increasing in the sample-to-dimension ratio $\delta$.

As we will see in the coming sections, SA and RA curves have a highly non-trivial behavior which also strongly depend on the choice of $p$. This necessitate a rigorous theory (such as the above) that can precisely predict these curves. To better understand the implications and consequences of this result we focus on its various predictions. Specifically, we find the global optima of the two low-dimensional optimization problems via simple gradient descent/ascent and use it to calculate the corresponding $\SR$ and $\AR$ based on \eqref{SRAR1} and \eqref{SRAR2}. We also verify these theoretical predictions with the performance of gradient descent on the loss \eqref{eq:inner} with a polyak/approximate polyak step size in the separable/non-separable regimes.\footnote{Specifically we run gradient descent iterations of the form $\vct{\theta}_{\tau+1}=\vct{\theta}_{\tau}-\alpha_\tau\nabla \mathcal{L}(\vct{\theta}_\tau)$ on \eqref{eq:inner} with a Polyak step size $\alpha_\tau=\frac{\mathcal{L}(\vct{\theta}_\tau)}{\twonorm{\nabla \mathcal{L}(\vct{\theta}_\tau)}^2}$ in the separable regime and an approximate Polyak step size $\alpha_\tau=\frac{\mathcal{L}(\vct{\theta}_\tau)-\min_{0\le t\le \tau} \mathcal{L}(\vct{\theta}_t)+\frac{\gamma}{\tau}}{\twonorm{\nabla \mathcal{L}(\vct{\theta}_\tau)}^2}$ in the non-separable regime.} 

We plot the theoretically predicted standard and robust accuracy versus the adversary's power $\eps_0$ together with the corresponding empirical results in Figure \ref{fig:SA-RA-p21} (a) and (b). The solid lines depict theoretical predictions with the dots representing the empirical performance of gradient descent with the algorithmic settings discussed above. The data set is generated according to a Gaussian Mixture Model per Section \ref{sec:data} with $\vct{\mu}\in\R^d$ consisting of i.i.d.~$\mathcal{N}\left(0,\frac{1}{d}\right)$ entries with dimension $d=400$. Each dot represents the average of $100$ trials. These figures demonstrate that even for moderate dimension sizes our theoretical prediction is a near perfect match with the empirical performance of gradient descent. We note that when $\eps_0$ is sufficiently large then the adversarially trained model $\hth^\eps$ becomes zero due to the large regularization in the argument of loss function in~\eqref{eq:inner} and $\SR$ and $\AR$ measures are not defined. The curves are plotted up to that $\eps_0$.

An intriguing observation of Corollary~\ref{cor:p2-sep} is that in the separable regime in the case of $p=2$, the standard accuracy does not depend on $\eps_0$. In other words, adversarial training has no effect on the performance on benign unperturbed data. The robust accuracy, however is decreasing in $\eps_0$. Figure \ref{fig:SA-RA-p21} (a) and (b) also verify this predicted behavior and capture the effect of the adversary's power $\eps_0$ on standard and robust accuracy. In the separable regime, $\SR$ is flat which implies that adversarial training has no effect on standard accuracy (or the generalization error on unperturbed data). 
 However, adversarial training does affect $\AR$ because now the trained model is used to classify the adversarially perturbed test data. 
 
In the non-separable regime, we observe that adversarial training helps with improving the standard accuracy! Further, such positive impact is observed for all choice of $\delta$ with a rather robust trend. Note that this behavior is significantly different from a regression setting where adversarial training first improves with the standard accuracy but then there is a turning point beyond which the standard accuracy will decrease as $\eps_0$ grows. We refer to~\cite[Figure 3]{pmlr-v125-javanmard20a} and discussion therein for more details on a regression setting.  Moreover, as depicted in Figure~\ref{fig:SA-RA-p21}(b) we see that $\AR$ always declines as adversary gets more powerful (i.e., $\eps_0$ grows) as expected. 

\begin{figure}[t]
\centering
\begin{minipage}{.485\textwidth}
  \centering
    \includegraphics[scale=0.7]{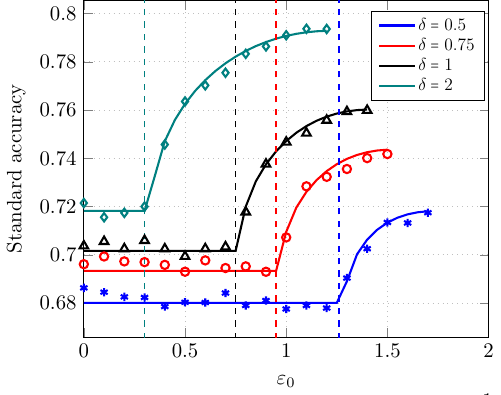}
   \caption*{(a) Standard accuracy}
\end{minipage}
\begin{minipage}{.485\textwidth}
  \centering
  \includegraphics[scale=0.7]{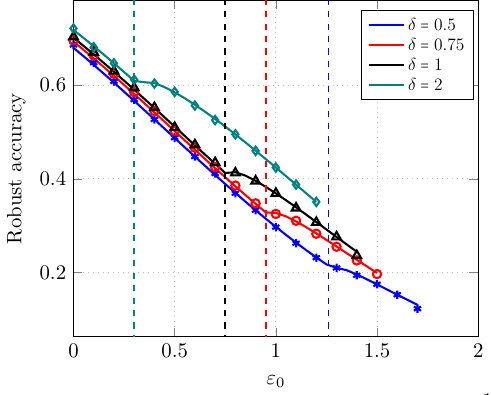}
   \caption*{(b) Robust accuracy}
\end{minipage}
\caption{Depiction of standard and robust accuracies as a function of the adversary's normalized power $\eps_0$ with $\ell_2$ ($p=2$) perturbation for different values of $\delta$. Solid curves are theoretical predictions and dots are results obtained based on gradient descent on the robust objective \eqref{eq:inner}. The dashed lines depict the separability threshold for that $\delta$. Each dot represents the average of $100$ trials. The data set is generated according to a Gaussian Mixture Model per Section \ref{sec:data} with $\vct{\mu}\in\R^d$ consisting of i.i.d.~$\mathcal{N}\left(0,\frac{1}{d}\right)$ entries with dimension $d=400$.} 
\label{fig:SA-RA-p21}
\end{figure}

\begin{figure}[t]
\centering
\begin{minipage}{.485\textwidth}
  \centering
    \includegraphics[scale=0.7]{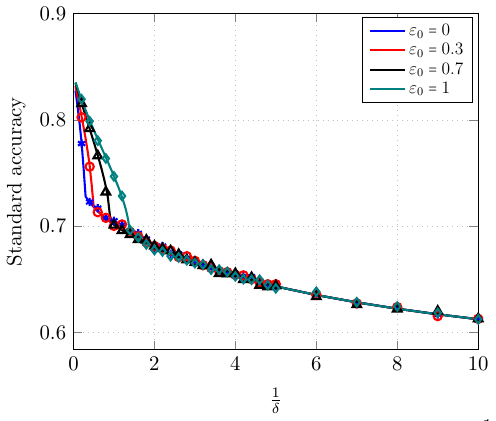}
   \caption*{(a) Standard accuracy}
\end{minipage}
\begin{minipage}{.485\textwidth}
  \centering
  \includegraphics[scale=0.7]{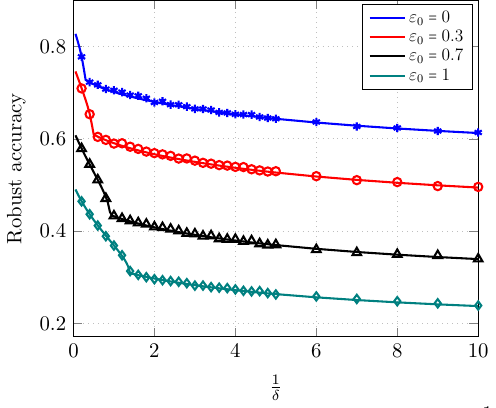}
   \caption*{(b) Robust accuracy}
\end{minipage}
\caption{Depiction of standard and robust accuracies as a function of dimension-to-sample ratio $\frac{1}{\delta} = \frac{d}{n}$, which is a measure of model complexity, for several values of $\eps_0$ with $\ell_2$ ($p=2$) perturbation, under a similar setting as in Figure~\ref{fig:SA-RA-p21}.}
\label{fig:SA-RA-p2-f}
\end{figure}

Next in Figure \ref{fig:SA-RA-p2-f}, we plot $\SR$ and $\AR$ versus dimension-to-sample ratio $\frac{1}{\delta} = \frac{d}{n}$, which is a measure of model complexity, for several values of $\eps_0$. It has been shown that the standard risk (which amounts to $1-\SR$ in our setting) as a function of model complexity $\frac{1}{\delta}$ undergoes a double-descent behavior for various learning models~\cite{belkin2018understand,belkin2018reconciling,hastie2019surprises}.  Specifically, the risk depicts a U-shape before the interpolation threshold (separability threshold in binary classification) and then starts to decline afterwards. Interestingly, for the current setting of experiments here we do not observe such double descent behavior and the standard accuracy always decreases as $\frac{1}{\delta}$ grows, albeit at different rates in the separable and non-separable regimes.\footnote{It is worth noting that the double descent phenomenon has been observed for binary classification in a non-adversarial setting with model misspecification. In such a model the learner observes only a subset $S\subset [d]$ of size $p$ of the covariates with $d/n \to \zeta\ge1$ and $p/n\to \kappa \in (0,\zeta]$ (see ~\cite{deng2019model} for further details). Our theoretical analysis can in principle be used to analyze such a setting, however we do not pursue this direction in this paper.}

\subsection{Results for $\ell_\infty$ perturbation}
\label{secinf}
For the case of $\ell_\infty$ perturbation ($p=\infty$ and $q=1$), Theorem \ref{thm:sep-thresh}, Theorem~\ref{thm:isotropic-separable}, and Theorem \ref{thm:isotropic-nonseparable} do not substantially simplify. However, we can calculate the function $J_q$ defined by~\eqref{eq:Jq} in closed form. In this case $J_q$ becomes the Huber function given by
\[
J_1(x,\lambda) = \begin{cases}
\lambda |x| -\frac{\lambda^2}{2} & |x|\ge \lambda\\
\frac{x^2}{2}& |x|\le \lambda
\end{cases}
\]
Using Theorems \ref{thm:sep-thresh}, \ref{thm:isotropic-separable}, and \ref{thm:isotropic-nonseparable} with this closed form for $J_q$, in Figure \ref{fig:SA-RA-pinf}, we again depict our theoretical predictions for standard and robust accuracy as well as the empirical performance of gradient descent as a function of the adversary's normalized power for various values of $\delta$. As in the $p=2$ case our theoretical predictions is very accurate even for moderate dimensions $d$. 

\begin{figure}[t]
\centering
\begin{minipage}{.485\textwidth}
  \centering
    \includegraphics[scale=0.73]{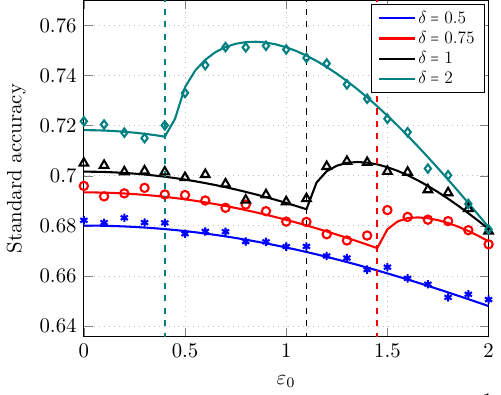}
   \caption*{(a) Standard accuracy}
\end{minipage}
\begin{minipage}{.485\textwidth}
  \centering
  \includegraphics[scale=0.73]{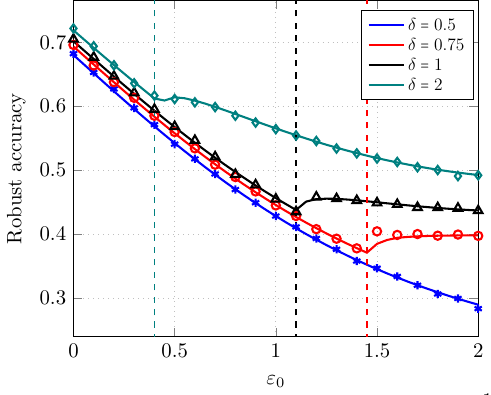}
   \caption*{(b) Robust accuracy}
\end{minipage}
\caption{Depiction of standard and robust accuracies as a function of $\eps_0$ with $\ell_\infty$ ($p=\infty$) perturbation for different values of $\delta$, and under a similar setting as in Figure~\ref{fig:SA-RA-p21}.} 
\label{fig:SA-RA-pinf}
\end{figure} 

More specifically, Figure \ref{fig:SA-RA-pinf}(a) depicts the standard accuracy ($\SR$) versus the adversary's normalized power. Similar to our $p=2$ results the data set is generated according to a Gaussian Mixture Model per Section \ref{sec:data} with $\vct{\mu}\in\R^d$ consisting of i.i.d.~$\mathcal{N}\left(0,\frac{1}{d}\right)$ entries with dimension $d=400$ and each data points represents the average of $100$ trials. In the case of $p=\infty$ however, we do not use the scaling $\eps=\eps_0\infnorm{\vct{\mu}}$ as $\infnorm{\sqrt{d}\vct{\mu}}$ grows with $\sqrt{\log d}$ and therefore violates Assumption \ref{ass:alternative}. Instead we shall use a slightly different scaling of $\eps=\frac{\eps_0}{\sqrt{d}}$. In the separable regime, we see that adversarial training hurts the standard accuracy. However, in the non-separable regime, the standard accuracy starts increasing indicating that adversarial training is improving the standard accuracy. Furthermore, after some value of $\eps_0$, which interestingly shifts with $\delta$, the standard accuracy starts to go down as $\eps_0$ grows.\footnote{Note that for $\delta = 0.5$, we are in the separable regime over the entire range $[0,\eps_0]$.} We note that this behavior is rather counterintuitive and very different from the $p=2$ case, further highlighting the need for a precise theory that can predict such nuanced behavior. Figure \ref{fig:SA-RA-pinf}(b) shows the robust accuracy $\AR$ versus  $\eps_0$ for various values of $\delta$. In the separable regime, we observe a similar trend for all $\delta$, namely $\AR$ decreases at an almost linear rate. In the non-separable regime though we have different trends depending on the value of $\delta$. 

\begin{figure}[t]
\centering
\begin{minipage}{.485\textwidth}
  \centering
    \includegraphics[scale=0.73]{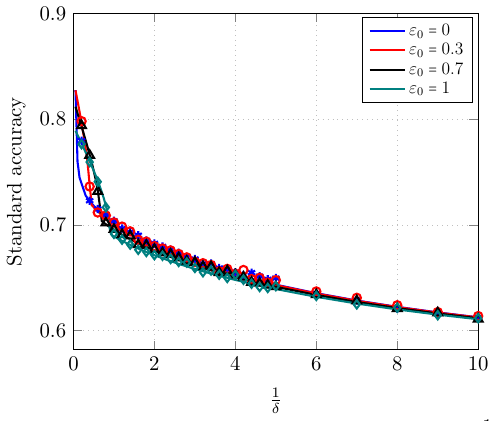}
   \caption*{(a) Standard accuracy}
\end{minipage}
\begin{minipage}{.485\textwidth}
  \centering
  \includegraphics[scale=0.73]{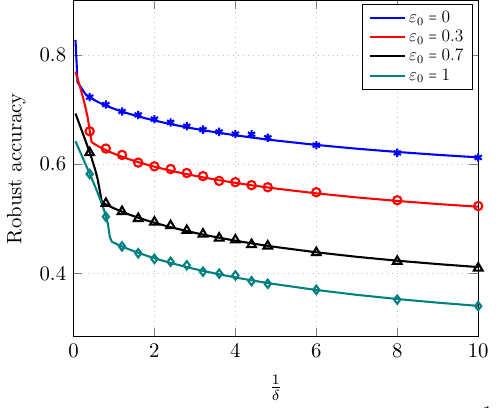}
   \caption*{(b) Robust accuracy}
\end{minipage}
\caption{Depiction of standard and robust accuracies as a function of dimension-to-sample ratio $\frac{1}{\delta} = \frac{d}{n}$, which is a measure of model complexity, for several values of $\eps_0$ with $\ell_\infty$ ($p=\infty$) perturbation, under a similar setting as in Figure~\ref{fig:SA-RA-p21}.}
\label{fig:SA-RA-pinf-f}
\end{figure} 

Finally, in Figure \ref{fig:SA-RA-pinf-f} we depict the effect of overparameterization $\frac{1}{\delta}$ on $\SR$ and $\AR$. We observe a similar pattern as in the case of $p=2$. In particular, we do not observe a double descent behavior and the standard accuracy always decreases as $\frac{1}{\delta}$ grows, albeit at different rates in the separable and non-separable regimes.

%
%

\subsection{Results for $\ell_1$ perturbation}
\label{extentsec}
Our characterization of $\SR$ and $\AR$ given by Theorem~\ref{thm:isotropic-separable}, for separable regime, and by Theorem~\ref{thm:isotropic-nonseparable}, for non-separable regime involve the function $\EJ$ defined by~\eqref{eq:EJ} which in turn depends on the function $J_q$ given by~\eqref{eq:Jq}. However, $J_q$ is only defined for finite $q$ and therefore the case of $p=1$, $q=\infty$ is not directly covered by our results in Section~\ref{sec:main-isotorpic}.  That said, a very similar analysis can be used to characterize $\SR$ and $\AR$ in this case. We formalize our results for this case in the next theorem.

%

\begin{theorem}\label{thm:isotropic-qinf}
Consider a data set generated i.i.d.~according to an isotropic Gaussian mixture data model per Section \ref{sec:data} and suppose the mean vector $\vct{\mu}$ obeys Assumptions~\ref{ass:alternative}.  Also define
\begin{align}\label{eq:f-def}
f(c_0,c_1;t_0) = \frac{1}{2}\E\left[\ST\left(\frac{c_0}{\sqrt{\delta}} h - c_1 \frac{M}{\sigma_{M,2}}; \frac{t_0}{\sigma_{M,1}} \right)^2\right]\,,
\end{align}
where $\ST(x;a):=\sgn{x}\left(\abs{x}-a\right)_{+}$ is the soft-thresholding function. Then in the asymptotic setting of Assumption \ref{ass:asymptotic} we have:
\begin{itemize}
\item [(a)] {The separability threshold $\delta_*$ is given by
\begin{align}
&\delta_* :=  \min_{\alpha\ge0, \theta} \frac{ \omega\left(\alpha, \theta, \eps_0\right)^2}{\E\left[\left(1-V\theta+\sqrt{\alpha^2+\theta^2}g\right)_{+}^2\right]}\nn\\
&\text{with}\quad \omega\left(\alpha, \theta, \eps_0\right) := \min_{\eta\ge0, \nu}  \;\; \sqrt{\delta} \left\{\frac{\nu^2}{2\eta}+\frac{1}{2\eta \delta} + \frac{\eta}{2}\alpha^2 -\eta f\bigg(\frac{1}{\eta},
 \frac{\nu}{\eta}-\theta;\frac{1}{\eps_0}\bigg)\right\} \,.\label{eq:omega-p1}
\end{align}
}
\item [(b)] In the separable regime where $\delta<\delta_*$, the followings hold in probability for the max margin solution $\tth^\eps$  (see ~\eqref{eq:MM}): 
\begin{align}
\lim_{n\to\infty} \SR(\tth^\eps) &= \Phi\left(\sigma_{M,2} \frac{\theta_*}{\alpha_*}\right)\,,\\
\lim_{n\to\infty} \AR(\tth^\eps) &= \Phi\left(-\frac{\eps_0\gamma_{0*}}{\alpha_*} + \sigma_{M,2}\frac{\theta_*}{\alpha_*}\right)\,.
\end{align}
Here, $(\alpha_*,\gamma_{0*},\theta_*)$ are the unique minimization component of the following convex-concave minimax scalar optimization with bounded solution $(\alpha_*,\gamma_{0*},\theta_*,\beta_*,\eta_*,\tilde{\eta}_*)$.
\begin{align}
\min_{\alpha,\gamma_0\ge 0,\theta} \max_{\beta,\eta\ge0, \tilde{\eta}} \quad &D_{\rm s}(\alpha,\gamma_{0},\theta,\beta,\eta,\tilde{\eta}), \quad \text{where}\nn\\
D_{\rm s}(\alpha,\gamma_{0},\theta,\beta,\lambda_{0},\eta,\tilde{\eta}) \;=\; 
&\min_{\alpha,\gamma_0\ge 0,\theta} \max_{\beta,\eta\ge0, \tilde{\eta}} \quad
\frac{1}{2(1+\frac{\eta}{2\alpha})} f\left( \beta, \tilde{\eta}; 2\gamma_0\left(1+\frac{\eta}{2\alpha}\right)\right) \nn\\
&\quad\quad\quad\quad\quad\quad-\left(\frac{\beta^2}{\delta}+\tilde{\eta}^2\right) \frac{1}{4(1+\frac{\eta}{2\alpha})} -\frac{\eta\alpha}{2} -\tilde{\eta}\theta \nn\\
&\quad\quad\quad\quad\quad\quad  +\beta\sqrt{\E\Bigg[\left(\left(1+\eps_0\gamma_0-\theta\sigma_{M,2} \right)+\alpha g\right)_{+}^2\Bigg]}\,,
\end{align}
with expectation taken with respect to $g\sim\normal(0,1)$.
\item[(c)]  In the non-separable regime where $\delta>\delta_*$, the followings hold in probability the optimal solution $\hth^\eps$ of~\eqref{eq:inner}: 
\begin{align}
\lim_{n\to\infty} \SR(\hth^\eps) &= \Phi\left(\frac{\sigma_{M,2}\theta_*}{\sqrt{\alpha_*^2+\theta_*^2}}\right)\,,\\
\lim_{n\to\infty} \AR(\hth^\eps) &= \Phi\left(\frac{-\eps_0\gamma_{0*} + \sigma_{M,2}\theta_*}{\sqrt{\alpha_*^2+\theta_*^2}}\right)\,.
\end{align}
Here, $(\alpha_*,\gamma_{0*},\theta_*)$ are the unique minimization components of the following convex-concave minimax scalar optimization with bounded solution $(\theta_*,\alpha_*,\gamma_{0*},\tau_{g*},\beta_*,\tau_{h*})$.
\begin{align}
\min_{\theta, 0\le \alpha, \gamma_0, \tau_g}\;\;\max_{0\le\beta, \tau_h} \;\; &D_{\rm ns}(\alpha, \gamma_0, \theta, \tau_g, \beta,\tau_h)\nn\\
 D_{\rm ns}(\alpha, \gamma_0, \theta, \tau_g, \beta,\tau_h) &= \frac{\beta \tau_g}{2} +L\left(\sqrt{\alpha^2+\theta^2},\sigma_{M,2} \theta-\eps_0 \gamma_0,\frac{\tau_g}{\beta}\right) \nn \\
&- \min_{\nu} \left[ \frac{\alpha}{\tau_h}\left\{\frac{\beta^2}{2\delta}  + \frac{\nu^2}{2}
 - f\left(\beta , \frac{\tau_h\theta}{\alpha} 
+ {\nu}; \frac{\gamma_0\tau_h}{\alpha}\right) 
\right\} 
+\frac{\alpha\tau_h}{2} 
 \right]\,.
\end{align}

\end{itemize}

\end{theorem}

\begin{figure}[]
\centering
\begin{minipage}{.485\textwidth}
  \centering
    \includegraphics[scale=0.73]{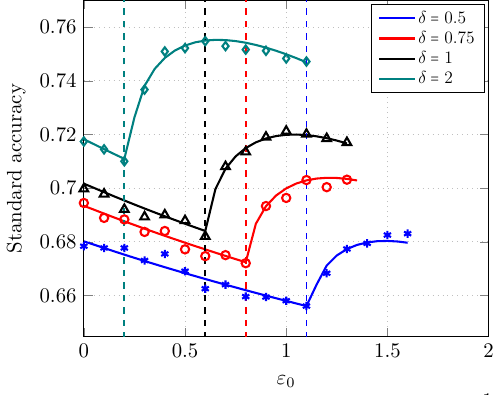}
   \caption*{(a) Standard accuracy}
\end{minipage}
\begin{minipage}{.485\textwidth}
  \centering
  \includegraphics[scale=0.73]{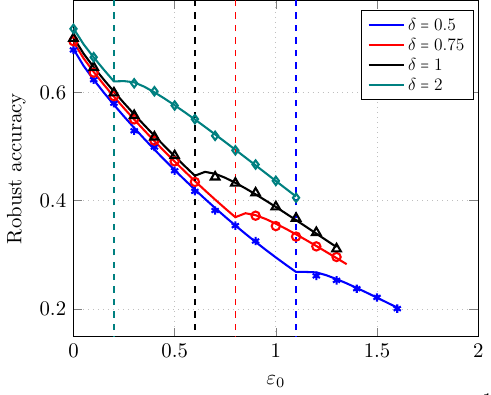}
   \caption*{(b) Robust accuracy}
\end{minipage}
\caption{Depiction of standard and robust accuracies as a function of $\eps_0$ with $\ell_1$ ($p=1$) perturbation for different values of $\delta$,
under a similar setting as in Figure~\ref{fig:SA-RA-p21}.}
\label{fig:SA-RA-p1}
\end{figure}

In Figure \ref{fig:SA-RA-p1}, we again depict our theoretical predictions for standard and robust accuracy as well as the empirical performance of gradient descent as a function of the adversary's normalized power for various values of $\delta$. We note however that in this case we do not actually run gradient descent in our simulations as $p=1$ corresponds to $q=+\infty$ and GD convergence is extremely slow since the gradient only has one non-zero entry. Therefore, for our empirical simulations we use CVX, a package for specifying and solving convex programs~\cite{cvx}, in the non-separable regime which given the uniqueness of the global optima yields the same answer as GD. Similarly, in the separable regime we use \eqref{eq:MM} which based on Proposition \ref{pro:SVM} is the direction GD eventually converges to. We observe that as in the $p=2$ and $p=+\infty$ cases our theoretical predictions are very accurate even for moderate dimensions $d$. 

More specifically, Figure \ref{fig:SA-RA-p1}(a) depicts the standard accuracy ($\SR$) versus the adversary's normalized power. Similar to our $p=+\infty$ results the data set is generated according to a Gaussian Mixture Model per Section \ref{sec:data} with $\vct{\mu}\in\R^d$ consisting of i.i.d.~$\mathcal{N}\left(0,\frac{1}{d}\right)$ entries with dimension $d=400$ and each data points represents the average of $100$ trials. In the separable regime, we see that adversarial training hurts the standard accuracy. However, in the non-separable regime, the standard accuracy starts increasing indicating that adversarial training is improving the standard accuracy. Furthermore, after some value of $\eps_0$, which interestingly shifts with $\delta$, the standard accuracy starts to go down as $\eps_0$ grows.\footnote{Note that for $\delta = 0.5$, we are in the separable regime over the entire range $[0,\eps_0]$.} We note that this behavior is rather counterintuitive and very different from the $p=2$ case but somewhat similar to the $p=+\infty$ case. This again highlights the need for a precise theory that can predict such nuanced behavior. Figure \ref{fig:SA-RA-p1}(b) shows the robust accuracy $\AR$ versus  $\eps_0$ for various values of $\delta$. In the separable regime, we observe a similar trend for all $\delta$, namely $\AR$ decreases at an almost linear rate. In the non-separable regime though we have different trends depending on the value of $\delta$. 


\section{Extension to anisotropic Gaussians}\label{sec:extension-anisotropic}
In this section we extend our results to Gaussian distributions with general covariance matrices that obey a certain spiked covariance assumption stated below.
\begin{assumption} \label{spikedcov}
(Spiked covariance) $\vct{\mu}$ is an eigenvector of $\mtx{\Sigma}$ with eigenvalue $a^2$, i.e, $\mtx{\Sigma} \vct{\mu} = a^2\vct{\mu}$.
\end{assumption}
Similar spiked covariance models have been used to model data in a number of statistical problems,  including matrix denoising and structured learning~\cite{johnstone2001distribution,donoho2018optimal}, sparse PCA \cite{deshpande2014sparse}, synchronization and clustering \cite{javanmard2016phase}.

To extend our results in Section~\ref{sec:main-isotorpic} to the anisotropic case we also need to generalize the definition of the set $\cS$ as follows:
\[
\cS(\alpha, \theta, \eps_0,\vct{\mu}) : = \left\{\vct{z}\in \reals^d :\quad  \bz^T\tmu =0,\; \twonorm{\vct{z}} =\alpha,\; \qnorm{\mtx{\Sigma}^{-1/2}\vct{z} +\theta \tmu} \le \frac{1}{\eps_0\pnorm{\vct{\mu}}} \right\}\,.
\]
We are now ready to state our main results in the anisotropic case. We start by the separability threshold which generalizes Theorem~\ref{thm:sep-thresh}.
\begin{theorem}\label{thm:sep-thresh-aniso}
Consider a data set generated i.i.d.~according to an anisotropic Gaussian mixture data model per Section \ref{sec:data} with a spiked covariance per Assumption \ref{spikedcov}. Also suppose the mean vector $\vct{\mu}$ and covariance matrix $\mtx{\Sigma}$ obey Assumptions~\ref{ass:norm-mu} and \ref{ass:omegas}. Also define
\begin{align}\label{eq:threshold2}
\delta_* :=  \min_{\alpha\ge0, \theta} \frac{ \omega\left(\alpha, \theta, \eps_0\right)^2}{\E\left[\left(1-V\theta+\sqrt{\alpha^2+a^2\theta^2}g\right)_{+}^2\right]}\,,
\end{align}
where the expectation is taken with respect to $g\sim \normal(0,1)$. Then, under the asymptotic setting of Assumption \ref{ass:asymptotic}, for $\delta< \delta_*$ the data are $(\eps,q)$- separable with high probability and for $\delta>\delta_*$, the data are non-separable, with high probability. Namely,
\begin{align*}
\delta < \delta_* &\Rightarrow  \lim_{n\to\infty} \prob(\text{data is $(\eps,q)$-separable}) =1\,,\nn\\
\delta > \delta_* &\Rightarrow  \lim_{n\to\infty} \prob(\text{data is $(\eps,q)$-separable}) = 0\,.\nn
\end{align*}
\end{theorem}
Our  next theorem precisely characterizes $\SR$ and $\AR$ in the separable regime and generalizes Theorem~\ref{thm:isotropic-separable} to the anisotropic case. Before proceeding to state the theorem we need to establish some definitions and assumptions.
\begin{definition}\label{def:WME}
For a given matrix $\mtx{A}\succeq 0$ and a function $f$, we define the \emph{weighted Moreau envelope} of $f$ as follows:
\[
e_{f,\mtx{A}}(\vct{x};\lambda) : = \min_{\vct{v}}  \frac{1}{2}\|\vct{x}-\vct{v}\|_{\mtx{A}}^2 + \lambda f(\vct{v})
\]
When $\mtx{A} = \mtx{I}$, we recover the (scaled) classical Moreau envelope. We denote by $e_{q,\mtx{\Sigma}}$ the weighted Moreau envelope corresponding to $\qnorm{\cdot}^q$ function.
\end{definition}

\begin{assumption}\label{ass:converging2}
For the sequence of instances $\{\mtx{\Sigma}(n), \vct{\mu}(n), d(n)\}_{n\in\naturals}$ indexed by $n$, we assume that:
\begin{itemize}
 \item[(a)] The following (in probability) limit exists for any scalars $c_0, c_1,\lambda_0, \eta\in \reals_{+}$:
\begin{align*}
\sF(c_0,c_1;b_0,b_1): = \lim_{n\to\infty}  e_{q,\mtx{I}+b_0\mtx{\Sigma}}\left((\mtx{I}+ b_0\mtx{\Sigma})^{-1} \left\{\frac{c_0}{2\sqrt{n}} \mtx{\Sigma}^{1/2} \pproj_{\vct{\mu}}\vct{h} - 
 \frac{c_1}{2} \tmu\right\};b_1 \pnorm{\vct{\mu}}^q\right)\,.
\end{align*}
\item[(b)] The empirical distribution of eigenvalues of $\mtx{\Sigma}$ converges weakly to a distribution $\rho$ with Stieltjes transform $S_{\rho}(z): = \int \frac{\rho(t)}{z-t}\de t$.
\end{itemize}
\end{assumption}
With these definitions and assumptions in place we are ready to state our result in the separable regime.
\begin{theorem}\label{thm:isotropic-separable-ansio}
Consider a data set generated i.i.d.~according to an anisotropic Gaussian mixture data model per Section \ref{sec:data} with a spiked covariance per Assumption \ref{spikedcov}. Also suppose the mean vector $\vct{\mu}$ and covariance matrix $\mtx{\Sigma}$ obey Assumptions~\ref{ass:norm-mu}, \ref{ass:omegas}, and \ref{ass:converging2}. Also let $\tth^\eps$ be the max margin solution per~\eqref{eq:MM}. If $\delta<\delta_*$, with $\delta_*$ given by~\eqref{eq:threshold}, then in the asymptotic setting of Assumption \ref{ass:asymptotic} we have:
\begin{itemize}
\item[(a)] The following convex-concave minimax scalar optimization problem has bounded solution $(\alpha_*,\gamma_{0*},\theta_*,\beta_*,\lambda_{0*},\eta_*,\tilde{\eta}_*)$ with the minimization components $(\alpha_*,\gamma_{0*},\theta_*)$ unique:
\begin{align}\label{eq:sep-AO92}
\min_{\alpha,\gamma_0\ge 0,\theta} \max_{\beta,\lambda_0,\eta\ge0, \tilde{\eta}} \quad &D_{\rm s}(\alpha,\gamma_{0},\theta,\beta,\lambda_{0},\eta,\tilde{\eta}), \quad \text{where}\nn\\
D_{\rm s}(\alpha,\gamma_{0},\theta,\beta,\lambda_{0},\eta,\tilde{\eta}) \;=\; &2 \sF\left(\beta,\tilde{\eta};\frac{\eta}{2\alpha},\frac{\lambda_0}{q\gamma_0^{q-1}}\right)-\frac{\beta^2\alpha}{2\delta \eta} \left(1 +\frac{2\alpha}{\eta}S_{\rho}\left(-\frac{2\alpha}{\eta}\right)\right) \nn\\
&-\frac{2\lambda_0}{q}  \gamma_0 -\frac{\eta\alpha}{2} -\tilde{\eta}\theta \nn\\
&-\frac{\tilde{\eta}^2}{4(1+\frac{\eta}{2\alpha} a^2)}  +\beta\sqrt{\E\Bigg[\left(\left(1+\eps_0\gamma_0-\theta V \right)+\alpha g\right)_{+}^2\Bigg]}\,,
\end{align}
with expectation in last part taken with respect to $g\sim\normal(0,1)$.
\item[(b)] It holds in probability that
\begin{align}
\lim_{n\to\infty} \frac{1}{\twonorm{\vct{\mu}}}\<\vct{\mu},\tth^\eps\> &= \theta_*\,,\\
\lim_{n\to\infty} \twonorm{\tth^\eps} &= \alpha_*\,,\\
\lim_{n\to\infty} \pnorm{\vct{\mu}}\qnorm{\tth^\eps} &= \gamma_{0*}\,.
\end{align}
\item[(c)] As a corollary of part (b) and Lemma~\ref{lem:SR-AR}, the following limits hold in probability:
\begin{align}
\lim_{n\to\infty} \SR(\tth^\eps) &= \Phi\left(V \frac{\theta_*}{\alpha_*}\right)\,,\\
\lim_{n\to\infty} \AR(\tth^\eps) &= \Phi\left(-\frac{\eps_0\gamma_{0*}}{\alpha_*} + V\frac{\theta_*}{\alpha_*}\right)\,.
\end{align}
\end{itemize}
\end{theorem}
Next we turn our attention to characterizing $\SR$ and $\AR$ on the non-separable regime.  To state result we need an additional assumption
\begin{assumption}\label{ass:converging}
For the sequence of instances $\{\mtx{\Sigma}(n), \vct{\mu}(n), p(n)\}_{n\in\naturals}$ indexed by $n$, we assume that the following (in probability) limit exists for any scalars $c_0, c_1\in \reals_{+}$ and $\lambda_0 \in \reals$:
\begin{align}
\sE(c_0,c_1;\lambda_0): = \lim_{n\to\infty}  e_{q,\mtx{\Sigma}}\left(\frac{c_0}{\sqrt{n}}\mtx{\Sigma}^{-1/2}\vct{h}- c_1 \tmu; \lambda_0 \pnorm{\vct{\mu}}^q\right)\,,
\end{align}
where we recall $\tmu =\bmu/\twonorm{\bmu}$.
\end{assumption}
Our next theorem generalizes Theorem~\ref{thm:isotropic-nonseparable} to anisotropic case. 
\begin{theorem}\label{thm:isotropic-nonseparable-aniso}
Consider a data set generated i.i.d.~according to an anisotropic Gaussian mixture data model per Section \ref{sec:data} with a spiked covariance per Assumption \ref{spikedcov}. Also suppose the mean vector $\vct{\mu}$ and covariance matrix $\mtx{\Sigma}$ obey Assumptions~\ref{ass:norm-mu}, \ref{ass:omegas}, and \ref{ass:converging}. Also let $\hth^\eps$ be the solution to optimization~\eqref{eq:inner}. If $\delta>\delta_*$, with $\delta_*$ given by~\eqref{eq:threshold}, then in the asymptotic setting of Assumption \ref{ass:asymptotic} we have:
\begin{itemize}
\item[(a)] 
The following convex-concave minimax scalar optimization problem has bounded solution $(\theta_*,\alpha_*,\gamma_{0*},\tau_{g*},\beta_*,\tau_{h*})$ with the minimization components $(\alpha_*,\gamma_{0*},\theta_*)$ unique:
%
\begin{align}
&\min_{\theta, 0\le \alpha, \gamma_0, \tau_g}\;\;\max_{0\le\beta, \tau_h} \;\; D_{\rm ns}(\alpha, \gamma_0, \theta, \tau_g, \beta,\tau_h)\nn\\
& D_{\rm ns}(\alpha, \gamma_0, \theta, \tau_g, \beta,\tau_h) = \frac{\beta \tau_g}{2} +L\left(\sqrt{\alpha^2+a^2\theta^2},V \theta-\eps_0 \gamma_0,\frac{\tau_g}{\beta}\right) \label{eq:AO-gen-ns} \nn \\
&\quad\quad\quad- \min_{\lambda_0\ge0,\nu} \left[\frac{\alpha}{\tau_h}\left\{\frac{\beta^2}{2\delta} + \lambda_0 \left(\frac{\gamma_0\tau_h}{\alpha}\right)^q  + \frac{\nu^2}{2}
-\sE\left(\beta,\left(\frac{\tau_h\theta}{\alpha} + \frac{\nu}{a}\right) ;\lambda_0\right)\right\} +\frac{\alpha\tau_h}{2} \right].\end{align}
\item[(b)] It holds in probability that
\begin{align}
\lim_{n\to\infty} \frac{1}{\twonorm{\vct{\mu}}}\<\vct{\mu},\hth^\eps\> &= \theta_*\,,\\
\lim_{n\to\infty} \twonorm{\pproj_{\vct{\mu}}  \hth^\eps} &= \alpha_*\,,\\
\lim_{n\to\infty} \pnorm{\vct{\mu}}\qnorm{\hth^\eps} &= \gamma_{0*}\,.
\end{align}
\item[(c)] As a corollary of part (b) and Lemma~\ref{lem:SR-AR}, the following limits hold in probability:
\begin{align}
\lim_{n\to\infty} \SR(\hth^\eps) &= \Phi\left(\frac{V\theta_*}{\sqrt{\alpha_*^2+ a^2\theta_*^2}}\right)\,,\\
\lim_{n\to\infty} \AR(\hth^\eps) &= \Phi\left(\frac{-\eps_0\gamma_{0*} + V\theta_*}{\sqrt{\alpha_*^2+a^2\theta_*^2}}\right)\,.
\end{align}
\end{itemize}
\end{theorem}
The results above generalize out results to the anisotropic case. The reader may of course be wondering when Assumptions~\ref{ass:converging2} and \ref{ass:converging} hold. This is the subject of the next Remark which we prove in Appendix~\ref{proof:rem:p-q-2-Sigma}.
\begin{remark}\label{rem:p-q-2-Sigma}
For the case of $\ell_2$ perturbation ($p=q=2$), the following two conditions are sufficient for  Assumption~\ref{ass:converging2} and \ref{ass:converging} to hold:
\begin{itemize}
\item[$(i)$]The empirical distribution of the entries of $\sqrt{d} \vct{\mu}$ converges weakly to a distribution $\prob_{M}$ on real line, with bounded second moment, i.e. $\int x^2 \de\prob_M(x) = \sigma_{M,2}^2<\infty$. 
\item[$(ii)$] The empirical distribution of eigenvalues of $\mtx{\Sigma}$ converges weakly to a distribution $\rho$ with Stieltjes transform $S_{\rho}(z): = \int \frac{\rho(t)}{z-t}\de t$.
\end{itemize}
\end{remark}
\section{Proof sketch and mathematical challenges}\label{sec:proofSketch}
Our theoretical results on adversarial training for binary classification fits in the rapidly growing recent literature on developing \emph{sharp} high-dimensional asymptotics of (possibly non-smooth) convex optimization-based estimators \cite{DMM,Sto,montanariLasso,TroppEdge,StoLasso, donoho2013information,OTH13,thrampoulidis2015regularized,karoui2013asymptotic,karoui15,donoho2016high,Master,miolane2018distribution,wang2019does,celentano2019fundamental,hu2019asymptotics,bu2019algorithmic}.  Most of this line of work focus on linear models and regression problems. It has been only recently that the literature witnessed a surge of interest in sharp analysis of a variety of methods tailored to binary classification models \cite{huang2017asymptotic,candes2020phase,sur2019modern,mai2019large,svm_abla,salehi2019impact,taheri2020sharp,deng2019model,montanari2019generalization,liang2020precise,mignacco2020role,taheri2020sharp,lolas2020regularization}. However, none of these papers study adversarial training and its impact on standard/robust accuracies.

  On a technical level, our sharp analysis relies on the Convex Gaussian Min-max Theorem (CGMT) \cite{thrampoulidis2015regularized} (see also \cite{StoLasso,OTH13,Master})), which is  a powerful extension of the Gordon's Gaussian comparison inequality~\cite{gordon1988milman}. We refer to Section~\ref{sec:proofSketch} for an overview of this framework and the mathematical challenges we encounter in applying it to our adversarial setting. We next present a proof sketch for deriving our main results which illustrates the key ideas.

To be able to provide a precise characterization of the various tradeoffs we need to develop a precise understanding of the adversarial training objective 
\begin{align}
\label{temploss}
\underset{\vct{\theta}\in\R^d}{\min}\text{ }\mathcal{L}(\vct{\theta}):=\frac{1}{n} \sum_{i=1}^n  \ell\left(y_i \<\x_i,\bth\> - \eps \qnorm{\bth}\right),
\end{align}
and its optimal solution $\hth^{\eps} \in \arg\min_{\bth\in \reals^d }\mathcal{L}(\vct{\theta})$. Given the classification nature of the problem, as discussed earlier, we have to study this loss in the two different regimes of separable and non-separable as well as characterize the threshold of separability. In this section we wish to provide a brief overview of the steps of our proofs and some of the challenges. We focus our exposition on the non-separable case. While the details of the derivations for the separable case and the calculation of the separability threshold differ from the non-separable case the general steps are similar and therefore the steps below also provides a general road map for the proof of these results as well. Specifically, our proofs in the non-separable regime consists of the following steps:

\smallskip

\noindent\textbf{Step I: Reformulation of the loss.}\\
\noindent The loss \eqref{temploss}, while significantly simplified due to the removal of the max function, is still rather complicated and precisely characterizing the behavior and the quality of its optimal solution is still challenging. In particular, the dependence on the random data matrix $\mtx{X}$ is still rather complex hindering statistical analysis even in an asymptotic setting. To bring the optimization problem into a form more amenable to precise asymptotic analysis we carry out a series of reformulations of the optimization problem. Combining these reformulation steps we arrive at the following equivalent Primal Optimization (PO) problem
\begin{align}
\label{mainPOps}
\min_{\bth, \vct{v}\in \reals^n} \max_{\mtx{\vct{u}\in \reals^n}}
\frac{1}{n} \Big\{\vct{u}^\sT \vct{1}\vct{\mu}^T \bth +\vct{u}^\sT \mtx{D_y} \mtx{Z} \mtx{\Sigma}^{1/2} \bth  -  \vct{u}^\sT \vct{v} \Big\}
+ \frac{1}{n} \sum_{i=1}^n  \ell \left(v_i - \eps \qnorm{\bth}\right)
\end{align}

\noindent\textbf{Step II: Reduction to an Auxiliary Optimization (AO) problem.}\\
The equivalent form above may be counter-intuitive as we started by simplifying a different mini-max optimization problem and we have now again introduced a new maximization! The main advantage of this new form is that it is in fact affine in the data matrix $\bX$. This particular form allows us to use a powerful extension of a classical Gaussian process inequality due to \cite{gordon1988milman} known as Convex Gaussian Minimax Theorem (CGMT) \cite{thrampoulidis2015regularized} which focuses on characterizing the asymptotic behavior of mini-max optimization problems that are affine in a Gaussian matrix $\bX$. This result enables us to characterize the properties of \eqref{temploss} by studying the asymptotic behavior of the following, arguable simpler, \emph{Auxiliary Optimization (AO)} problem instead
\begin{align}
\label{mainAOps}
\min_{\bth, \vct{v}\in \reals^n } \max_{\vct{u}\in \R^n} \frac{1}{n} &\Big\{
 \twonorm{\pproj_{\vct{\mu}}\mtx{\Sigma}^{1/2}\bth} \vct{g}^T \mtx{D_y} \vct{u} +  \twonorm{\mtx{D_y}\vct{u}} \vct{h}^T\pproj_{\vct{\mu}} \mtx{\Sigma}^{1/2}\bth  \nn\\
 &+(\vct{u}^\sT \mtx{D_y}\vct{z})(\tmu^T\mtx{\Sigma}^{1/2}\vct{\theta})+  \vct{u}^\sT \vct{1}\vct{\mu}^T \bth -  \vct{u}^\sT \vct{v} \Big\}
+ \frac{1}{n} \sum_{i=1}^n  \ell \left(v_i - \eps \qnorm{\bth}\right)\,, 
\end{align}
 where $\vct{g}\sim\normal(0,\mtx{I}_n)$ and $\vct{h}\sim \normal(0,\mtx{I}_{d})$, $\pproj_{\vct{\mu}} := \mtx{I} - \tmu\tmu^T$, and $\proj_{\vct{\mu}} := \tmu\tmu^T$.
 
We emphasize that the relationship between the above PO problem \eqref{mainPOps} and how it is exactly related to the AO problem \eqref{mainAOps} is more intricate and technical compared with classical CGMT and related work in the context of classification \cite{salehi2019impact,taheri2020sharp}. In particular, prior work on binary classification such as \cite{salehi2019impact,taheri2020sharp} via CGMT (which corresponds to the non-robust case i.e.~$\eps=0$) utilize the fact that \eqref{mainPOps} is rotationally invariant and hence one can assume $\vct{\mu}=\vct{e}_1$ without loss of generality. However, in the robust version (unless $p=q=2$) the direction of $\vct{\mu}$ plays a crucial role  due to the regularization term $\frac{1}{n} \sum_{i=1}^n  \ell \left(v_i - \eps \qnorm{\bth}\right)$. 
%
\smallskip

\noindent\textbf{Step III: Scalarization of the Auxiliary Optimization (AO) problem.}\\
In this step we further simplify the AO problem in \eqref{mainAOps}. In particular we show the asymptotic behavior of the AO can be characterized rather precisely via the  scalar optimization problem
\begin{align}\label{eq:AO-gen-nspf}
&\min_{\theta, 0\le \alpha, \gamma_0, \tau_g}\;\;\max_{0\le\beta, \tau_h} \;\; D_{\rm ns}(\alpha, \gamma_0, \theta, \tau_g, \beta,\tau_h):= \frac{\beta \tau_g}{2} +L\left(\sqrt{\alpha^2+a^2\theta^2},V \theta-\eps_0 \gamma_0,\frac{\tau_g}{\beta}\right) \nn \\
&\quad\quad\quad\quad\quad- \min_{\lambda_0\ge0,\nu} \left[\frac{\alpha}{\tau_h}\left\{\frac{\beta^2}{2\delta} + \lambda_0 \left(\frac{\gamma_0\tau_h}{\alpha}\right)^q  + \frac{\nu^2}{2}
-\sE\left(\beta,\left(\frac{\tau_h\theta}{\alpha} + \frac{\nu}{a}\right) ;\lambda_0\right)\right\} +\frac{\alpha\tau_h}{2} \right].\end{align}
 %
More specifically, a variety of conclusions can be derived based on the optimal solutions of the above optimization problem as we discuss in the next step. We note that while this expression may look complicated we prove that this optimization problem is in fact convex in the minimization parameters $(\theta,\alpha,\tau_g)$ and concave in the maximization parameters $(\beta,\tau_h)$ so that its optimal solutions can be easily derived via a simple low-dimensional gradient descent rather quickly and accurately. We also note that this proof is quite intricate and involved, so it is not possible to give an intuitive sketch of the arguments here. We refer to Section \ref{proofs} for details. However, we briefly state a few mathematical challenges that is unique to simplifying \eqref{mainAOps}. First, the AO \eqref{mainAOps} does not have a simple regularization whose scalarization reduces to a simple mean width calculation as in most simple CGMT uses. Instead the regularization has a complicated form $\frac{1}{n} \sum_{i=1}^n  \ell \left(v_i - \eps \qnorm{\bth}\right)$ which requires rather intricate and involved scalarization calculations. Second, this AO regularization term is not separable in $\vct{\theta}$ which significantly complicates the scalarization of the AO. Finally, we handle the case of more general covariance matrices where $\mtx{\Sigma}\neq\mtx{I}$.  


\smallskip
\noindent\textbf{Step IV: Completing the proof of the theorems.}\\
Finally, we utilize the above scalar form to derive all of the different theorems and results. This is done by relating the quantities of interest in each theorem to the optimal solutions of \eqref{eq:AO-gen-nspf}. For instance, we show that
$\lim_{n\to\infty} \SR(\hth^\eps) = \Phi\left(\frac{V\theta_*}{\sqrt{\alpha_*^2+ a^2\theta_*^2}}\right)$ and $\lim_{n\to\infty} \AR(\hth^\eps) = \Phi\left(\frac{-\eps_0\gamma_{0*} + V\theta_*}{\sqrt{\alpha_*^2+a^2\theta_*^2}}\right)$ where $\alpha_*$  and $\theta_*$ are the optimal solutions over $\alpha$ and $\theta$. These calculations/proofs are carried out in detail in Section \ref{proofs}. Since each argument is different we do not provide a summary here and refer to the corresponding sections.

\section{Discussion}
We conclude the paper by discussing some of the potential extensions and applications of our theory as well as comparison with more classical approaches to binary classification.  
\subsection{Generalization to random features models} While our focus in this paper was on linear classifiers, these models are quite foundational and serve as the basis for more complex models. For instance, one potential generalization of our results is to the class of random features models given by
\begin{align*}
\mathcal{F}_{\rm RF} :=\left\{f(\bx;\bth, \bW) = \sign(\<\bth, \sigma(\bW \bx)\>):\quad \bth\in\reals^N \right\}\,,
\end{align*}
where $\bx\in \reals^d$ represents the feature vector, $\bW\in \reals^{N\times d}$ is a random matrix whose rows are chosen uniformly at random from the unit sphere in $d$-dimension, and $\sigma:\reals \mapsto \reals$ is a nonlinear function (for a vector $\bv$, $\sigma(\bv) = (\sigma(v_1), \dotsc, \sigma(v_m))$ is applied entry-wise). Random features model can also be described as a two-layer fully connected neural network with random first-layer weights fixed to $\bW$ and not optimized, while the second layer weights are represented by vector $\bth$ and are optimized over to minimize the loss of interest. The random features model was introduced by~\cite{rahimi2007random} for scaling kernel methods to large datasets and there has been a large body of work drawing connections between random features models, kernel methods and fully trained neural networks~\cite{daniely2016toward,daniely2017sgd,jacot2018neural,li2018learning}.  

An intriguing phenomenon, pointed out by \cite{mei2019generalization,montanari2019generalization} from the analysis of random features model in non-adversarial contexts, is that the random
features model has the same asymptotic behavior as a simpler 
noisy linear features model  whose second order statistics match the nonlinear random features model, namely a linear model with noisy features $\bu\in \reals^N$ given by $\bu = \eta_0 + \eta_1 \bW \bx + \eta_2 \bz$, 
where $\bz$ has i.i.d standard normal entries, independent of $\bW$ and $\bx$. Also, the constants $\eta_0,\eta_1,\eta_2$ depend on the activation function $\sigma(\cdot)$ and are chosen so that the two models have the same first and second moments. A promising direction is to establish a similar connection for an adversarial setting and use our theory (relied on CGMT framework) to analyze the equivalent noisy linear model, from which we obtain an asymptotic characterization for adversarial training under the random features model. 
Very recently and after this paper was posted,~\cite{hassani2022curse} has pursued a similar approach to precisely characterize the role of overparametrization on robust generalization of random features in a regression setting. 

\subsection{Optimal $\eps_0$ for the robust minimax estimator}
An interesting application of our theory is to derive the optimal value $\eps_0^{{\rm op}}$ (perceived perturbation level) in the robust minimax estimator~\eqref{eq:minimaxEst}, while fixing the adversary's (actual) perturbation level on test inputs to $\eps_{0,{\rm test}}$. (See Remark~\ref{rem:2eps} on how our theory applies to this setting.)
The optimality here is with respect to maximizing the robust accuracy. Somewhat surprisingly $\eps_0^{{\rm op}}$ is different than $\eps_{0,{\rm test}}$ in general and depends on $\delta$ and the choice of perturbation norm $\ell_p$ in a non-trivial way (There is no one-fit-all solution and this highlights the importance of having a precise theory to understand the effect of adversarial training which is the primary goal of the current work). For example, in the particular case of $\eps_{0,{\rm test}} = 0$, the question reduces to finding the value of $\eps_0$ which maximizes standard accuracy. As we already discussed, the answer very much depends on $\delta$ and $p$. For $p=2$, we observe that (cf. Figure~\ref{fig:SA-RA-p21}(a)) adversarial training helps with improving the standard accuracy. However for $p=\infty$, $\eps_0^{{\rm op}}$ should be large enough so that the problem becomes non-separable  and also its value decreases as $\delta$ increases (cf. Figure \ref{fig:SA-RA-pinf}(a)). As another example, we consider the case of $\eps_{0,{\rm test}} = 0.3$ with $\ell_\infty$ perturbations. In Figure~\ref{fig:eps_opt_pinf}  we plot the robust accuracy versus $\eps_0$, and the dashed vertical lines show the value of $\eps_0^{{\rm op}}$. As we see its value decreases by increasing $\delta$, however, its exact value requires a precise analysis.
\begin{figure}[]
\centering
    \includegraphics[scale=0.7]{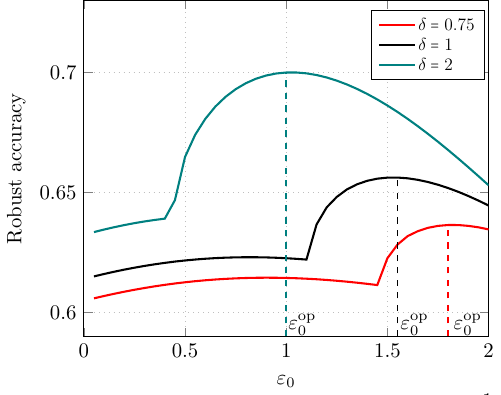}
\caption{Robust accuracy curves versus $\eps_0$ for different choices of $\delta$, and the perturbation norm $\ell_1$ ($p=1$). The optimal choice of $\eps_0$ for the robust minimax estimator decreases with $\delta$.}
\label{fig:eps_opt_pinf}
\end{figure}

\subsection{Comparison with Linear Discriminant Analysis (LDA)} A classical approach to binary classification under the Gaussian-mixture model is the Linear Discriminant Analysis. In comparing the robustness property of LDA and the robust minimax estimator studied in this paper, we cannot say one estimator always outperforms the others. To further discuss this point, we consider the Gaussian-mixture model with identity covariance $\bSigma = \Iden$ and balanced classes. In this case, the LDA estimator reduces to $\LDA = \frac{1}{n}\sum_{i=1}^n y_i \bx_i$ and the corresponding classification rule given by $\hat{y} = \sign(\<\bx, \LDA\>)$. 
In the supplementary~\cite{supp:displ} (Section A), we compare the robust accuracy of LDA estimator with that of the robust minimax estimator $\hth^\eps$ for some choices of $p$. As we will discuss, the depending on $p$ and the adversary's power $\eps_0$, one can outperform the other.

\section*{Acknowledgements}
A. Javanmard is supported in part by a Google Faculty Research Award, an Adobe Data Science Research Award and the NSF CAREER Award DMS-1844481.
M. Soltanolkotabi is supported by the Packard Fellowship in Science
and Engineering, a Sloan Research Fellowship in Mathematics, an NSF-CAREER under award
$\#1846369$, the Air Force Office of Scientific Research Young Investigator Program (AFOSR-YIP)
under award $\#$FA$9550-18-1-0078$, DARPA Learning with Less Labels (LwLL) and FastNICS programs, and NSF-CIF awards $\#1813877$ and $\#2008443$. 
\newpage

%
%
%
%
%
%


\bibliographystyle{amsalpha}
\bibliography{Bibfiles.bib,compbib.bib}
\newpage
\appendix
\section{Comparison with Linear Discriminant Analysis (LDA)}
A classical approach to binary classification under the Gaussian-mixture model is the Linear Discriminant Analysis. In comparing the robustness property of LDA and the robust minimax estimator studied in this paper, we cannot say one estimator always outperforms the others. To further discuss this point, we consider the Gaussian-mixture model with identity covariance $\bSigma = \Iden$ and balanced classes. In this case, the LDA estimator reduces to $\LDA = \frac{1}{n}\sum_{i=1}^n y_i \bx_i$ and the corresponding classification rule given by $\hat{y} = \sign(\<\bx, \LDA\>)$. Under the Gaussian-mixture model we have
$\bx = y \bmu + \bz$ with $\bz\sim\normal(0,\Iden)$. Therefore,
\begin{align*}
\LDA &=  \frac{1}{n}\sum_{i=1}^n y_i (y_i \bmu + \bz_i) = \bmu + \frac{1}{n} \sum_{i=1}^n y_i\bz_i=  \bmu  + \tilde{\bz}\,, \quad \tilde{\bz}\sim\normal({\bf 0},\frac{1}{n}\Iden)
\end{align*}
For simplicity we assume that the class averages $\bmu$ is generated as $\bmu\sim({\bf 0}, \frac{1}{d}\Iden)$, similar to the setting considered in the numerical experiments. In asymptotic regime of $n\to\infty$ and $n/d\to\delta$, we have that in probability:
\begin{align*}
\lim_{n\to\infty} \<\bmu,\LDA\> &= \lim_{n\to\infty} \twonorm{\bmu}^2  = 1\,,\\
\lim_{n\to\infty} d^{1/2-1/q}\qnorm{\LDA} &=  \lim_{n\to\infty} d^{1/2-1/q} \qnorm{\bmu+\tilde{\bz}} \nn\\
&=  \lim_{n\to\infty} d^{1/2-1/q} \left(\frac{1}{d} + \frac{1}{n}\right)^{1/2} d^{1/q} C_q = \left(1+\frac{1}{\delta}\right)^{1/2} C_q\,,
\end{align*}
where in the first equation we used the fact that $\<\bmu, \tilde{\bz}\> \sim\normal(0,\frac{1}{n}\twonorm{\bmu}^2)$ has vanishing variance as $n\to\infty$. In the second inequality, $C_q$ is the $q$-th moment of standard normal distribution. Recall that $\eps = \eps_0 \pnorm{\bmu}$ with $1/p+1/q = 1$, and also $\pnorm{\bmu} \to d^{1/p-1/2} C_p = d^{1/2-1/q} C_p$. Using these identities along with the characterization of standard and robust accuracies given by Lemma \ref{lem:SR-AR} of the paper, we arrive at
\begin{align}
&\lim_{n\to\infty} \SR(\LDA) = \Phi\left(\sqrt{\frac{\delta}{1+\delta}} \right)\,,\nn\\
&\lim_{n\to\infty} \AR(\LDA) = \Phi\left(\sqrt{\frac{\delta}{1+\delta}} - \eps_0 C_q C_p\right)\,.\label{eq:LDA}
\end{align}

We next compare the robust accuracy of LDA estimator with that of the robust minimax estimator $\hth^\eps$ for some choices of $p$. As we will discuss, the depending on $p$ and the adversary's power $\eps_0$, one can outperform the other. 
\begin{itemize}
\item $(p=q=2)$. Figure~\ref{fig:LDA_RA_p2}(a) compares $\AR(\LDA)$ with $\AR(\hth^\eps)$ versus $\eps_0$ for several values of $\delta$. Here, the solid lines 
correspond to the robust minimax estimator and the dashed lines correspond to the LDA estimator.
Figure \ref{fig:LDA_RA_p2}(b) compares $\AR(\LDA)$ with $\AR(\hth^\eps)$ versus $1/\delta$ for various choices of $\eps_0$. As we see for the case of $p=2$, the LDA has better robust accuracy and it is mostly very close to that of the robust estimator.

\item ($p=\infty$, $q=1$). Similar to the setting of experiments in Section~\ref{secinf}, here we consider the scaling $\eps = \eps_0/\sqrt{d}$.  Figure~\ref{fig:LDA_RA_pinf} (a)  compares the robust accuracies versus $\eps_0$ for several values of $\delta$. As we see for any $\delta$, there exists $\eps_0^*(\delta)$ above which the robust minimax outperforms the LDA. 
Figure \ref{fig:LDA_RA_pinf}(b) compares the robust accuracies versus $1/\delta$ for several values of $\eps_0$. Rewording the above observation, for any $\eps_0$ there exists $\delta^*(\eps_0)$ below which the robust minimax outperforms the LDA estimator.

\item ($p=1$, $q=\infty$). Similar to the setting of experiments in Section \ref{extentsec}, we have $\eps = \eps_0 \pnorm{\bmu} = \sqrt{\frac{2}{\pi}} \frac{\eps_0}{\sqrt{d}}$. Invoking equations $\eqref{eq:LDA}$, we have $\lim_{n\to\infty}\AR(\LDA) = 0$ because $C_q  = \sqrt{2\log d} \to\infty$. However, as we see in Figure \ref{fig:SA-RA-p1}, the robust minimax estimator $\hth^\eps$ achieves non-trivial positive robust accuracies and hence outperforms LDA. 
\end{itemize}

\begin{figure}[]
\centering
\begin{minipage}{.485\textwidth}
  \centering
    \includegraphics[scale=0.73]{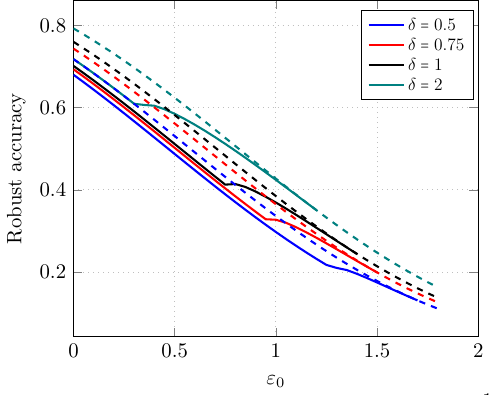}
\end{minipage}
\begin{minipage}{.485\textwidth}
  \centering
  \includegraphics[scale=0.73]{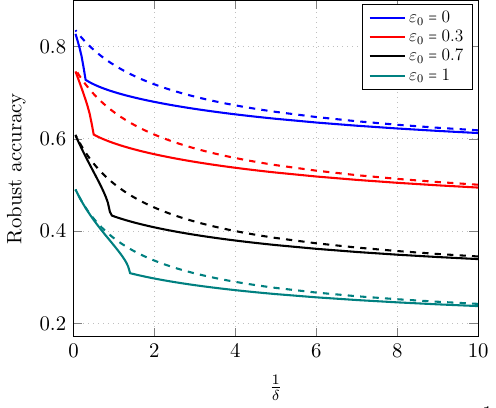}
\end{minipage}
\caption{Robust accuracies for the LDA estimator and the robust minimax estimator versus the adversary's power with $\ell_2$ ($p=2$) perturbations for different values of $\delta$. Solid curves correspond to the robust minimax estimator and the dashed curves correspond to the LDA estimator. }
\label{fig:LDA_RA_p2}
\end{figure} 

\begin{figure}[]
\centering
\begin{minipage}{.485\textwidth}
  \centering
    \includegraphics[scale=0.73]{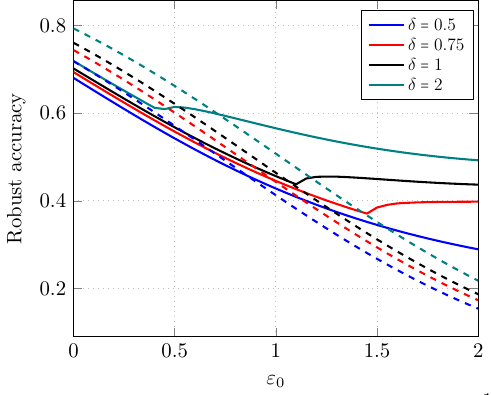}
\end{minipage}
\begin{minipage}{.485\textwidth}
  \centering
  \includegraphics[scale=0.73]{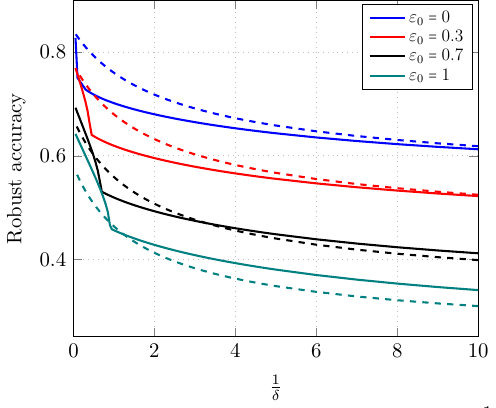}
\end{minipage}
\caption{Robust accuracies for LDA estimator and the robust minimax estimator versus the adversary's power with $\ell_\infty$ ($p=\infty$) perturbations for different values of $\delta$. Solid curves correspond to the robust minimax estimator and the dashed curves correspond to the LDA estimator. }
\label{fig:LDA_RA_pinf}
\end{figure}

\section{Proofs for anisotropic Gaussian model (Section~6)}
\label{proofs}
\subsection{Proof of Theorem~\ref{thm:sep-thresh-aniso}}

As discussed the $(\eps,q)$-separability condition can alternatively be written as~\eqref{eq:serp-def2}, which we repeat here:
\begin{align}
\exists \vct{\theta}, \;\;\qnorm{\bth} \le \frac{1}{\eps}:\quad \forall i\in[n],\;\; y_i\<\vct{x}_i,\vct{\theta}\> > 1. 
\end{align}
 To find the separability threshold we consider the following feasibility problem
 \begin{align}
 \min_{\bth \in \reals^d} 0 \quad \text{subject to } y_i\<\vct{x}_i,\bth\> >1,\quad \qnorm{\bth} \le \frac{1}{\eps}\,.
 \end{align}
Clearly this is a convex optimization problem since $q\ge 1$. Writing the partial Lagrangian for the above problem with $u_i/n$ as dual coefficients, this is equivalent to
\begin{align}
\min_{\vct{\theta}}\text{ }\max_{u_i\ge 0}\quad  \frac{1}{n} \sum_{i=1}^n u_i\left(1- y_i\<x_i,\bth\>\right) \quad\text{subject to}\quad \qnorm{\bth} \le \frac{1}{\eps}\,.
\end{align}
Under our Gaussian Mixture data model, we can substitute for $\mtx{X} = \vct{y}\vct{\mu}^T +\mtx{Z}\mtx{\Sigma}^{1/2}$, which results in
\begin{align}\label{eq:Fsep-AO0}
\min_{\vct{\theta}} \max_{u_i\ge 0}\quad\frac{1}{n} \vct{u}^T\left(\vct{1}\left(1-\vct{\mu}^T \vct{\theta}\right)
- \mtx{D_y} \mtx{Z}\mtx{\Sigma}^{1/2}\vct{\theta}\right) \quad\text{subject to}\quad \qnorm{\bth} \le \frac{1}{\eps}\,.
\end{align}
The above dual problem has a finite optimal value if and only if the data is $(\eps,q)$-separable.  So we aim at finding the largest $\delta$ such that the above problem has still a finite optimal value. (Recall that $\frac{n}{d} \to \delta$.)
\smallskip

\noindent{\bf Reduction to an auxiliary optimization problem via CGMT.}
Note that $y_i = \pm 1$ are independent of $\mtx{Z}$. In addition, the objective function in~\eqref{eq:Fsep-AO0}
is affine in the standard Gaussian matrix $\mtx{Z}$ and the rest of the terms form a convex-concave
function in $\bth$, $\vct{u}$. Due to this particular form we are able to apply a powerful extension of a classical Gaussian process inequality due to Gordon \cite{gordon1988milman} known as Convex Gaussian Minimax Theorem (CGMT) \cite{thrampoulidis2015regularized}, and is discussed in the proof sketch in Section~\ref{sec:proofSketch}. The CGMT framework provides a principled machinery to characterize the asymptotic behavior of certain minimax optimization problems that are affine in a Gaussian matrix $\bX$.

As discussed in the CGMT framework in Section~\ref{sec:proofSketch}, we require minimization/maximization to be over compact sets. The vector $\bth$ already lies in the $\ell_q$ ball of radius $1/\eps$ by constraint. In addition, since $u_i\ge0$, and we are focused on the regime that \eqref{eq:Fsep-AO0} has finite optimal value, the optimal values of $u_i$ should all be finite as well.

We are now ready to applying the CGMT framework. The corresponding Auxiliary Optimization (AO) reads as
\begin{align}
&\min_{\bth} \max_{\vct{u}\ge0}\quad
\frac{1}{n}\Big\{(\vct{u}^T\vct{1})\left(1-\vct{\mu}^T \vct{\theta}\right)+ \twonorm{ \mtx{\Sigma}^{1/2} \bth} \vct{g}^T\vct{u} + \twonorm{\vct{u}} \vct{h}^T \mtx{\Sigma}^{1/2} \bth\Big\}\,,\nn\\
&\text{subject to}\quad \qnorm{\bth} \le \frac{1}{\eps}\,,\label{eq:Fsep-0}
\end{align}
where $\vct{g}\sim\normal(0,\mtx{I}_n)$ and $\vct{h}\sim \normal(0,\mtx{I}_{d})$. Fixing $\beta:=\frac{\twonorm{\vct{u}}}{\sqrt{n}}$ and optimizing over $\vct{u}$ on the non-negative orthant we get
\begin{align}
&\min_{\bth} \max_{\beta\ge0}\quad
\frac{\beta}{\sqrt{n}}  \vct{h}^T  \mtx{\Sigma}^{1/2} \bth+\frac{\beta}{\sqrt{n}}\twonorm{\left(\left(1-\vct{\mu}^T\vct{\theta}\right)\vct{1}+\twonorm{\mtx{\Sigma}^{1/2} \bth}\vct{g}\right)_{+}}\,,\nn\\
&\text{subject to}\quad \qnorm{\bth} \le \frac{1}{\eps}\,.
\end{align}
For data to be separable the above dual optimization should take finite optimal value and therefore the coefficient of $\beta$ should be non-positive. As such the problem is separable if and only if the optimal value of the following problem is non-positive:
\begin{align}\label{eq:sep-condition}
&\min_{\bth} \frac{1}{\sqrt{n}}  \vct{h}^T \mtx{\Sigma}^{1/2} \bth+
\frac{1}{\sqrt{n}}\twonorm{\left(\left(1-\vct{\mu}^T\vct{\theta}\right)\vct{1}+\twonorm{ \mtx{\Sigma}^{1/2} \bth}\vct{g}\right)_{+}}
\le 0\,,\nn\\
&\text{subject to}\quad \qnorm{\bth} \le \frac{1}{\eps}\,.
\end{align}

Consider the decomposition $\bth = \bth_\perp + \theta\tmu$ with $\bth_\perp = \pproj_{\vct{\mu}}\bth$. Note that
\begin{align*}
\frac{1}{\sqrt{n}}  \vct{h}^T \mtx{\Sigma}^{1/2} \bth &= \frac{1}{\sqrt{n}}  \vct{h}^T \mtx{\Sigma}^{1/2} \bth_\perp + \frac{1}{\sqrt{n}}  \theta \vct{h}^T \mtx{\Sigma}^{1/2} \tmu\nn\\
& = \frac{1}{\sqrt{n}}  \vct{h}^T \mtx{\Sigma}^{1/2} \bth_\perp + \frac{1}{\sqrt{n}}  a \theta \vct{h}^T  \tmu
\end{align*}
Since $\vct{h}^T\tmu \sim\normal(0,1)$ and $\theta$ is bounded the contribution of the second term is negligible in the large sample limit $n\to \infty$. This along with the symmetry of the distribution of $\vct{h}$ bring us to
\begin{align}\label{eq:sep1}
 &\min_{\alpha\ge0, \theta, \bth }\;\; -\frac{1}{\sqrt{n}}  \vct{h}^T \mtx{\Sigma}^{1/2} \bth_\perp +\frac{1}{\sqrt{n}}\twonorm{\left(\left(1-\twonorm{\vct{\mu}}\theta\right)\vct{1}+\sqrt{\alpha^2+a^2\theta^2}\; \vct{g}\right)_{+}}\nn\\
 &\text{subject to }\quad \qnorm{\bth} \le \frac{1}{\eps}, \quad \twonorm{\mtx{\Sigma}^{1/2}\bth_\perp} = \alpha\,, \quad \tmu^T\bth = \theta 
\end{align}
\smallskip

\noindent{\bf Scalarization of the auxiliary optimization problem.}
To continue recall the definition of set $\cS$ given by 
\[
\cS(\alpha, \theta, \eps_0,\vct{\mu}) : = \left\{\vct{z}\in \reals^d :\quad  \bz^T\tmu =0,\; \twonorm{\vct{z}} =\alpha,\; \qnorm{\mtx{\Sigma}^{-1/2}\vct{z} +\theta \tmu} \le \frac{1}{\eps_0\pnorm{\vct{\mu}}} \right\}\,.
\]
Recall that $\eps = \eps_0\pnorm{\vct{\mu}}$ and so the optimization problem~\eqref{eq:sep1} above can be rewritten in the form
\begin{align}\label{eq:Sep1-1}
 \min_{\alpha\ge0, \theta}\;\; \min_{\bz\in\cS(\alpha,\theta,\eps_0,\vct{\mu}) }\;\; -\frac{1}{\sqrt{n}}  \vct{h}^T \bz +\frac{1}{\sqrt{n}}\twonorm{\left(\left(1-\twonorm{\vct{\mu}}\theta\right)\vct{1}+\sqrt{\alpha^2+a^2\theta^2}\; \vct{g}\right)_{+}} 
\end{align}
Recall the spherical width of a set $\mathcal{S}\subset\R^d$ defined as 
$$\omega_s\left(\mathcal{S}\right)=\E\Big[\sup_{\vct{z}\in\mathcal{S}} \vct{z}^T\vct{u}\Big]\,,$$ 
where $\vct{u}\in\mathcal{S}^{d-1}$ is a vector chosen at random from the unit sphere. Using this definition and the fact that $\min z = - \max -z$ we have
\begin{align*}
\min_{\bz\in\cS(\alpha,\theta,\eps_0)}\;\; -\frac{1}{\sqrt{n}}  \vct{h}^T\bz=-\frac{1}{\sqrt{n/d}}\sup_{\bz\in\cS(\alpha,\theta,\eps_0) }\;\; \frac{1}{\sqrt{d}}  \vct{h}^T\bz\rightarrow -\frac{1}{\sqrt{\delta}} \omega(\alpha,\theta,\eps_0)\,,
\end{align*}
in probability, where in the last line we use the fact, for $\mathcal{S}\in\mathbb{S}^{d-1}$, the function $f(\vct{u})=\sup_{\vct{z}\in\mathcal{S}} \vct{z}^T\vct{u}$ is Lipschitz. Therefore, using the concentration of Lipschitz functions  of Gaussian random vectors (see e.g. \cite[Theorem 5.2.2]{vershynin2018high}), $f(\vct{u})$ concentrates around its mean $\E f(\vct{u}) = \omega_s(\cS(\alpha,\theta,\eps_0,\vct{\mu}))$. More precisely, 
\begin{align*}
\prob\left\{\Big|\sup_{\vct{z}\in\mathcal{S}} \frac{1}{\sqrt{d}}\vct{h}^T\bz - \omega_s(\cS(\alpha,\theta,\eps_0,\vct{\mu})) \Big| \right\}\le 2e^{-cdt^2}\,,
\end{align*}
for an absolute constant $c>0$ and for every $t\ge0$. Therefore, by invoking the assumption on the convergence of spherical width, cf. Assumption~\ref{ass:omegas}, we arrive at
\begin{align*}
\lim_{d\to\infty}\prob\left\{\Big|\sup_{\vct{z}\in\mathcal{S}} \frac{1}{\sqrt{d}}\vct{h}^T\bz - \omega(\alpha,\theta,\eps_0) \Big| \ge \eta \right\} =0\,,\quad \forall \eta>0\,.
\end{align*}
Therefore, $\sup_{\bz\in\cS(\alpha,\theta,\eps_0) }\;\; \frac{1}{\sqrt{d}}  \vct{h}^T\bz\rightarrow  \omega(\alpha,\theta,\eps_0)$, in probability.


Furthermore, $\twonorm{\vct{\mu}} \to V$ by Assumption~\ref{ass:norm-mu} and since $\vct{g}\sim\normal(0,\mtx{I}_n)$ by applying the Weak Law of Large Numbers we have
\begin{align*}
&\frac{1}{\sqrt{n}}\twonorm{\left(\left(1-\twonorm{\vct{\mu}}\theta\right)\vct{1}+\sqrt{\alpha^2+a^2\theta^2}\; \vct{g}\right)_{+}}
\rightarrow \sqrt{\E\left[\left(1-V \theta+\sqrt{\alpha^2+a^2\theta^2}\; g\right)_{+}^2\right]}
\end{align*}
Thus the objective function in the optimization problem~\eqref{eq:Sep1-1} converges pointwise to
\begin{align}\label{eq:sep-AO00}
 \min_{\alpha\ge0, \theta}\;\; -\frac{1}{\sqrt{\delta}} \omega(\alpha,\theta,\eps_0) +\sqrt{\E\left[\left(1-V\theta+\sqrt{\alpha^2+a^2\theta^2}g\right)_{+}^2\right]} 
\end{align}
Also the problem~\eqref{eq:Sep1-1} is convex as a function of $(\alpha,\theta,\vct{z})$ and since partial maximization preserves convexity, the objective of \eqref{eq:Sep1-1} (after minimization over $\vct{z}$) is a convex function of $(\alpha,\theta)$. We can thus apply the convexity lemma~\cite[Lemma B.2]{thrampoulidis2018precise} to conclude that the minimum value of \eqref{eq:Sep1-1} over $\alpha\ge0, \theta$ also converges to that of~\eqref{eq:sep-AO00}. Therefore, we conclude that data is $(\eps,q)$-separable if and only if the optimal value in~\eqref{eq:sep-AO00} is finite. Rearranging the terms gives us that \eqref{eq:sep-AO00} has a finite optimal value if and only if
\begin{align}
\delta < \delta_*, \quad \text{with} \;\; \delta_* :=  \min_{\alpha\ge0, \theta} \frac{ \omega\left(\alpha, \theta, \eps_0\right)^2}{\E\left[\left(1-V\theta+\sqrt{\alpha^2+\theta^2}g\right)_{+}^2\right]} \, \,.
\end{align} 
 This completes the proof of Theorem~\ref{thm:sep-thresh-aniso}.
 
 
 \subsection{Proof of Theorem~\ref{thm:isotropic-separable-ansio}}
 \smallskip
 We prove Theorem~\ref{thm:isotropic-separable-ansio} using the Convex Gaussian Minimax Theorem (CGMT) as outlined in Section~\ref{sec:proofSketch}. The max-margin problem~\eqref{eq:MM} can be equivalently written as
\begin{align}\label{eq:MM2}
(\tth^\eps,\widehat{\gamma}) = &\arg\min_{\bth, \gamma\ge 0} \quad  \twonorm{\bth}^2  \\
&{\rm subject}\;{\rm to}\;\; y_i\<\vct{x}_i,\vct{\theta}\> - \eps \gamma \ge1,\quad \gamma \ge \qnorm{\bth}\nn
\end{align} 
Now note that writing the Lagrangian for the max-margin problem with $u_i/n$ and $2\lambda$ as dual coefficients, this is equivalent to
\begin{align}
\min_{\vct{\theta}, \gamma \ge 0}\text{ }\max_{u_i,\lambda\ge 0}\quad \twonorm{\vct{\theta}}^2 + \frac{1}{n} \sum_{i=1}^n u_i\left(1+ \eps\gamma - y_i\<x_i,\bth\>\right) + 2\lambda \left(\qnorm{\bth} - \gamma\right)\,.
\end{align}
We next substitute for $\mtx{X} = \vct{y}\vct{\mu}^T +\mtx{Z}\mtx{\Sigma}^{1/2}$ based on the Gaussian mixtures model to arrive at
\begin{align}\label{eq:Lag-sep-F}
\min_{\vct{\theta}, \gamma \ge 0} \max_{u_i\ge 0, \lambda \ge 0}\quad\twonorm{\vct{\theta}}^2 + \frac{1}{n} \left(\vct{u}^T\mathbf{1}+ \eps\gamma \vct{u}^T\mathbf{1}- \vct{u}^T\mtx{D_y} \mtx{Z}\mtx{\Sigma}^{1/2}\vct{\theta} - \vct{u}^T\mathbf{1} \vct{\mu}^T \vct{\theta}\right)+  2\lambda \left(\qnorm{\bth} - \gamma\right)\,.
\end{align}

The advantage of the Lagrangian form in~\eqref{eq:Lag-sep-F} is that it is a minimax problem and the objective is an affine function of the standard Gaussian matrix $\mtx{Z}$. Therefore, we can deploy the Convex Gaussian Minimax Theorem (CGMT) \cite{thrampoulidis2015regularized}, described in Section~\ref{sec:proofSketch}, to characterize asymptotic values of certain functions of this optimization solution, in a high probability sense. 

To recall, the CGMT framework shows that a problem of the form
 \begin{align}
 \label{generalPO1}
 \min_{\bth\in\mathcal{S}_{\bth}}\text{ }\max_{\vct{u}\in\mathcal{S}_{\vct{u}}}\quad \vct{u}^T\mtx{Z}\vct{\theta}+\psi(\vct{\theta},\vct{u})
 \end{align}
 with $\mtx{Z}$ a matrix with $\mathcal{N}(0,1)$ entries can be replaced asymptotically with
 \begin{align}
 \label{generalAO1}
 \min_{\vct{\theta}\in\mathcal{S}_{\vct{\theta}}}\text{ }\max_{\vct{u}\in\mathcal{S}_{\vct{u}}}\quad\twonorm{\vct{\theta}}\vct{g}^T\vct{u}+\twonorm{\vct{u}}\vct{h}^T\vct{\theta}+\psi(\vct{\theta},\vct{u})
 \end{align}
 where $\vct{g}$ and $\vct{h}$ are independent Gaussian vectors with i.i.d.~$\mathcal{N}(0,1)$ entries and $\psi(\vct{\theta},\vct{u})$ is convex in $\vct{\theta}$ and concave in $\vct{u}$. Specifically, the optimal value and corresponding solution  of \eqref{generalPO1} converge in probability to the optimal value and the corresponding solution of \eqref{generalAO1}. In the above $\mathcal{S}_{\vct{\theta}}$ and $\mathcal{S}_{\vct{u}}$ are compact sets. We refer to \cite[Theorem 3]{thrampoulidis2015regularized} for precise statements. As explained in the proof sketch in~\ref{sec:proofSketch}, we follow \cite{thrampoulidis2015regularized} in referring to problems of the form \eqref{generalPO1} and \eqref{generalAO1} as the Primal Problem (PO) and the Auxiliary Problem (AO).

Note that in order to apply CGMT, we need the minimization/maximization to be over compact sets. This technical issue can be avoided by introducing ``artificial'' boundedness constraints on the optimization variables that they do not change the optimal solution. Concretely, we can add constraints of the form $\cS_{\bth} = \{\bth:\;\; \qnorm{\vct{\theta}}\le K_{\bth}\}$ 
 and $\cS_{\vct{u}} = \{\vct{u}:\, 0\le u_i, \; \frac{1}{n}\mathbf{1}^T\bu \le K_{\bu} \}$ for sufficiently large constants $K_{\bth}$, $K_{\vct{u}}$ without changing the optimal solution of \eqref{eq:Lag-sep-F} in a precise asymptotic sense. We refer to Appendix \ref{setres} for precise statements and proofs. This allows us to replace \eqref{eq:Lag-sep-F} with
\begin{align}\label{eq:Lag-sep}
\min_{\bth\in\cS_{\bth}, \gamma \ge0} \max_{\bu\in \cS_{\bu}, \lambda \ge0}\quad\twonorm{\vct{\theta}}^2 + \frac{1}{n} \left(\vct{u}^T\mathbf{1}+ \eps\gamma \vct{u}^T\mathbf{1}- \vct{u}^T\mtx{D_y} \mtx{Z}\mtx{\Sigma}^{1/2}\vct{\theta} - \vct{u}^T\mathbf{1} \vct{\mu}^T \vct{\theta}\right)+  2\lambda \left(\qnorm{\bth} - \gamma\right)\,.
\end{align}

\smallskip

 \noindent{\bf Reduction to an auxiliary optimization problem via CGMT.}
With these compact constraints in place we can now apply the CGMT result to obtain the auxiliary optimization (AO) problem.
 
We proceed by defining the projection matrices 
\begin{align*}
\pproj_{\vct{\mu}} := \mtx{I} - \tmu\tmu^T,\quad 
\proj_{\vct{\mu}} := \tmu\tmu^T\,
\end{align*}
and rewrite $\mtx{Z} \mtx{\Sigma}^{1/2}=\mtx{Z}\left(\proj_{\vct{\mu}}+\pproj_{\vct{\mu}}\right) \mtx{\Sigma}^{1/2}$. Since $\mtx{Z}\proj_{\vct{\mu}}$ and $\mtx{Z}\pproj_{\vct{\mu}}$ are independent from each other the latter has the same distribution as 

\[
\mtx{Z} \mtx{\Sigma}^{1/2}\sim\vct{z}\left(\mtx{\Sigma}^{1/2}\tmu\right)^T+\mtx{Z}\pproj_{\vct{\mu}} \mtx{\Sigma}^{1/2}
\]
where $\vct{z}\sim\mathcal{N}(0,\mtx{I}_n)$ and is independent from the matrix $\mtx{Z}$. This brings us to the following representation
\begin{align*}
\min_{\bth, \gamma\ge 0} \max_{\vct{u}\ge0, \lambda\ge0}  &\twonorm{\vct{\theta}}^2+2\lambda \left(\qnorm{\bth} - \gamma\right) \nn\\
&+\frac{1}{n} \Big\{\vct{u}^T\mathbf{1}+ \eps\gamma \vct{u}^T\mathbf{1} -  \vct{u}^T\mathbf{1} \vct{\mu}^T \vct{\theta}- (\vct{u}^\sT \mtx{D_y}\vct{z})(\tmu^T\mtx{\Sigma}^{1/2}\vct{\theta}) -\vct{u}^\sT \mtx{D_y} \mtx{Z} \pproj_{\vct{\mu}}\mtx{\Sigma}^{1/2} \bth \Big\} 
\end{align*}
Since $y_i = \pm 1$ are independent of $\mtx{Z}$, by applying CGMT framework, the AO reads as
\begin{align}
\min_{\bth, \gamma\ge 0} \max_{\vct{u}\ge0, \lambda\ge0}\quad
\twonorm{\vct{\theta}}^2 +2\lambda \left(\qnorm{\bth} - \gamma\right)+ \frac{1}{n} \Big\{&\vct{u}^T\mathbf{1}+ \eps\gamma \vct{u}^T\mathbf{1} -  \vct{u}^T\mathbf{1} \vct{\mu}^T \vct{\theta}\nn\\
&+ (\vct{u}^T\vct{z}) (\tmu^T\mtx{\Sigma}^{1/2}\bth) + \twonorm{ \pproj_{\vct{\mu}}\mtx{\Sigma}^{1/2} \bth} \vct{g}^T\vct{u}\nn\\ &+ \twonorm{\vct{u}} \vct{h}^T  \pproj_{\vct{\mu}}\mtx{\Sigma}^{1/2} \bth\Big\}\nn
\end{align}
Fixing $\beta:=\frac{\twonorm{\vct{u}}}{\sqrt{n}}$ and optimizing over $\vct{u}$ on the non-negative orthant we get
\begin{align}
\min_{\bth, \gamma\ge 0} \max_{\beta\ge0, \lambda\ge0}\quad
&\twonorm{\vct{\theta}}^2 +2\lambda \left(\qnorm{\bth} - \gamma\right)+\frac{\beta}{\sqrt{n}}  \vct{h}^T  \pproj_{\vct{\mu}}\mtx{\Sigma}^{1/2} \bth\nn\\
&\quad+\frac{\beta}{\sqrt{n}}\twonorm{\left(\left(1+\eps\gamma-\vct{\mu}^T\vct{\theta}\right)\vct{1}+(\tmu^T\mtx{\Sigma}^{1/2}\bth)\vct{z}+\twonorm{ \pproj_{\vct{\mu}}\mtx{\Sigma}^{1/2} \bth}\vct{g}\right)_{+}}
\end{align}
Since $\vct{z}, \vct{g}\sim\normal(0,\mtx{I}_n)$ are independent,  by applying the Weak Law of Large Numbers we have
\begin{align*}
&\frac{1}{\sqrt{n}}\twonorm{\left(\left(1+\eps\gamma-\vct{\mu}^T\vct{\theta}\right)\vct{1}+(\tmu^T\mtx{\Sigma}^{1/2}\bth)\vct{z}+\twonorm{ \pproj_{\vct{\mu}}\mtx{\Sigma}^{1/2} \bth}\vct{g}\right)_{+}}\\
&\quad\quad\quad\quad\rightarrow \left(\E\Big[\left(\left(1+\eps\gamma-\vct{\mu}^T\vct{\theta}\right)+\sqrt{(\tmu^T\mtx{\Sigma}^{1/2}\bth)^2+\twonorm{ \pproj_{\vct{\mu}}\mtx{\Sigma}^{1/2} \bth}^2}g\right)_{+}^2\Big]\right)^{\frac{1}{2}}\\
&\quad\quad\quad\quad\quad= \left(\E\Bigg[\left(\left(1+\eps\gamma-\vct{\mu}^T\vct{\theta}\right)+\twonorm{\mtx{\Sigma}^{\frac{1}{2}}\vct{\theta}}g\right)_{+}^2\Bigg]\right)^{\frac{1}{2}}
\end{align*}
Thus we arrive at
\begin{align}\label{eq:sep-AO0}
\min_{\bth,\gamma\ge 0}\text{ }\max_{\beta\ge0, \lambda\ge0}\quad
&\twonorm{\vct{\theta}}^2 +2\lambda \left(\qnorm{\bth} - \gamma\right)+\frac{\beta}{\sqrt{n}}  \vct{h}^T  \pproj_{\vct{\mu}}\mtx{\Sigma}^{1/2} \bth\nn\\
&+\beta\sqrt{\E\Bigg[\left(\left(1+\eps\gamma-\vct{\mu}^T\vct{\theta}\right)+\twonorm{\mtx{\Sigma}^{\frac{1}{2}}\vct{\theta}}g\right)_{+}^2\Bigg]}\,.
\end{align}

We note that for $a\ge0$,
\[
\E[ag+b]_+^2 = \frac{a^2+b^2}{2}\left(1+\erf\left(\frac{b}{\sqrt{2}a}\right)\right) + \frac{ab}{\sqrt{2\pi}} e^{-\frac{b^2}{2a^2}}\,.
\]
and its derivative with respect to $a$ is given by $2a(1+\erf(\frac{b}{\sqrt{2}a})) > 0$ which implies that the function is increasing in $a>0$. Therefore the optimization \eqref{eq:sep-AO0} can be equivalently written as
\begin{align}\label{eq:sep-AO1}
&\min_{\bth,  \gamma, \alpha\ge 0}\text{ }\max_{\beta\ge0, \lambda\ge0}\quad
\twonorm{\vct{\theta}}^2 +2\lambda \left(\qnorm{\bth} - \gamma\right)+\frac{\beta}{\sqrt{n}}  \vct{h}^T  \pproj_{\vct{\mu}}\mtx{\Sigma}^{1/2} \bth +\beta\sqrt{\E\Bigg[\left(\left(1+\eps\gamma-\vct{\mu}^T\vct{\theta}\right)+\alpha g\right)_{+}^2\Bigg]}\nn\\
&\quad \text{subject to } \quad \twonorm{\mtx{\Sigma}^{\frac{1}{2}}\vct{\theta}}\le \alpha.
\end{align}
Note that the above is trivially jointly convex in $(\vct{\theta}, \gamma, \alpha)$ and jointly concave in $(\beta,\lambda)$. We fix the parallel component of $\bth$ on $\vct{\mu}$ to $\theta$, namely $\theta = \tmu^T\bth$.  We next optimize over $\bth$ while fixing $\theta$.
\begin{align}\label{eq:sep-AO2}
&\min_{\theta,\bth, \gamma\ge 0, \alpha\ge 0}\text{ }\max_{\beta\ge0, \lambda\ge0}\quad
\twonorm{\vct{\theta}}^2 +2\lambda \left(\qnorm{\bth} - \gamma\right)+\frac{\beta}{\sqrt{n}}  \vct{h}^T  \pproj_{\vct{\mu}}\mtx{\Sigma}^{1/2} \bth +\beta\sqrt{\E\Bigg[\left(\left(1+\eps\gamma-\theta\twonorm{\vct{\mu}}\right)+\alpha g\right)_{+}^2\Bigg]}\nn\\
&\quad \text{subject to } \quad \twonorm{\mtx{\Sigma}^{\frac{1}{2}}\vct{\theta}}\le \alpha, \quad \tmu^T\bth = \theta
\end{align}
Bringing the constraints into the objective via Lagrange multipliers we obtain
\begin{align}\label{eq:sep-AO3}
&\min_{\theta,\bth, \gamma\ge 0, \alpha\ge 0}\text{ }\max_{\beta,\lambda,\eta\ge0, \widetilde{\eta}}\quad
\twonorm{\vct{\theta}}^2 +2\lambda \left(\qnorm{\bth}- \gamma\right)+\frac{\beta}{\sqrt{n}}  \vct{h}^T  \pproj_{\vct{\mu}}\mtx{\Sigma}^{1/2} \bth +\beta\sqrt{\E\Bigg[\left(\left(1+\eps\gamma-\theta\twonorm{\vct{\mu}}\right)+\alpha g\right)_{+}^2\Bigg]}\nn\\
&\quad \quad\quad\quad\quad\quad+\eta \left(\twonorm{\mtx{\Sigma}^{\frac{1}{2}}\vct{\theta}}- \alpha\right) +\widetilde{\eta} \left(\tmu^T\bth - \theta\right)
\end{align}
Next note that $\twonorm{\mtx{\Sigma}^{\frac{1}{2}}\vct{\theta}}=\min_{\tau\ge 0} \frac{\twonorm{\mtx{\Sigma}^{\frac{1}{2}}\vct{\theta}}^2}{2\tau}+\frac{\tau}{2}$ and $\qnorm{\bth}=\min_{t\ge 0} \frac{\qnorm{\bth}^q}{q t^{q-1}} + \frac{q-1}{q} t$

Thus, above reduces to
\begin{align}\label{eq:sep-AO33}
&\min_{\theta,\bth, \gamma\ge 0, \alpha\ge 0}\text{ }\max_{\beta,\lambda,\eta\ge0, \widetilde{\eta}}\text{ }\min_{\tau\ge 0, t\ge 0}\quad
\twonorm{\vct{\theta}}^2 +\frac{2\lambda}{qt^{q-1}} \qnorm{\bth}^q+2\lambda\frac{q-1}{q}t- 2\lambda\gamma+\frac{\beta}{\sqrt{n}}  \vct{h}^T  \pproj_{\vct{\mu}}\mtx{\Sigma}^{1/2} \bth \nn\\
&\quad \quad\quad\quad\quad\quad+\beta\sqrt{\E\Bigg[\left(\left(1+\eps\gamma-\theta\twonorm{\vct{\mu}}\right)+\alpha g\right)_{+}^2\Bigg]}\nn\\
&\quad \quad\quad\quad\quad\quad+ \frac{\eta}{2\tau}\twonorm{\mtx{\Sigma}^{\frac{1}{2}}\vct{\theta}}^2+\frac{\eta\tau}{2}- \eta\alpha+\widetilde{\eta} \left(\tmu^T\bth - \theta\right)
\end{align}
To continue note that $\frac{\qnorm{\bth}^q}{t^{q-1}}=t\qnorm{\frac{\bth}{t}}^q$ and $\frac{\twonorm{\mtx{\Sigma}^{\frac{1}{2}}\vct{\theta}}^2}{\tau}=\tau \twonorm{\mtx{\Sigma}^{\frac{1}{2}}\frac{\vct{\theta}}{\tau}}^2$ and thus using the fact that the perspective of a convex function is convex both are jointly convex with respect to $(\vct{\theta},t)$ and $(\vct{\theta},\tau)$. Thus the objective above is jointly convex in $(\vct{\theta},\theta,\gamma,\alpha,t,\tau)$ and jointly concave in $(\beta,\lambda,\eta,\widetilde{\eta})$. Due to this convexity/concavity with respect to the minimization/maximization parameters we can change the order of min and max. We thus proceed by optimizing over $\bth$. The optimization over $\bth$ takes the form
\begin{align}\label{eq:sep-AO4}
\min_{\bth} \quad \twonorm{\vct{\theta}}^2 +\frac{2\lambda}{qt^{q-1}} \qnorm{\bth}^q+ \frac{\eta}{2\tau}\twonorm{\mtx{\Sigma}^{\frac{1}{2}}\vct{\theta}}^2+\frac{\beta}{\sqrt{n}}  \vct{h}^T  \pproj_{\vct{\mu}}\mtx{\Sigma}^{1/2} \bth+\widetilde{\eta}\tmu^T\bth
\end{align}
By completing the square the objective can be alternatively written as
\begin{align}\label{eq:sep-AO4}
&\bth^T\left(\mtx{I}+\frac{\eta}{2\tau}\mtx{\Sigma}\right)\vct{\theta} +\frac{2\lambda}{qt^{q-1}} \qnorm{\bth}^q+\frac{\beta}{\sqrt{n}}  \vct{h}^T  \pproj_{\vct{\mu}}\mtx{\Sigma}^{1/2} \bth+\widetilde{\eta}\tmu^T\bth\nn\\
&=\twonorm{\left(\mtx{I}+\frac{\eta}{2\tau}\mtx{\Sigma}\right)^{1/2}\bth + \frac{\beta}{2\sqrt{n}}\left(\mtx{I}+\frac{\eta}{2\tau}\mtx{\Sigma}\right)^{-1/2} \mtx{\Sigma}^{1/2} \pproj_{\vct{\mu}}\vct{h} + 
\left(\mtx{I}+\frac{\eta}{2\tau}\mtx{\Sigma}\right)^{-1/2} \frac{\tilde{\eta}}{2} \tmu}^2\nn\\
&+\frac{2\lambda}{qt^{q-1}} \qnorm{\bth}^q - \frac{\beta^2}{4n} \vct{h}^T\pproj_{\vct{\mu}}\mtx{\Sigma}^{1/2} \left(\mtx{I}+\frac{\eta}{2\tau}\mtx{\Sigma}\right)^{-1} \mtx{\Sigma}^{1/2} \pproj_{\vct{\mu}}\vct{h} \nn\\
&-\frac{\tilde{\eta}^2}{4}\tmu^T\left(\mtx{I}+\frac{\eta}{2\tau}\mtx{\Sigma}\right)^{-1}\tmu - \frac{\beta\tilde{\eta}}{2\sqrt{n}} \tmu^T \left(\mtx{I}+\frac{\eta}{2\tau}\mtx{\Sigma}\right)^{-1} \mtx{\Sigma}^{1/2} \pproj_{\vct{\mu}}\vct{h}\,.
\end{align}
Since $\mtx{\Sigma}\tmu = a^2 \tmu$ we have
\[
\tmu^T \left(\mtx{I}+\frac{\eta}{2\tau}\mtx{\Sigma}\right)^{-1} \mtx{\Sigma}^{1/2} \pproj_{\vct{\mu}}\vct{h} = 0, \quad \tmu^T\left(\mtx{I}+\frac{\eta}{2\tau}\mtx{\Sigma}\right)^{-1}\tmu = \frac{1}{(1+\tfrac{\eta}{2\tau} a^2)}\,.
\]
We consider a singular value decomposition $\mtx{\Sigma} = \mtx{U}\mtx{S}\mtx{U}^T$ with $\mtx{S} = \diag{s_1,\dotsc, s_d}$, and the first column of $\mtx{U}$ being $\tmu$ and $s_{1} = a^2$ (Recall that $\tmu$ is a singular value of $\mtx{\Sigma}$ with eigenvalue $a^2$.) Then,
\begin{align*}
\frac{1}{n} \vct{h}^T\pproj_{\vct{\mu}}\mtx{\Sigma}^{1/2} \left(\mtx{I}+\frac{\eta}{2\tau}\mtx{\Sigma}\right)^{-1} \mtx{\Sigma}^{1/2} \pproj_{\vct{\mu}}\vct{h}&= \frac{1}{n}\vct{h}^T \pproj_{\vct{\mu}}\mtx{U}\left(\mtx{I}+\frac{\eta}{2\tau}\mtx{S}\right)^{-1}\mtx{S}
\mtx{U}^T\pproj_{\vct{\mu}} \vct{h}\\
& = \frac{1}{\delta d}\sum_{i=2}^d \frac{s_i}{1+\frac{\eta}{2\tau} s_i} h_i^2\\
& \stackrel{P}{\Rightarrow} \frac{1}{\delta d}\sum_{i=1}^d \frac{s_i}{1+\frac{\eta}{2\tau} s_i}\\
&= \frac{2\tau}{\delta d\eta} \sum_{i=1}^d \left(1 - \frac{1}{\frac{\eta}{2\tau}(s_i+\frac{2\tau}{\eta} )} \right)\\
&= \frac{2\tau}{\delta \eta} \left(1 +\frac{2\tau}{\eta}S_{\rho}\left(-\frac{2\tau}{\eta}\right)\right) 
\end{align*}
 with $S_{\rho}(z): = \int \frac{\rho(t)}{z-t}\de t$ the Stieltjes transform of the spectrum of $\mtx{\Sigma}$. 
  
 Using the above identities \eqref{eq:sep-AO4} reduces to
 \begin{align*}
 &\twonorm{\left(\mtx{I}+\frac{\eta}{2\tau}\mtx{\Sigma}\right)^{1/2}\bth + \frac{\beta}{2\sqrt{n}}\left(\mtx{I}+\frac{\eta}{2\tau}\mtx{\Sigma}\right)^{-1/2} \mtx{\Sigma}^{1/2} \pproj_{\vct{\mu}}\vct{h} + 
\left(\mtx{I}+\frac{\eta}{2\tau}\mtx{\Sigma}\right)^{-1/2} \frac{\tilde{\eta}}{2} \tmu}^2\nn\\
&+\frac{2\lambda}{qt^{q-1}} \qnorm{\bth}^q -  \frac{\tau\beta^2}{2\delta \eta} \left(1 +\frac{2\tau}{\eta}S_{\rho}\left(-\frac{2\tau}{\eta}\right)\right)-\frac{\tilde{\eta}^2}{4(1+\tfrac{\eta}{2\tau} a^2)} \,.
 \end{align*}

We then write the minimum value over $\bth$ in terms of the weighted Moreau envelope, given by Definition~\ref{def:WME}. 
\begin{align}\label{eq:sep-AO5}
& \min_{\bth}  
\twonorm{\left(\mtx{I}+\frac{\eta}{2\tau}\mtx{\Sigma}\right)^{1/2}\bth + \frac{\beta}{2\sqrt{n}}\left(\mtx{I}+\frac{\eta}{2\tau}\mtx{\Sigma}\right)^{-1/2} \mtx{\Sigma}^{1/2} \pproj_{\vct{\mu}}\vct{h} + 
\left(\mtx{I}+\frac{\eta}{2\tau}\mtx{\Sigma}\right)^{-1/2} \frac{\tilde{\eta}}{2} \tmu}^2+\frac{2\lambda}{qt^{q-1}} \qnorm{\bth}^q\nn\\
&= 2 e_{q,\mtx{I}+\tfrac{\eta}{2\tau}\mtx{\Sigma}}\left(\left(\mtx{I}+ \frac{\eta}{2\tau}\mtx{\Sigma}\right)^{-1} \left\{\frac{\beta}{2\sqrt{n}} \mtx{\Sigma}^{1/2} \pproj_{\vct{\mu}}\vct{h} - 
 \frac{\tilde{\eta}}{2} \tmu\right\}; \frac{\lambda}{qt^{q-1}}\right)\,,
\end{align}
where we used symmetry of the distribution of $\vct{h}$.

Putting all pieces together in\eqref{eq:sep-AO4} we get
\begin{align}
&\min_{\bth} 
\twonorm{\vct{\theta}}^2 +\frac{2\lambda}{qt^{q-1}} \qnorm{\bth}^q+ \frac{\eta}{2\tau}\twonorm{\mtx{\Sigma}^{\frac{1}{2}}\vct{\theta}}^2+\frac{\beta}{\sqrt{n}}  \vct{h}^T  \pproj_{\vct{\mu}}\mtx{\Sigma}^{1/2} \bth+\widetilde{\eta}\tmu^T\bth\label{eq:sep-AO6}\\
&= 2 e_{q,\mtx{I}+\tfrac{\eta}{2\tau}\mtx{\Sigma}}\left(\left(\mtx{I}+ \frac{\eta}{2\tau}\mtx{\Sigma}\right)^{-1} \left\{\frac{\beta}{2\sqrt{n}} \mtx{\Sigma}^{1/2} \pproj_{\vct{\mu}}\vct{h} - 
 \frac{\tilde{\eta}}{2} \tmu\right\}; \frac{\lambda}{qt^{q-1}}\right)
-\frac{\beta^2\tau}{2\delta \eta} \left(1 +\frac{2\tau}{\eta}S_{\rho}\left(-\frac{2\tau}{\eta}\right)\right) \nn\\
&-\frac{\tilde{\eta}^2}{4(1+\frac{\eta}{2\tau} a^2)} \nn\,.
\end{align}
Using \eqref{eq:sep-AO6} in \eqref{eq:sep-AO33}, the AO problem reduces to
\begin{align}\label{eq:sep-AO7}
\min_{\gamma\ge 0,\theta} \max_{\beta,\lambda,\eta\ge0, \tilde{\eta}}\text{ }\min_{\tau\ge 0, t\ge 0}\quad
&2 e_{q,\mtx{I}+\tfrac{\eta}{2\tau}\mtx{\Sigma}}\left(\left(\mtx{I}+ \frac{\eta}{2\tau}\mtx{\Sigma}\right)^{-1} \left\{\frac{\beta}{2\sqrt{n}} \mtx{\Sigma}^{1/2} \pproj_{\vct{\mu}}\vct{h} - \frac{\tilde{\eta}}{2} \tmu\right\}; \frac{\lambda}{qt^{q-1}}\right)\nn\\
&-\frac{\beta^2\tau}{2\delta \eta} \left(1 +\frac{2\tau}{\eta}S_{\rho}\left(-\frac{2\tau}{\eta}\right)\right)  \nn\\
&-\frac{\tilde{\eta}^2}{4(1+\frac{\eta}{2\tau} a^2)}   +\beta\sqrt{\E\Bigg[\left(\left(1+\eps\gamma-\theta\twonorm{\vct{\mu}}\right)+\alpha g\right)_{+}^2\Bigg]}\nn\\
& +2\lambda\frac{q-1}{q}t- 2\lambda\gamma +\frac{\eta\tau}{2}- \eta\alpha -\tilde{\eta}\theta
\end{align}

\smallskip

\noindent{\bf Scalarization of the auxiliary optimization problem.}
We proceed by defining $\lambda_0 := \frac{\lambda}{\pnorm{\vct{\mu}}}$, $\gamma_0: = \gamma\pnorm{\vct{\mu}}$ and $t_0: = t\pnorm{\vct{\mu}}$. 
Under Assumptions~\ref{ass:norm-mu} and~\ref{ass:converging2}, the asymptotic auxiliary optimization (AO) problem becomes
\begin{align}\label{eq:sep-AO8}
\min_{\alpha,\gamma_0\ge 0,\theta} \max_{\beta,\lambda_0,\eta\ge0, \tilde{\eta}} \text{ }\min_{\tau\ge 0, t_0\ge 0}\quad
&2 \sF\left(\beta,\tilde{\eta};\frac{\eta}{2\tau},\frac{\lambda_0}{qt_0^{q-1}}\right)-\frac{\beta^2\tau}{2\delta \eta} \left(1 +\frac{2\tau}{\eta}S_{\rho}\left(-\frac{2\tau}{\eta}\right)\right)\nn\\
&+2\lambda_0\frac{q-1}{q}t_0- 2\lambda_0\gamma_0 +\frac{\eta\tau}{2}- \eta\alpha -\tilde{\eta}\theta \nn\\
&-\frac{\tilde{\eta}^2}{4(1+\frac{\eta}{2\tau} a^2)}  +\beta\sqrt{\E\Bigg[\left(\left(1+\eps_0\gamma_0-\theta V\right)+\alpha g\right)_{+}^2\Bigg]}
\end{align}
Here we used the relation $\eps = \eps_0\pnorm{\vct{\mu}}$.

We next solve for some of the variables in the AO problem by writing the KKT conditions.
\begin{enumerate}
\item Define
\[
f\left(\frac{\eta}{\tau}\right) : = 2 \sF\left(\beta,\tilde{\eta};\frac{\eta}{2\tau},\frac{\lambda_0}{qt_0^{q-1}}\right)-\frac{\beta^2\tau}{2\delta \eta} \left(1 +\frac{2\tau}{\eta}S_{\rho}\left(-\frac{2\tau}{\eta}\right)\right)-\frac{\tilde{\eta}^2}{4(1+\frac{\eta}{2\tau} a^2)} \,,
\]
where we only made the dependence on $\frac{\eta}{\tau}$ explicit in the notation $f\left(\frac{\eta}{\tau}\right)$. Setting derivative with respect to $\eta$ to zero, we obtain
\begin{align}\label{eq:diff-eta}
\frac{1}{\tau} f'\left(\frac{\eta}{\tau}\right) + \frac{\tau}{2}-\alpha = 0\,.
\end{align}
Setting derivative with respect to $\tau$ to zero, we obtain
\begin{align}\label{eq:diff-tau}
-\frac{\eta}{\tau^2} f'\left(\frac{\eta}{\tau}\right) + \frac{\eta}{2} = 0\,.
\end{align}
Combining \eqref{eq:diff-eta} and \eqref{eq:diff-tau}, we get $\eta(1-\frac{\alpha}{\tau}) = 0$. So either $\alpha = \tau$ or $\eta = 0$. If $\eta=0$, then it is clear that the terms involving $\tau$ in the AO problem would vanish and therefore the value of $\tau$ does not matter. So in this case, we can as well assume $\tau = \alpha$. This simplifies the AO problem by replacing for $\tau$:
 \begin{align}\label{eq:sep-AO8-1}
\min_{\alpha,\gamma_0\ge 0,\theta} \max_{\beta,\lambda_0,\eta\ge0, \tilde{\eta}} \text{ }\min_{t_0\ge 0}\quad
&2 \sF\left(\beta,\tilde{\eta};\frac{\eta}{2\alpha},\frac{\lambda_0}{qt_0^{q-1}}\right)-\frac{\beta^2\alpha}{2\delta \eta} \left(1 +\frac{2\alpha}{\eta}S_{\rho}\left(-\frac{2\alpha}{\eta}\right)\right)\nn\\
&+2\lambda_0\frac{q-1}{q}t_0- 2\lambda_0\gamma_0 -\frac{\eta\alpha}{2} -\tilde{\eta}\theta \nn\\
&-\frac{\tilde{\eta}^2}{4(1+\frac{\eta}{2\alpha} a^2)}  +\beta\sqrt{\E\Bigg[\left(\left(1+\eps_0\gamma_0-\theta V\right)+\alpha g\right)_{+}^2\Bigg]}\,.
\end{align}
\item Setting derivative with respect to $\lambda_0$ to zero, we get
\begin{align}\label{eq:diff-lambda}
2 \sF_4'\left(\beta,\tilde{\eta};\frac{\eta}{2\alpha},\frac{\lambda_0}{qt_0^{q-1}}\right) \frac{1}{qt_0^{q-1}}+2\frac{q-1}{q}t_0- 2\gamma_0 = 0\,,
\end{align}
where $\sF_4'$ denotes the derivative of function $\sF$ with respect to its forth argument.  Also, by setting derivative with respect to $t_0$ to zero we get
\begin{align}\label{eq:diff-t}
2 \sF_4'\left(\beta,\tilde{\eta};\frac{\eta}{2\alpha},\frac{\lambda_0}{qt_0^{q-1}}\right) \lambda_0 \frac{1-q}{q}t_0^{-q}+2\lambda_0 \frac{q-1}{q} = 0\,.
\end{align}
Combining \eqref{eq:diff-lambda} and \eqref{eq:diff-t} implies that
\begin{align}
2\lambda_0(q-1) \left(\frac{\gamma_0}{t_0}-1\right)= 0\,.
\end{align}
Therefore either $\gamma_0 = t_0$ or $\lambda_0=0$ or $q=1$.
If $\lambda = 0$ or $q=1$ then the terms involving $t_0$ in \eqref{eq:sep-AO8-1} vanish and hence we can assume $t_0= \gamma_0$ in this cases as well. Replacing $t_0$ with $\gamma_0$ in \eqref{eq:sep-AO8-1} we obtain
 \begin{align}\label{eq:sep-AO8-2}
&\min_{\alpha,\gamma_0\ge 0,\theta} \max_{\beta,\lambda_0,\eta\ge0, \tilde{\eta}} \text{ }\quad
2 \sF\left(\beta,\tilde{\eta};\frac{\eta}{2\alpha},\frac{\lambda_0}{q\gamma_0^{q-1}}\right)-\frac{\beta^2\alpha}{2\delta \eta} \left(1 +\frac{2\alpha}{\eta}S_{\rho}\left(-\frac{2\alpha}{\eta}\right)\right) -\frac{2\lambda_0}{q}  \gamma_0 -\frac{\eta\alpha}{2} -\tilde{\eta}\theta \nn\\
&\quad\quad\quad\quad\quad\quad-\frac{\tilde{\eta}^2}{4(1+\frac{\eta}{2\alpha} a^2)}  +\beta\sqrt{\E\Bigg[\left(\left(1+\eps_0\gamma_0-\theta V\right)+\alpha g\right)_{+}^2\Bigg]}\,,
\end{align}
which is the expression for $D_{\rm s}(\alpha,\gamma_{0},\theta,\beta,\lambda_{0},\eta,\tilde{\eta})$ given by~\eqref{eq:sep-AO92}. 
 \end{enumerate}
\smallskip

\noindent{\bf Uniqueness and boundedness of the solution to AO problem.}
Note that since $\delta \le \delta_*$, by using Theorem~\ref{thm:sep-thresh-aniso}, we are in the separable regime and therefore optimization~\eqref{eq:MM} is feasible with high probability and admits a bounded solution. This implies that the PO problem~\eqref{eq:Lag-sep} has bounded solution and since AO and PO problems are asymptotically equivalent this implies that the AO problem~\eqref{eq:sep-AO8-2} has bounded solution. 

To show the uniqueness of the solution of~\eqref{eq:sep-AO8-2}, note that as we argued throughout the proof, its objective function $D_{\rm s}$ is jointly strictly convex in $(\alpha, \gamma_0, \theta)$ and jointly concave in $(\beta,\lambda_0,\eta,\tilde{\eta})$. 
Therefor, $\max_{\beta,\lambda_0,\eta,\tilde{\eta}} D_{\rm s}(\alpha,\gamma_{0},\theta,\beta,\lambda_{0},\eta,\tilde{\eta})$ is strictly convex in $(\alpha, \gamma_0, \theta)$. This follows from the fact that if a function $f(\vct{x},\vct{y})$ is strictly convex in $\vct{x}$, then $\max_{\vct{y}} f(\vct{x},\vct{y})$ is also strictly convex in $\vct{x}$ and therefore has a unique minimizer $(\alpha_*,\gamma_{0*},\theta_*)$.

Part (b) of the theorem follows readily from our definition of parameters $\alpha$, $\theta$ and  $\gamma$.

Part (c) also follows from combining Lemma~\ref{lem:SR-AR} with part (b) of the theorem.
\subsection{Proof of Theorem~\ref{thm:isotropic-nonseparable-aniso}}
The goal of this theorem is to derive precise asymptotic behavior for the adversarially trained model $\hth^\eps$ given by
\begin{align}
\hth^\eps = \arg\min_{\bth\in \reals^d} \frac{1}{n} \sum_{i=1}^n  \ell\left(y_i \<\x_i,\bth\> - \eps \qnorm{\bth}\right) \,. 
\end{align}
Letting $v_i:= y_i\<\x_i,\bth\>$, this optimization can be equivalently written as
\[
\min_{\bth,\vct{v}\in \reals^n} \frac{1}{2p} \sum_{i=1}^n  \ell \left(v_i - \eps \qnorm{\bth}\right)    \quad \text{subject to }
\vct{v} = \mtx{D_y} \vct{X} \bth\,,
\]
with $\mtx{D_y} = \diag{y_1,\dotsc, y_n}$. 
Therefore, by writing the Lagrangian by $\vct{u}/n$ as the dual variable for the equality constraint, we arrive at
\[
\min_{\bth, \vct{v}\in \reals^n} \max_{\mtx{\vct{u}\in \reals^n}}
\frac{1}{n} \Big\{\vct{u}^\sT \mtx{D_y X} \bth -  \vct{u}^\sT \vct{v} \Big\}
+ \frac{1}{n} \sum_{i=1}^n  \ell \left(v_i - \eps \qnorm{\bth}\right) 
 \]
We next substitute for $\mtx{X}  = \vct{y}\vct{\mu}^T + \mtx{Z} \mtx{\Sigma}^{1/2}$, under the Gaussian mixtures model, which gives us
 \begin{align}\label{eq:min-max-1}
\min_{\bth, \vct{v}\in \reals^n} \max_{\mtx{\vct{u}\in \reals^n}}
\frac{1}{n} \Big\{\vct{u}^\sT \vct{1}\vct{\mu}^T \bth +\vct{u}^\sT \mtx{D_y} \mtx{Z} \mtx{\Sigma}^{1/2} \bth  -  \vct{u}^\sT \vct{v} \Big\}
+ \frac{1}{n} \sum_{i=1}^n  \ell \left(v_i - \eps \qnorm{\bth}\right) 
 \end{align}
 Note that by the above Lagrangian is in a minimax problem in the form of $\min_{\bth} \max_{\bu} \; \bu^T\mtx{Z}\bth +\psi(\bth,\bu)$, with $\mtx{Z}$ standard Gaussian matrix and $\psi(\bth,\bu)$ is convex in the minimization variable $\bth$ and concave in the maximization variable $\bu$. This form allows us to apply the CGMT framework as outlined in Section~\ref{sec:proofSketch} and similar to the proof of Theorem~\ref{thm:isotropic-separable-ansio}. But in order to do that, we need  the minimization/maximization to be over compact sets. Similar to the proof of Theorem~\ref{thm:isotropic-separable-ansio} we cope with this technical issue by introducing artificial boundedness  constraints on the optimization variables that they do not change the optimal solution. Specifically, we can add constraints of the form $\cS_{\bth} = \{\bth:\; \qnorm{\bth}\le K_{\bth}\}$ and $\cS_{\bu} = \{\bu:\; \|\bu\|_\infty\le K_{\bu}\}$ for sufficiently large constants $K_{\bth}, K_{\bu}$, without changing the optimal solution of \eqref{eq:min-max-1}. We refer to Appendix \ref{setres} for precise statements and proofs. This allows us to replace~\eqref{eq:min-max-1} with 
 \begin{align}\label{eq:min-max-1-1}
\min_{\bth\in\cS_{\bth}, \vct{v}\in \reals^n} \max_{\mtx{\vct{u}\in \cS_{\bu}}}
\frac{1}{n} \Big\{\vct{u}^\sT \vct{1}\vct{\mu}^T \bth +\vct{u}^\sT \mtx{D_y} \mtx{Z} \mtx{\Sigma}^{1/2} \bth  -  \vct{u}^\sT \vct{v} \Big\}
+ \frac{1}{n} \sum_{i=1}^n  \ell \left(v_i - \eps \qnorm{\bth}\right) \,.
 \end{align}
\smallskip

\subsubsection{Reduction to an auxiliary optimization problem via CGMT}  
Next we define the projection matrices 
\begin{align*}
\pproj_{\vct{\mu}} := \mtx{I} - \tmu\tmu^T,\quad 
\proj_{\vct{\mu}} := \tmu\tmu^T\,
\end{align*}
and rewrite $\mtx{Z} \mtx{\Sigma}^{1/2}=\mtx{Z}\left(\proj_{\vct{\mu}}+\pproj_{\vct{\mu}}\right) \mtx{\Sigma}^{1/2}$. Since $\mtx{Z}\proj_{\vct{\mu}}$ and $\mtx{Z}\pproj_{\vct{\mu}}$ are independent from each other the latter has the same distribution as 
\begin{align}\label{eq:Z-decompose}
\mtx{Z} \mtx{\Sigma}^{1/2}\sim\vct{z}\left(\mtx{\Sigma}^{1/2}\tmu\right)^T+\mtx{Z}\pproj_{\vct{\mu}} \mtx{\Sigma}^{1/2}\,.
\end{align}
where $\vct{z}\sim\mathcal{N}(0,\mtx{I}_n)$ and is independent from the matrix $\mtx{Z}$. Thus the above optimization problem is equivalent to
\begin{align}
\min_{\bth\in \cS_{\bth}, \vct{v}} \max_{\vct{u}\in \cS_{\vct{u}}}
\frac{1}{n} \Big\{\vct{u}^\sT \vct{1}\vct{\mu}^T \bth +(\vct{u}^\sT \mtx{D_y}\vct{z})(\tmu^T\mtx{\Sigma}^{1/2}\vct{\theta}) +\vct{u}^\sT \mtx{D_y} \mtx{Z} \pproj_{\vct{\mu}}\mtx{\Sigma}^{1/2} \bth  -  \vct{u}^\sT \vct{v} \Big\}
+ \frac{1}{n} \sum_{i=1}^n  \ell \left(v_i - \eps \qnorm{\bth}\right).
\end{align}

 Using CGMT and the corresponding AO takes the form
 \begin{align}
\min_{\bth \in \cS_{\bth},\vct{v} } \max_{\vct{u}\in \cS_{\vct{u}}} \frac{1}{n} &\Big\{
 \twonorm{\pproj_{\vct{\mu}}\mtx{\Sigma}^{1/2}\bth} \vct{g}^T \mtx{D_y} \vct{u} +  \twonorm{\mtx{D_y}\vct{u}} \vct{h}^T\pproj_{\vct{\mu}} \mtx{\Sigma}^{1/2}\bth  
 +(\vct{u}^\sT \mtx{D_y}\vct{z})(\tmu^T\mtx{\Sigma}^{1/2}\vct{\theta}) \nn\\
 &\;+  \vct{u}^\sT \vct{1}\vct{\mu}^T \bth -  \vct{u}^\sT \vct{v} \Big\}+ \frac{1}{n} \sum_{i=1}^n  \ell \left(v_i - \eps \qnorm{\bth}\right)\,, 
\end{align}
 where $\vct{g}\sim\normal(0,\mtx{I}_n)$ and $\vct{h}\sim \normal(0,\mtx{I}_{d})$. 

Given $y_i = \pm 1$ are independent of $\mtx{Z}$ and hence $\vct{g}$, we have $\mtx{D_y \vct{g}}, \mtx{D_y \vct{z}}\sim\normal(0,\mtx{I}_n)$ and $\twonorm{\mtx{D_y}\vct{u}} = \twonorm{\vct{u}}$. This results in
 \begin{align}
\min_{\bth \in \cS_{\bth},\vct{v} } \max_{\vct{u}\in \cS_{\vct{u}}} \frac{1}{n} &\Big\{
 \twonorm{\pproj_{\vct{\mu}}\mtx{\Sigma}^{1/2}\bth} \vct{g}^T  \vct{u} +  \twonorm{\vct{u}} \vct{h}^T\pproj_{\vct{\mu}} \mtx{\Sigma}^{1/2}\bth  
 +(\vct{u}^\sT \vct{z})(\tmu^T\mtx{\Sigma}^{1/2}\vct{\theta})+  \vct{u}^\sT \vct{1}\vct{\mu}^T \bth -  \vct{u}^\sT \vct{v} \Big\}\nn\\
&+ \frac{1}{n} \sum_{i=1}^n  \ell \left(v_i - \eps \qnorm{\bth}\right)\,. 
\end{align}
 Letting $\beta:= \tfrac{1}{\sqrt{n}}\twonorm{\vct{u}}$ and optimizing over direction of $\vct{u}$, we get
  \begin{align}
& \max_{\vct{u}\in \cS_{\vct{u}}} 
 \frac{1}{n} \left(
\twonorm{\pproj_{\vct{\mu}}\mtx{\Sigma}^{1/2}\bth} \vct{g}^T\vct{u} +  \twonorm{\vct{u}} \vct{h}^T\pproj_{\vct{\mu}}\mtx{\Sigma}^{1/2}\bth +(\vct{u}^\sT \vct{z})(\tmu^T\mtx{\Sigma}^{1/2}\vct{\theta}) +  \vct{u}^\sT \vct{1}\vct{\mu}^T \bth -  \vct{u}^\sT \vct{v} \right) \nonumber\\
 &= \max_{0\le \beta\le K} \frac{\beta}{\sqrt{n}} \twonorm{\twonorm{\pproj_{\vct{\mu}}\mtx{\Sigma}^{1/2}\bth} \vct{g} +\tmu^T\mtx{\Sigma}^{1/2}\vct{\theta}\vct{z}+ \vct{1}\vct{\mu}^T \bth -\vct{v}}
 +\frac{\beta}{\sqrt{n}} \vct{h}^T \pproj_{\vct{\mu}}\mtx{\Sigma}^{1/2} \bth \,,
 \end{align}
 where $K: = \max_{\vct{u}\in \cS_{\vct{u}}} \tfrac{1}{\sqrt{n}}\twonorm{\vct{u}} <K_{\bu}$ by definition of $\cS_{\vct{u}}$.
  
Plugging the latter into AO becomes
 \begin{align}
&\min_{\bth\in\cS_{\bth} ,\vct{v} }\; \max_{0\le\beta\le K}
\frac{\beta}{\sqrt{n}} \twonorm{\twonorm{\pproj_{\vct{\mu}}\mtx{\Sigma}^{1/2}\bth} \vct{g} +\tmu^T\mtx{\Sigma}^{1/2}\vct{\theta}\vct{z}+ \vct{1}\vct{\mu}^T \bth -\vct{v}} +\frac{\beta}{\sqrt{n}} \vct{h}^T \pproj_{\vct{\mu}}\mtx{\Sigma}^{1/2} \bth + \frac{1}{n} \sum_{i=1}^n  \ell \left(v_i - \eps \qnorm{\bth}\right). 
 \end{align}

We hereafter use the shorthand
\[
\ell(\bv,\bth) = \frac{1}{n} \sum_{i=1}^n  \ell \left(v_i - \eps \qnorm{\bth}\right)\,, 
\]
for simplicity of notation. For the minimization, with respect to $\bth$ and then $\vct{v}$, to become easier in our later calculation we proceed by writing $ \ell(\vct{v},\bth)$ in terms of its conjugate with respect to $\bth$. That is,
 \begin{align*}
 \ell(\vct{v},\bth)=\sup_{\vct{w}} \vct{w}^T\bth-\widetilde{\ell}(\vct{v},\vct{w})
 \end{align*}
 where $ \widetilde{\ell}(\vct{v},\vct{w})$ is the conjugate of $\ell$ with respect to $\bth$. The logic behind this is that AO will then simplify to
  \[
 \min_{\bth\in\cS_{\bth},\vct{v}}\max_{0\le\beta\le K, \vct{w}}\;\; \frac{\beta}{\sqrt{n}} \twonorm{\twonorm{\pproj_{\vct{\mu}}\mtx{\Sigma}^{1/2}\bth} \vct{g} +\tmu^T\mtx{\Sigma}^{1/2}\vct{\theta}\vct{z}+ \vct{1}\vct{\mu}^T \bth -\vct{v}} +\frac{\beta}{\sqrt{n}} \vct{h}^T \pproj_{\vct{\mu}}\mtx{\Sigma}^{1/2} \bth   +\vct{w}^T\bth-\widetilde{\ell}(\vct{v},\vct{w})
\]
which after flipping (allowed based on the correct form of convexity/concavity of PO) becomes
  \begin{align}
\max_{0\le\beta\le K, \vct{w}}  \min_{\bth\in\cS_{\bth},\vct{v}}\;\; \frac{\beta}{\sqrt{n}} \twonorm{\twonorm{\pproj_{\vct{\mu}}\mtx{\Sigma}^{1/2}\bth} \vct{g} +\tmu^T\mtx{\Sigma}^{1/2}\vct{\theta}\vct{z}+ \vct{1}\vct{\mu}^T \bth -\vct{v}} +\frac{\beta}{\sqrt{n}} \vct{h}^T \pproj_{\vct{\mu}}\mtx{\Sigma}^{1/2} \bth   +\vct{w}^T\bth-\widetilde{\ell}(\vct{v},\vct{w})\,.
\end{align}
We define the parallel and perpendicular components of $\bth$ along vector $\vct{\mu}$ as follows:
\begin{align}
\tbth =\pproj_{\vct{\mu}} \bth,\quad \theta := \tmu^T \bth,\quad \proj_{\vct{\mu}}\bth = \theta \tmu\,.
\end{align}
Given that $\tmu$ is an eigenvector of $\mtx{\Sigma}$, cf. Assumption~\ref{spikedcov}, we have $\proj_{\vct{\mu}}\mtx{\Sigma}^{1/2}\pproj_{\vct{\mu}} =0$ and therefore
\begin{align*}
\pproj_{\vct{\mu}}\mtx{\Sigma}^{1/2}\bth=\pproj_{\vct{\mu}}\mtx{\Sigma}^{1/2}\left(\proj_{\vct{\mu}}+\pproj_{\vct{\mu}}\right)\bth=
\pproj_{\vct{\mu}}\mtx{\Sigma}^{1/2}\pproj_{\vct{\mu}}\bth
= \pproj_{\vct{\mu}}\mtx{\Sigma}^{1/2}\pproj_{\vct{\mu}}\tbth\,.
\end{align*}
Similarly, since $\mtx{\Sigma}^{1/2}\tmu = a\tmu$.
\begin{align*}
\tmu^T\mtx{\Sigma}^{1/2}\vct{\theta}
= a\tmu^T \bth = a \theta.
\end{align*}

Rewriting the AO problem, we get
  \begin{align}
\max_{0\le\beta\le K, \vct{w}}  \min_{\bth\in\cS_{\bth},\vct{v}}\;\; &\frac{\beta}{\sqrt{n}} \twonorm{\twonorm{\pproj_{\vct{\mu}}\mtx{\Sigma}^{1/2}\pproj_{\vct{\mu}}\tbth} \vct{g} +a\theta\vct{z}+ \vct{1}\twonorm{\vct{\mu}} \theta -\vct{v}}\nn \\
&+\frac{\beta}{\sqrt{n}} \vct{h}^T \pproj_{\vct{\mu}}\mtx{\Sigma}^{1/2}\pproj_{\vct{\mu}}\tbth   +\vct{w}^T\pproj_{\vct{\mu}}\tbth
 +\vct{w}^T\tmu\theta-\widetilde{\ell}(\vct{v},\vct{w}).
\end{align}
We can rewrite this as
  \begin{align}
\max_{0\le\beta\le K, \vct{w}}  \min_{\bth\in\cS_{\bth},\vct{v}}\;\; &\frac{\beta}{\sqrt{n}} \twonorm{\twonorm{\pproj_{\vct{\mu}}\mtx{\Sigma}^{1/2}\pproj_{\vct{\mu}}\tbth} \vct{g} +a\theta\vct{z}+ \vct{1}\twonorm{\vct{\mu}} \theta -\vct{v}}\nn \\
&+\frac{\beta}{\sqrt{n}} \left(\pproj_{\vct{\mu}}\vct{h}\right)^T \pproj_{\vct{\mu}}\mtx{\Sigma}^{1/2}\pproj_{\vct{\mu}}\tbth   +\vct{w}^T\pproj_{\vct{\mu}}\mtx{\Sigma}^{-1/2}\pproj_{\vct{\mu}}\left(\pproj_{\vct{\mu}}\mtx{\Sigma}^{1/2}\pproj_{\vct{\mu}}\right)\tbth
 +\vct{w}^T\tmu\theta-\widetilde{\ell}(\vct{v},\vct{w}).
\end{align}
Here we used the assumption that $\tmu$ is an eigenvector of $\mtx{\Sigma}$ which in turn implies that
\begin{align}\label{eq:aux1}
\pproj_{\vct{\mu}}\mtx{\Sigma}^{-1/2}\pproj_{\vct{\mu}}\left(\pproj_{\vct{\mu}}\mtx{\Sigma}^{1/2}\pproj_{\vct{\mu}}\right) 
=\pproj_{\vct{\mu}}\mtx{\Sigma}^{-1/2} \pproj_{\vct{\mu}} \mtx{\Sigma}^{1/2}\pproj_{\vct{\mu}} 
=\pproj_{\vct{\mu}}\mtx{\Sigma}^{-1/2}\left(\proj_{\vct{\mu}}+\pproj_{\vct{\mu}}\right) \mtx{\Sigma}^{1/2}\pproj_{\vct{\mu}}
 = \pproj_{\vct{\mu}}.
\end{align}
We next optimize over $\bth$ using lemma below and its proof is deferred to Appendix~\ref{proof:lem:tbth}.
\begin{lemma}\label{lem:tbth}
For a given vector $\vct{r}$ and $\alpha\ge0$ consider the following optimization
\begin{align}
&\text{min}_{\bth\in \reals^p} \;\;\; \<\pproj_{\vct{\mu}}\vct{r},\left(\pproj_{\vct{\mu}}\mtx{\Sigma}^{1/2}\pproj_{\vct{\mu}}\right)\tbth\> \\  
&\text{subject to } \twonorm{\left(\pproj_{\vct{\mu}}\mtx{\Sigma}^{1/2}\pproj_{\vct{\mu}}\right) \tbth} = \alpha
\end{align}
Under the assumption that $\proj_{\vct{\mu}}\mtx{\Sigma}^{1/2}\pproj_{\vct{\mu}} =0$, the optimal value of this optimization is given by
$-\alpha\twonorm{\pproj_{\vct{\mu}}\vct{r}}$.
\end{lemma}
Now note that 
\begin{align*}
|\theta| &= |\tmu^T\bth| \le \twonorm{\bth}\,,\\
\alpha &= \twonorm{\left(\pproj_{\vct{\mu}}\mtx{\Sigma}^{1/2}\pproj_{\vct{\mu}}\right) \tbth}  = 
 \twonorm{\pproj_{\vct{\mu}}\mtx{\Sigma}^{1/2} \bth} \le  \twonorm{\mtx{\Sigma}^{1/2} \bth}
 \le C^{1/2}_{\max} \twonorm{\bth}
\end{align*}
where in the second line we used Assumption \ref{ass:asymptotic}$(b)$,$(d)$. Since $\bth\in\cS_{\bth}$ a bounded set, we can choose $K'>0$ large enough so that $0\le |\theta|, \alpha\le K'$ and hence so do the optimization over this bounded range. That said, we use Lemma~\ref{lem:tbth} with $\vct{r} = \frac{\beta}{\sqrt{n}} \vct{h}+\mtx{\Sigma}^{-1/2}\pproj_{\vct{\mu}}\vct{w}$, to simplify the AO problem as follows:  
  \begin{align}\label{eq:AO1}
\max_{0\le\beta\le K, \vct{w}}  \min_{0\le\alpha,|\theta| \le K',\vct{v}}\;\; \frac{\beta}{\sqrt{n}} &\twonorm{\alpha \vct{g} +a\theta\vct{z}+ \vct{1}\twonorm{\vct{\mu}} \theta -\vct{v}}\nn \\
&-\alpha\twonorm{\frac{\beta}{\sqrt{n}} \pproj_{\vct{\mu}}\vct{h}+\pproj_{\vct{\mu}}\mtx{\Sigma}^{-1/2}\pproj_{\vct{\mu}}\vct{w}}
 +\vct{w}^T\tmu \theta-\widetilde{\ell}(\vct{v},\vct{w})
\end{align}

To continue we shall calculate the conjugate function $\widetilde{\ell}$. This is the subject of the next lemma and we refer to Appendix~\ref{proof:lem:conj-L} for its proof.
\begin{lemma}\label{lem:conj-L} The conjugate of the function
\[
\ell(\bv,\bth) = \frac{1}{n} \sum_{i=1}^n  \ell \left(v_i - \eps \qnorm{\bth}\right) 
\]
with respect to $\bth$ is equal to
\begin{align*}
\widetilde{\ell}(\vct{v},\vct{w})=\sup_{\gamma\ge 0}\quad \gamma \pnorm{\vct{w}}-\frac{1}{n} \sum_{i=1}^n  \ell \left(v_i - \eps \gamma \right).
\end{align*}
\end{lemma}
Using the above lemma we have
\begin{align*}
-\widetilde{\ell}(\vct{v},\vct{w})=-\left(\sup_{\gamma\ge 0}\quad \gamma \pnorm{\vct{w}}-\frac{1}{n} \sum_{i=1}^n  \ell \left(v_i - \eps \gamma \right)\right)=\inf_{\gamma\ge 0}\quad -\gamma \pnorm{\vct{w}}+\frac{1}{n} \sum_{i=1}^n  \ell \left(v_i - \eps \gamma \right).
\end{align*}
Plugging this into \eqref{eq:AO1} we arrive at
  \begin{align}
  \label{temp1}
\max_{0\le\beta\le K, \vct{w}}\;\; \min_{0\le \alpha, |\theta| \le K', \vct{v}, 0\le \gamma}\;\; 
 &\frac{\beta}{\sqrt{n}} \twonorm{\alpha \vct{g} +a\theta\vct{z}+ \vct{1}\twonorm{\vct{\mu}} \theta -\vct{v}}\nn \\
&-\alpha\twonorm{\frac{\beta}{\sqrt{n}} \pproj_{\vct{\mu}}\vct{h}+\pproj_{\vct{\mu}}\mtx{\Sigma}^{-1/2}\pproj_{\vct{\mu}}\vct{w}}
 +\vct{w}^T\tmu\theta
-\gamma \pnorm{\vct{w}}+\frac{1}{n} \sum_{i=1}^n  \ell \left(v_i - \eps \gamma \right)
\end{align}
Note that when $p\ge 1$ the objective is jointly concave in $(\vct{w},\beta)$ and jointly convex in $\alpha, \theta, \vct{v}$ and therefore we can switch the orders of min and max.

We next focus on optimization over $\vct{v}$. Using the observation that for all $x\in \reals$, $\min_{\tau \ge 0} \frac{\tau}{2} + \frac{x^2}{2\tau} = x$ we write
\begin{align}
&\min_{\vct{v}} \frac{\beta}{\sqrt{n}} \twonorm{\alpha \vct{g} +a\theta\vct{z}+ \vct{1}\twonorm{\vct{\mu}} \theta -\vct{v}} +\frac{1}{n} \sum_{i=1}^n  \ell \left(v_i - \eps \gamma \right)\nonumber\\
&=\min_{\vct{v}} \inf_{\tau_g\ge0}  \frac{\beta}{2\tau_g n} \twonorm{\alpha \vct{g} +a\theta\vct{z}+ \vct{1}\twonorm{\vct{\mu}} \theta -\vct{v}}^2 + \frac{\beta\tau_g}{2} +\frac{1}{n} \sum_{i=1}^n  \ell \left(v_i - \eps \gamma \right)\nonumber\\
&=\min_{\vct{v}} \inf_{\tau_g\ge0}  \frac{\beta}{2\tau_g n}\sum_{i=1}^n \left(\alpha g_i +a\theta z_i+ \twonorm{\vct{\mu}} \theta -v_i\right)^2 + \frac{\beta\tau_g}{2} +\frac{1}{n} \sum_{i=1}^n  \ell \left(v_i - \eps \gamma \right)\nn\\
&=\min_{\widetilde{v}_i} \inf_{\tau_g\ge0}  \frac{\beta}{2\tau_g n}\sum_{i=1}^n \left(\alpha g_i +a\theta z_i+ \twonorm{\vct{\mu}} \theta -\widetilde{v}_i-\eps\gamma \right)^2 + \frac{\beta\tau_g}{2} +\frac{1}{n} \sum_{i=1}^n  \ell \left(\widetilde{v}_i \right)
\label{eq:aux2}
\end{align} 
As a result \eqref{temp1} can be rewritten as
  \begin{align}
  \label{temp2}
\max_{0\le\beta\le K, \vct{w}}\;\; \min_{0\le \alpha, |\theta| \le K',\widetilde{\vct{v}},0\le\gamma}\;\; \inf_{\tau_g\ge0}\;\; 
 &\frac{\beta}{2\tau_g n} \twonorm{\alpha \vct{g} +a\theta\vct{z}+ \vct{1}\twonorm{\vct{\mu}} \theta -\widetilde{\vct{v}}-\eps\gamma\vct{1}}^2+\frac{\beta \tau_g}{2}\nn \\
&-\alpha\twonorm{\frac{\beta}{\sqrt{n}} \pproj_{\vct{\mu}}\vct{h}+\pproj_{\vct{\mu}}\mtx{\Sigma}^{-1/2}\pproj_{\vct{\mu}}\vct{w}}
 +\vct{w}^T\tmu\theta
-\gamma \pnorm{\vct{w}}+\frac{1}{n} \sum_{i=1}^n  \ell \left(\widetilde{v}_i \right)
\end{align}
We note that since the quadratic over linear function is jointly convex the above loss is jointly convex in the parameters $(\alpha, \gamma, \theta, \tau_g, \widetilde{\vct{v}})$. Also for $p\ge 1$ the $\pnorm{\cdot}$ is convex and thus the objective is also jointly concave in $(\beta, \vct{w})$. 

We recall the definition of the Moreau envelope of function $\ell$ at a point $x$ with parameter $\mu$, that is given by
\begin{align}
e_\ell(x;\mu) := \min_{t} \frac{1}{2\mu} (x-t)^2 + \ell(t)\,. 
\end{align}
We can now rewrite equation~\eqref{temp2} in terms of Moreau envelope of the loss function ${\ell}$.
\begin{align}
\min_{\widetilde{\vct{v}}} \;\; \inf_{\tau_g\ge0}\;\; 
 &\frac{\beta}{2\tau_g n} \twonorm{\alpha \vct{g} +a\theta\vct{z}+ \vct{1}\twonorm{\vct{\mu}} \theta -\widetilde{\vct{v}}-\eps\gamma\vct{1}}^2+\frac{\beta \tau_g}{2} +\frac{1}{n} \sum_{i=1}^n  \ell \left(\widetilde{v}_i \right)\nonumber\\
&\quad\quad=\inf_{\tau_g\ge0} \frac{\beta \tau_g}{2} + \frac{1}{n}\sum_{i=1}^n e_{\ell}\left(\alpha g_i +a\theta z_i+ \twonorm{\vct{\mu}} \theta - \eps\gamma;\frac{\tau_g}{\beta}\right) 
\end{align} 
Thus \eqref{temp2} can be rewritten in the form
\begin{align}
\label{temp3}
\max_{0\le\beta\le K, \vct{w}}\;\; \min_{0\le \alpha, |\theta| \le K',0\le\gamma}\;\; \inf_{\tau_g\ge0}\;\;  &\frac{\beta \tau_g}{2} + \frac{1}{n}\sum_{i=1}^n e_{\ell}\left(\alpha g_i +a\theta z_i+ \twonorm{\vct{\mu}} \theta - \eps\gamma;\frac{\tau_g}{\beta}\right) \nn\\
&-\alpha\twonorm{\frac{\beta}{\sqrt{n}} \pproj_{\vct{\mu}}\vct{h}+\pproj_{\vct{\mu}}\mtx{\Sigma}^{-1/2}\pproj_{\vct{\mu}}\vct{w}}
 +\vct{w}^T\tmu\theta
-\gamma \pnorm{\vct{w}}
\end{align}
Note that since \eqref{temp2} is jointly convex in $(\alpha, \gamma, \theta, \tau_g, \widetilde{\vct{v}})$ and jointly concave in $(\beta, \vct{w})$ and partial minimization preserves convexity thus \eqref{temp3} is jointly convex in $(\alpha, \gamma, \theta, \tau_g)$ and jointly concave in $(\beta, \vct{w})$.

\subsubsection{Scalarization of the auxiliary optimization problem} 
The auxiliary problem~\eqref{temp3} is in terms of high-dimensional vectors $\vct{g},\vct{z}, \vct{h}, \vct{w}, \vct{\mu}$. We turn this problem into a scalar optimization by taking the pointwise limit of its objective and then showing that such convergence indeed holds in a uniform sense and therefore the minimax value also converges to that of the limit objective.

Note that by definition of the Moreau envelope,  for all $x$ and $\mu$ we have
\[
e_{{\ell}}(x;\mu) \le \frac{1}{2\mu}(x-x)^2 + {\ell}(x) = {\ell}(x) = \log(1+e^{-x})
\le \log 2+ |x|\,. 
\] 
Hence,
\[
\E\left[e_{{\ell}}\left(\alpha g +a\theta z+\twonorm{\vct{\mu}} \theta-\eps\gamma;\frac{\tau_g}{\beta}\right)\right]
\le \log 2  + \E[|\alpha g +a\theta z+\twonorm{\vct{\mu}} \theta - \eps \gamma|]<\infty\,,
\]
for any finite value of $\alpha$, $\theta$ and $\gamma$. Therefore by an application of the Weak Law of Large Numbers, we have that 
\begin{align*}
\frac{1}{n}\sum_{i=1}^n e_{\ell}\left(\alpha g_i +a\theta z_i+ \twonorm{\vct{\mu}} \theta-\eps \gamma;\frac{\tau_g}{\beta}\right) &\to
\E\left[e_{{\ell}}\left(\alpha g +a\theta z+ \twonorm{\vct{\mu}} \theta-\eps \gamma;\frac{\tau_g}{\beta} \right) \right]\\
&=\E\left[e_{{\ell}}\left(\sqrt{\alpha^2+a^2\theta^2}g+ \twonorm{\vct{\mu}} \theta-\eps \gamma;\frac{\tau_g}{\beta} \right) \right]\,. 
\end{align*}
We define the expected Moreau envelope $L(a,b,\mu) = \E[e_{\ell}(ag+b;\mu)]$, where the expectation is taken with respect to independent standard normal variable $g$. 

This simplifies the AO problem as
\begin{align}
\label{temp4}
\max_{0\le\beta\le K, \vct{w}}\;\; \min_{0\le \alpha, |\theta| \le K',0\le\gamma,\tau_g}\;\; 
 &\frac{\beta \tau_g}{2} +L\left(\sqrt{\alpha^2+a^2\theta^2},\twonorm{\vct{\mu}} \theta-\eps \gamma,\frac{\tau_g}{\beta}\right) \nn \\
&-\alpha\twonorm{\frac{\beta}{\sqrt{n}} \pproj_{\vct{\mu}}\vct{h}+\pproj_{\vct{\mu}}\mtx{\Sigma}^{-1/2}\pproj_{\vct{\mu}}\vct{w}}
 +\vct{w}^T\tmu\theta
-\gamma \pnorm{\vct{w}}\,.
\end{align}

We note that since \eqref{temp3} is jointly convex in $(\alpha, \gamma, \theta, \tau_g)$ and jointly concave in $(\beta, \vct{w})$ and expectation preserves convexity thus the objective in \eqref{temp4} is also jointly convex in $(\alpha, \gamma, \theta, \tau_g)$ and jointly concave in $(\beta, \vct{w})$ and by Sinov's theorem we can flip the maximization over $\vct{w}$ and the minimization to arrive at
\begin{align}
\max_{0\le\beta\le K}\;\; \min_{0\le \alpha, |\theta| \le K',0<\gamma,\tau_g}\;\; 
 &\frac{\beta \tau_g}{2} +L\left(\sqrt{\alpha^2+a^2\theta^2},\twonorm{\vct{\mu}} \theta-\eps \gamma,\frac{\tau_g}{\beta}\right) \nn \\
&-\min_{\vct{w}} \Big\{ \alpha\twonorm{\frac{\beta}{\sqrt{n}} \pproj_{\vct{\mu}}\vct{h}+\pproj_{\vct{\mu}}\mtx{\Sigma}^{-1/2}\pproj_{\vct{\mu}}\vct{w}}
 - \vct{w}^T\tmu\theta
 +\gamma \pnorm{\vct{w}}\Big\}\,. \label{eq:AO-final0}
\end{align}
By our asymptotic setting (cf. Definition~\ref{ass:asymptotic}, part (c)), $\eps = \eps_0 \pnorm{\vct{\mu}}$ for a constant $\eps_0$. We let $\gamma_0 := \gamma \pnorm{\vct{\mu}}$ and rewriting \eqref{eq:AO-final0} in terms of $\gamma_0$ in lieu of $\gamma$ we arrive at
\begin{align}
\max_{0\le\beta\le K}\;\; \min_{0\le \alpha, |\theta| \le K',0<\gamma_0<K'',0<\tau_g}\;\; 
 &\frac{\beta \tau_g}{2} +L\left(\sqrt{\alpha^2+a^2\theta^2},\twonorm{\vct{\mu}} \theta-\eps_0 \gamma_0,\frac{\tau_g}{\beta}\right) \nn \\
&-\min_{\vct{w}} \Big\{ \alpha\twonorm{\frac{\beta}{\sqrt{n}} \pproj_{\vct{\mu}}\vct{h}+\pproj_{\vct{\mu}}\mtx{\Sigma}^{-1/2}\pproj_{\vct{\mu}}\vct{w}}
 - \vct{w}^T\tmu\theta
 +\frac{\gamma_0}{\pnorm{\vct{\mu}}} \pnorm{\vct{w}}\Big\}\,. \label{eq:AO-final00}
\end{align}

\begin{itemize}[leftmargin=*]
\item {\bf Optimization over $\vct{w}$.} Continuing with optimization over $\vct{w}$ we have
\begin{align}
&\min_{\vct{w}}\;\;\; \alpha\twonorm{\frac{\beta}{\sqrt{n}} \pproj_{\vct{\mu}}\vct{h}+\pproj_{\vct{\mu}}\mtx{\Sigma}^{-1/2}\pproj_{\vct{\mu}}\vct{w}} - \vct{w}^T\tmu\theta+\frac{\gamma_0}{\pnorm{\vct{\mu}}} \pnorm{\vct{w}} \nn\\
&=\min_{\vct{w},\tau_h\ge0}\;\;\; \frac{\alpha}{2\tau_h} \twonorm{\frac{\beta}{\sqrt{n}} \pproj_{\vct{\mu}}\vct{h}+\pproj_{\vct{\mu}}\mtx{\Sigma}^{-1/2}\pproj_{\vct{\mu}}\vct{w}}^2 + \frac{\alpha\tau_h}{2}  - \vct{w}^T\tmu\theta+\frac{\gamma_0}{\pnorm{\vct{\mu}}} \pnorm{\vct{w}}\nn\\
&=\min_{\vct{w},\tau_h\ge0}\;\;\; \frac{\alpha}{2\tau_h} \twonorm{\frac{\beta}{\sqrt{n}} \pproj_{\vct{\mu}}\vct{h}+\pproj_{\vct{\mu}}\mtx{\Sigma}^{-1/2}\vct{w}}^2 + \frac{\alpha\tau_h}{2}  - \vct{w}^T\tmu\theta+\frac{\gamma_0}{\pnorm{\vct{\mu}}} \pnorm{\vct{w}}\,,\label{eq:OPTw1}
\end{align}
where in the last step we used that $\pproj_{\vct{\mu}}\mtx{\Sigma}^{-1/2}\proj_{\vct{\mu}} =0$, which follows from Assumption~\ref{spikedcov}. Note that the above loss is jointly convex in $(\vct{w},\tau_h)$. So that continuing from \eqref{eq:AO-final0} the overall objective is jointly convex in $(\alpha, \gamma, \theta, \tau_g)$ and jointly concave in $(\beta, \vct{w},\tau_h)$.

Let $\tw: = \mtx{\Sigma}^{-1/2}\vct{w}$. The optimization over $\vct{w}$ can be written as
\begin{align}\label{eq:obj-w}
\min_{\tw}\;\;\; \frac{1}{2} \twonorm{\frac{\beta}{\sqrt{n}} \pproj_{\vct{\mu}}\vct{h}+\pproj_{\vct{\mu}}\tw}^2  +f(\tw)\,,
\end{align}
where
\[
f(\tw): = - \<\tw, \mtx{\Sigma}^{1/2}\tmu\> \frac{\theta \tau_h}{\alpha}+\frac{\tau_h}{\alpha}\frac{\gamma_0}{\pnorm{\vct{\mu}}} \pnorm{\mtx{\Sigma}^{1/2}\tw}\,.
\]
Let $\tw^*$ be the optimal solution. Then,
\begin{align}\label{eq:partial-f}
-\pproj_{\vct{\mu}} \left(\frac{\beta}{\sqrt{n}}\vct{h} + \tw^*\right) \in \partial f(\tw^*)\,.
\end{align}
By the conjugate subgradient theorem, this implies that
\[
\tw^* \in \partial f^*\left(-\pproj_{\vct{\mu}} \left(\frac{\beta}{\sqrt{n}}\vct{h} + \tw^*\right)\right)\,.
\]
Let $\vct{t}^*:= \frac{\beta}{\sqrt{n}}\vct{h} + \tw^*$, then writing the above equation in terms of $t$,
\begin{align}\label{eq:conj}
\vct{t}^*-\frac{\beta}{\sqrt{n}}\vct{h}  \in \partial f^*\left(-\pproj_{\vct{\mu}} \vct{t}^*\right)\,.
\end{align}
Therefore,
\begin{align}\label{eq:conj2}
-\pproj_{\vct{\mu}} \left(\vct{t}^*-\frac{\beta}{\sqrt{n}}\vct{h} \right) \in -\pproj_{\vct{\mu}} \partial f^*\left(-\pproj_{\vct{\mu}} \vct{t}^*\right)\,.
\end{align}
Equation~\eqref{eq:conj2} is equivalent to saying that
\begin{align}\label{eq:conj3}
\vct{t}^* \in \arg\min_{\vct{t}} \frac{1}{2} \twonorm{\pproj_{\vct{\mu}}\left(\frac{\beta}{\sqrt{n}}\vct{h}-\vct{t}\right)}^2 + f^*\left(-\pproj_{\vct{\mu}} \vct{t}\right)\,.
\end{align}
\begin{lemma}\label{lem:f-conjugate}
For function $f:\reals^p \mapsto \reals$ given by
\[
f(\tw): = - \<\tw, \mtx{\Sigma}^{1/2}\tmu\> \frac{\theta \tau_h}{\alpha}+\frac{\tau_h}{\alpha} \frac{\gamma_0}{\pnorm{\vct{\mu}}}  \pnorm{\mtx{\Sigma}^{1/2}\tw}\,,
\]
its convex conjugate reads as
\[
f^*(\vct{u}) = \ind_S(\vct{u}),\quad S:=\left\{\vct{u}:\quad\qnorm{\mtx{\Sigma}^{-1/2}\vct{u} +  \frac{\tau_h\theta}{\alpha}\tmu}\le \frac{\gamma\tau_h}{\alpha}\right\}\,,\quad
\ind_S(\vct{u}) = \begin{cases}
0& \text{ if } \vct{u}\in S\\
\infty &\text{ if }\vct{u}\notin S
\end{cases}
\]
\end{lemma}
The proof of Lemma~\ref{lem:f-conjugate} is delegated to Appendix~\ref{proof:lem:f-conjugate}.

Define $\cal{B} : = \{\vct{\mu}\}^\perp \cap -S$. Then \eqref{eq:conj3} implies that
\begin{align}\label{eq:conj4}
\pproj_{\vct{\mu}} \vct{t}^* = \proj_{\cal{B}}\left(\pproj_{\vct{\mu}} \left(\frac{\beta}{\sqrt{n}} \vct{h}\right)\right) \,.
\end{align}
\begin{lemma}\label{lem:proj}
For a convex set $\cS$ and $\cal{B} : = \{\vct{\mu}\}^\perp \cap S$, we have
$\proj_{\cal{B}} \pproj_{\vct{\mu}}  = \proj_{\cal{B}}$.
\end{lemma}
We refer to Appendix~\ref{proof:lem:proj} for the proof of Lemma~\ref{lem:proj}.

Using Lemma~\ref{lem:proj} and~\eqref{eq:conj4} we obtain
\[
\pproj_{\vct{\mu}} \vct{t}^* = \proj_{\cal{B}}\left(\frac{\beta}{\sqrt{n}} \vct{h}\right) \,.
\]
Recalling definition of $\vct{t}^*$ this implies
\begin{align}
\pproj_{\vct{\mu}} \tw^* =  \proj_{\cal{B}}\left(\frac{\beta}{\sqrt{n}} \vct{h}\right) - \frac{\beta}{\sqrt{n}} \pproj_{\vct{\mu}} \vct{h}\,.
\end{align}
Now note that for $p> 1$, $\nabla \pnorm{\vct{w}}= \frac{1}{\pnorm{\vct{w}}^{p-1}} [|w_1|^{p-1} \sign(w_1), \dotsc, |w_p|^{p-1} \sign(w_p)]^\sT$. Therefore in this case $\<\nabla \pnorm{\vct{w}},\vct{w}\> = \pnorm{\vct{w}}$. Similarly, for $p=1$ for any $\vct{s}\in\partial \pnorm{\vct{w}}$ we have $\langle \vct{s},\vct{w}\rangle=\pnorm{\vct{w}}$. Therefore, for all $p\ge 1$ for any $\vct{s}\in\partial \pnorm{\vct{w}}$ we have $\langle \vct{s},\vct{w}\rangle=\pnorm{\vct{w}}$.

 Therefore, for the defined function $f$ and any $\vct{s}\in \partial f(\tw)$ there is a vector $\widetilde{\vct{s}}\in\partial \pnorm{\vct{x}}\Big|_{\vct{x}=\mtx{\Sigma}^{1/2}\tw}$ such that
\begin{align}
\<\vct{s}, \tw\> &= \left\<- \mtx{\Sigma}^{1/2}\tmu \frac{\theta \tau_h}{\alpha}+\frac{\tau_h}{\alpha}\frac{\gamma_0}{\pnorm{\vct{\mu}}}\mtx{\Sigma}^{1/2} \widetilde{\vct{s}},\tw\right\>\nn\\
&=  - \<\tw, \mtx{\Sigma}^{1/2}\tmu\> \frac{\theta \tau_h}{\alpha} + \frac{\tau_h}{\alpha} \frac{\gamma_0}{\pnorm{\vct{\mu}}}\left\<\widetilde{\vct{s}},\mtx{\Sigma}^{1/2}\tw \right\>\nn\\
&=  - \<\tw, \mtx{\Sigma}^{1/2}\tmu\> \frac{\theta \tau_h}{\alpha} + \frac{\tau_h}{\alpha} \frac{\gamma_0}{\pnorm{\vct{\mu}}}\pnorm{\mtx{\Sigma}^{1/2}\tw} \nn\\
&= f(\tw)\,.
\end{align}
Therefore by invoking~\eqref{eq:partial-f}
\begin{align}
f(\tw^*) &= \left\<-\pproj_{\vct{\mu}} \left(\frac{\beta}{\sqrt{n}}\vct{h} + \tw^*\right), \tw^* \right\>\nn\\
&= \left\<-\frac{\beta}{\sqrt{n}} \pproj_{\vct{\mu}} \vct{h} -\pproj_{\vct{\mu}} \tw^*, \pproj_{\vct{\mu}} \tw^* \right\>\nn\\
&=\left\<-\proj_{\cal{B}}\left(\frac{\beta}{\sqrt{n}} \vct{h}\right), \proj_{\cal{B}}\left(\frac{\beta}{\sqrt{n}} \vct{h}\right) - \frac{\beta}{\sqrt{n}} \pproj_{\vct{\mu}} \vct{h} \right\>\nn\\
&=- \twonorm{\proj_{\cal{B}}\left(\frac{\beta}{\sqrt{n}} \vct{h}\right)}^2
+\frac{\beta}{\sqrt{n}} \left\<\proj_{\cal{B}}\left(\frac{\beta}{\sqrt{n}} \vct{h}\right), \pproj_{\vct{\mu}} \vct{h} \right\>\,.
\end{align}

Putting things together, the optimal value of objective \eqref{eq:obj-w} over $\vct{w}$ is given by
\begin{align}
&\min_{\tw}\;\;\; \frac{1}{2} \twonorm{\frac{\beta}{\sqrt{n}} \pproj_{\vct{\mu}}\vct{h}+\pproj_{\vct{\mu}}\tw}^2  +f(\tw)\nn\\
&=\frac{1}{2} \twonorm{\frac{\beta}{\sqrt{n}} \pproj_{\vct{\mu}}\vct{h}+\pproj_{\vct{\mu}}\tw^*}^2  +f(\tw^*)\nn\\
&=\frac{1}{2}\twonorm{\proj_{\cal{B}}\left(\frac{\beta}{\sqrt{n}} \vct{h}\right)}^2- \twonorm{\proj_{\cal{B}}\left(\frac{\beta}{\sqrt{n}} \vct{h}\right)}^2
+\frac{\beta}{\sqrt{n}} \left\<\proj_{\cal{B}}\left(\frac{\beta}{\sqrt{n}} \vct{h}\right), \pproj_{\vct{\mu}} \vct{h} \right\>\nn\\
&=-\frac{1}{2} \twonorm{\proj_{\cal{B}}\left(\frac{\beta}{\sqrt{n}} \vct{h}\right)}^2
+\frac{\beta}{\sqrt{n}} \left\<\proj_{\cal{B}}\left(\frac{\beta}{\sqrt{n}} \vct{h}\right), \pproj_{\vct{\mu}} \vct{h} \right\>\nn\\
&= \frac{\beta^2}{2n} \twonorm{\pproj_{\vct{\mu}} \vct{h}}^2 - \frac{1}{2} \twonorm{\proj_{\cal{B}}\left(\frac{\beta}{\sqrt{n}} \vct{h}\right) - \frac{\beta}{\sqrt{n}} \pproj_{\vct{\mu}} \vct{h}}^2\nn\\
&= \frac{\beta^2}{2n} \twonorm{\pproj_{\vct{\mu}} \vct{h}}^2 - \frac{1}{2} \twonorm{\proj_{\cal{B}}\left(\frac{\beta}{\sqrt{n}} \vct{h}\right) - \frac{\beta}{\sqrt{n}}\vct{h}}^2 + \frac{\beta^2}{2n} \twonorm{\proj_{\vct{\mu}}\vct{h}}^2\nn\\
&= \frac{\beta^2}{2n} \twonorm{\vct{h}}^2 - \frac{1}{2} \twonorm{\proj_{\cal{B}}\left(\frac{\beta}{\sqrt{n}} \vct{h}\right) - \frac{\beta}{\sqrt{n}}\vct{h}}^2\,. \label{eq:opt-w-final}
\end{align}
Following the argument after \eqref{eq:OPTw1} since the objective is jointly convex in $(\alpha, \gamma, \theta, \tau_g)$ and jointly concave in $(\beta, \vct{w},\tau_h)$ and partial maximization preserves concavity after plugging the above the objective is jointly convex in $(\alpha, \gamma, \theta, \tau_g)$ and jointly concave in $(\beta, \tau_h)$.

\item {\bf On projection $\proj_{\cal{B}}$}. As part of our scalarization process of the auxiliary optimization problem, in the next lemma we provide an alternative characterization of the distance $\twonorm{\proj_{\mathcal{B}}(\vct{h}) - \vct{h}}$, and refer to Appendix~\ref{proof:lem:proj-main} for its proof.
\begin{lemma}\label{lem:proj-main}
Recall the set $\mathcal{B}: = \{\vct{\mu}\}^\perp\cap -\cal{S}$, where $\cal{S}$ is given by
\[
S:=\left\{\vct{u}:\quad\qnorm{\mtx{\Sigma}^{-1/2}\vct{u} +  \frac{\tau_h\theta}{\alpha}\tmu}\le \frac{\tau_h}{\alpha}\frac{\gamma_0}{\pnorm{\vct{\mu}}}\right\}\,.
\]
Also, suppose that $\mtx{\Sigma}^{1/2}\tmu = a\tmu$. Then, for any vector $\vct{h}$ the following holds:
\begin{align}
\frac{1}{2}\twonorm{\proj_{\cal{B}}\left(\vct{h}\right) - \vct{h}}^2 =
\sup_{\lambda\ge 0,\nu}\text{ } \;  e_{q,\mtx{\Sigma}}\left(\mtx{\Sigma}^{-1/2}\vct{h}- \left(\frac{\tau_h\theta}{\alpha}+\frac{\nu}{a}\right)\tmu;\lambda\right)
- \lambda \left(\frac{\gamma_0}{\pnorm{\vct{\mu}}}\frac{\tau_h}{\alpha}\right)^q +\nu \tmu^T\vct{h} -  \frac{\nu^2}{2}\end{align}
\end{lemma}
Using equation~\eqref{eq:opt-w-final} along with Lemma~\ref{lem:proj-main} we have
\begin{align}
&\min_{\tw}\;\;\; \frac{1}{2} \twonorm{\frac{\beta}{\sqrt{n}} \pproj_{\vct{\mu}}\vct{h}+\pproj_{\vct{\mu}}\tw}^2  +f(\tw)\nn\\
&= \inf_{\lambda\ge 0,\nu}\text{ } \; \frac{\beta^2}{2n} \twonorm{\vct{h}}^2+ {\lambda}\left(\frac{\tau_h}{\alpha}\frac{\gamma_0}{\pnorm{\vct{\mu}}}\right)^q  - \frac{\nu\beta}{\sqrt{n}} \tmu^T\vct{h} + \frac{\nu^2}{2}
- e_{q,\mtx{\Sigma}}\left(\frac{\beta}{\sqrt{n}}\mtx{\Sigma}^{-1/2}\vct{h}- \left(\frac{\tau_h\theta}{\alpha}+\frac{\nu}{a}\right)\tmu;\lambda\right) 
\end{align}

Recalling equation~\eqref{eq:OPTw1} we have
\begin{align}\label{eq:OPTw5}
&\min_{\vct{w}}\;\;\; \alpha\twonorm{\frac{\beta}{\sqrt{n}} \pproj_{\vct{\mu}}\vct{h}+\pproj_{\vct{\mu}}\mtx{\Sigma}^{-1/2}\pproj_{\vct{\mu}}\vct{w}} - \vct{w}^T\tmu\theta+ \frac{\gamma_0}{\pnorm{\vct{\mu}}} \pnorm{\vct{w}} \nn\\
&= \min_{\tau_h,\lambda\ge0,\nu}
\frac{\alpha}{\tau_h}\left\{\frac{\beta^2}{2n} \twonorm{\vct{h}}^2+ \lambda \left(\frac{\tau_h}{\alpha}\frac{\gamma_0}{\pnorm{\vct{\mu}}}\right)^q  - \frac{\nu\beta}{\sqrt{n}} \tmu^T\vct{h} + \frac{\nu^2}{2}
- e_{q,\mtx{\Sigma}}\left(\frac{\beta}{\sqrt{n}}\mtx{\Sigma}^{-1/2}\vct{h}- \left(\frac{\tau_h\theta}{\alpha}+\frac{\nu}{a}\right)\tmu;\lambda\right) \right\} \nn\\
&\quad\quad\quad\quad +\frac{\alpha\tau_h}{2}\nn\\
&= \min_{\tau_h,\lambda\ge0,\nu}
\frac{\alpha}{\tau_h}\left\{\frac{\beta^2}{2n} \twonorm{\vct{h}}^2+ \lambda_0 \left(\frac{\tau_h\gamma_0}{\alpha}\right)^q  - \frac{\nu\beta}{\sqrt{n}} \tmu^T\vct{h} + \frac{\nu^2}{2}
- e_{q,\mtx{\Sigma}}\left(\frac{\beta}{\sqrt{n}}\mtx{\Sigma}^{-1/2}\vct{h}- \left(\frac{\tau_h\theta}{\alpha}+\frac{\nu}{a}\right)\tmu;\lambda_0 \pnorm{\vct{\mu}}^q\right) \right\}\nn\\
&\quad\quad\quad\quad +\frac{\alpha\tau_h}{2}
\end{align}
where we used the reparameterization  $\lambda_0 := \frac{\lambda}{\pnorm{\vct{\mu}}^q}$. Next we use Assumption~\ref{ass:converging} to take the limit of the above expression as $n\to \infty$. By definition of function $\sE$ we have
\[
\lim_{n\to\infty} e_{q,\mtx{\Sigma}}\left(\frac{\beta}{\sqrt{n}}\mtx{\Sigma}^{-1/2}\vct{h}- \left(\frac{\tau_h\theta}{\alpha}+\frac{\nu}{a}\right)\tmu;\lambda_0 \pnorm{\vct{\mu}}^q\right) = \sE\left(\beta,\left(\frac{\tau_h\theta}{\alpha} + \frac{\nu}{a}\right) ;\lambda_0\right)\,.
\]
 Also, since $\vct{h}\sim\normal(0,\mtx{I}_d)$ we have
 \[
 \lim_{n\to\infty} \frac{1}{n}\twonorm{\vct{h}}^2 = \frac{1}{\delta} \,,\quad  \lim_{n\to\infty}\frac{1}{\sqrt{n}} \tmu^T\vct{h}= 0\,.
 \] 
Using the above two equations in \eqref{eq:OPTw5} we have
\begin{align}
&\lim_{n\to\infty} \min_{\bw}\;\left\{\alpha\twonorm{\frac{\beta}{\sqrt{n}} \pproj_{\vct{\mu}}\vct{h}+\pproj_{\vct{\mu}}\mtx{\Sigma}^{-1/2}\pproj_{\vct{\mu}}\vct{w}} - \vct{w}^T\tmu \theta+\frac{\gamma_0}{\pnorm{\vct{\mu}}} \pnorm{\vct{w}} \right\}\nn\\
&= \frac{\alpha}{\tau_h}\left\{\frac{\beta^2}{2\delta} + \lambda_0 \left(\frac{\gamma_0\tau_h}{\alpha}\right)^q + \frac{\nu^2}{2}
- \sE\left(\beta,\left(\frac{\tau_h\theta}{\alpha} + \frac{\nu}{a}\right) ;\lambda_0\right)\right\} +\frac{\alpha\tau_h}{2}\,.
\end{align}
Finally, incorporating the above equation in \eqref{eq:AO-final00} and using Assumption~\ref{ass:norm-mu}, the AO problem simplifies to:
\begin{align}
\max_{0\le\beta\le K}\;\; \min_{0\le \alpha, |\theta| \le K',0<\gamma_0<K'',0<\tau_g}\;\; 
 &\frac{\beta \tau_g}{2} +L\left(\sqrt{\alpha^2+a^2\theta^2},V \theta-\eps_0 \gamma_0,\frac{\tau_g}{\beta}\right) \nn \\
&- \min_{\tau_h,\lambda_0\ge0,\nu} \left[\frac{\alpha}{\tau_h}\left\{\frac{\beta^2}{2\delta} + \lambda_0 \left(\frac{\gamma_0\tau_h}{\alpha}\right)^q  + \frac{\nu^2}{2}
-\sE\left(\beta,\left(\frac{\tau_h\theta}{\alpha} + \frac{\nu}{a}\right) ;\lambda_0\right)\right\} +\frac{\alpha\tau_h}{2}\right]\label{eq:AO-final-02}
\end{align}

Now recall the argument after \eqref{eq:opt-w-final} that the objective is jointly convex in $(\alpha, \gamma, \theta, \tau_g)$ and jointly concave in $(\beta,\tau_h)$. We used Lemma \ref{lem:proj-main} to provide alternative characterization for quantity $\twonorm{\proj_{\cal{B}}\left(\vct{h}\right) - \vct{h}}^2$, which led into introducing the new variables $\lambda_0, \nu$. Therefore, the objective~\eqref{eq:AO-final-02}, after maximization over $\lambda_0,\nu$, is jointly convex in $(\alpha, \gamma_0, \theta, \tau_g)$ and jointly concave in $(\beta,\tau_h)$. Because of that we can interchange the order of minimization and minimization over using Sion's minimax theorem to get the following.
\begin{align}
&\min_{\theta, 0\le \alpha, \gamma_0, \tau_g}\;\;\max_{0\le\beta, \tau_h} \;\; D_{\rm ns}(\alpha, \gamma_0, \theta, \tau_g, \beta,\tau_h)\nn\\
& D_{\rm ns}(\alpha, \gamma_0, \theta, \tau_g, \beta,\tau_h) = \frac{\beta \tau_g}{2} +L\left(\sqrt{\alpha^2+a^2\theta^2},V \theta-\eps_0 \gamma_0,\frac{\tau_g}{\beta}\right) \nn \\
&\quad\quad\quad\quad\quad\quad\quad\quad\quad- \min_{\lambda_0\ge0,\nu} \left[\frac{\alpha}{\tau_h}\left\{\frac{\beta^2}{2\delta} + \lambda_0 \left(\frac{\gamma_0\tau_h}{\alpha}\right)^q  + \frac{\nu^2}{2}
-\sE\left(\beta,\left(\frac{\tau_h\theta}{\alpha} + \frac{\nu}{a}\right) ;\lambda_0\right)\right\} +\frac{\alpha\tau_h}{2} \right]\,.
\label{eq:AO-final2}
\end{align}
\end{itemize}
\subsubsection{Uniform convergence of the auxiliary problem to its scalarized version}
We showed that the auxiliary optimization objective converges pointwise to the function $D_{\rm ns}$ given by \eqref{eq:AO-final2}. However, we are interested in the minimax optimal solution of the auxiliary problem and need to have convergence of optimal points to the minimax solution of $D$. What is required for this aim is (local) uniform convergence of the auxiliary objective to function $D$. This can be shown by following
similar arguments as in \cite[Lemma A.5]{thrampoulidis2018precise} that is essentially based on a result known as “convexity
lemma” in the literature (see e.g. \cite[Lemma 7.75]{StatDecision}) by which pointwise convergence of convex
functions implies uniform convergence in compact subsets.
\subsubsection{Uniqueness of the solution of the AO problem}
First note that since the loss $\ell(t)$ is a convex function and $\frac{1}{2\mu}(x-t)^2$ is jointly convex in $(x,t,\mu)$, then $\frac{1}{2\mu}(x-t)^2 + \ell(t)$ is jointly convex in $(x,t,\mu)$. Given that partial minimization preserves convexity, the Moreau envelope $e_\ell(x;\mu)$ is jointly convex in $(x,\mu)$. In addition, by using the result of \cite[Lemma 4.4]{thrampoulidis2018precise} the expected Moreau envelope of a convex function is jointly ``strictly'' convex (indeed this holds without requiring any strong
or strict convexity assumption on the function itself). An application of this result to our case implies that $L\left(a,b,\mu\right)$ is jointly strictly convex in $\reals_{\ge0}\times \reals\times \reals_{\ge0}$. 

In addition, as we argued before the function $D_{\rm ns}$ given by \eqref{eq:AO-final2} is  jointly convex in $(\alpha, \gamma_0, \theta, \tau_g)$ and jointly concave in $(\beta,\tau_h)$. Hence, using strict convexity of $L\left(a,b,\mu\right)$, the function $D_{\rm ns}$ is indeed jointly ``strictly'' convex in $(\alpha, \gamma_0, \theta, \tau_g)$ and jointly concave in $(\beta,\tau_h)$. 

As the next step, we note that $\max_{\beta,\tau_h} D(\alpha,\gamma_0,\theta,\tau_g,\beta,\tau_h)$ is strictly convex in $(\alpha, \gamma_0, \theta, \tau_g)$. This follows from the fact that if a function $f(\vct{x},\vct{y})$ is strictly convex in $\vct{x}$, then $\max_{\vct{y}} f(\vct{x},\vct{y})$ is also strictly convex in $\vct{x}$. Moreover, by using the result of \cite[Lemma C.5]{thrampoulidis2018precise} we have that $\inf_{\tau_g>0} \max_{\beta,\tau_h} D(\alpha,\gamma_0,\theta,\tau_g,\beta,\tau_h)$ is strictly convex in $(\alpha,\gamma_0,\theta)$ and therefore has a unique minimizer $(\alpha_*,\gamma_{0*},\theta_*)$.
This concludes the part (a) of the theorem and the given scalar minimax optimization to characterize the limiting behavior of parameter of interest $\alpha,\gamma_0,\theta$.

Part (b) of the theorem follows readily from our definition of parameters $\alpha$, $\theta$ and  $\gamma$. Part (c) of the theorem also follows from combining Lemma~\ref{lem:SR-AR} with part (b) of the theorem.

This completes the proof of Theorem~\ref{thm:isotropic-nonseparable-aniso}.

\subsection{Proof of Remark \ref{rem:p-q-2-Sigma}}\label{proof:rem:p-q-2-Sigma}
We start by establishing an explicit expression for the weighted Moreau envelope $e_{q,\mtx{\Sigma}}$ for 
case of $p=q=2$.
 \begin{lemma}\label{lem:q=2}
We have
\[
e_{2,\mtx{\Sigma}} (\vct{x};\lambda) = \lambda \twonorm{(\mtx{\Sigma} +2\lambda\mtx{I})^{-1/2}\mtx{\Sigma}^{1/2}\vct{x}}^2 
\]
\end{lemma}
The proof of Lemma~\ref{lem:q=2} is given in Appendix~\ref{proof:lem:q=2}.

Suppose that items $(i), (ii)$ in the statement of the remark are satisfied. We then prove that Assumption~\ref{ass:converging2} and \ref{ass:converging} hold.

\begin{proof}[Verification of Assumption~\ref{ass:converging2}] To check Assumption~\ref{ass:converging2} for $p=q=2$, we use Lemma~\ref{lem:q=2} to get
 \begin{align}
&\lim_{n\to\infty}  e_{2,\mtx{I}+b_0\mtx{\Sigma}}\left((\mtx{I}+ b_0\mtx{\Sigma})^{-1} \left\{\frac{c_0}{2\sqrt{n}} \mtx{\Sigma}^{1/2} \pproj_{\vct{\mu}}\vct{h} - 
 \frac{c_1}{2} \tmu\right\};b_1 \twonorm{\vct{\mu}}^2\right)\nn\\
 &= \lim_{n\to\infty} b_1\twonorm{\vct{\mu}}^2 \twonorm{\left((1+2b_1\twonorm{\vct{\mu}}^2)\mtx{I} +b_0\mtx{\Sigma}\right)^{-1/2} (\mtx{I}+b_0\mtx{\Sigma})^{-1/2} \left(\frac{c_0}{2\sqrt{n}} \mtx{\Sigma}^{1/2} \pproj_{\vct{\mu}}\vct{h} - 
 \frac{c_1}{2} \tmu\right)}^2\label{eq:ass-q02}
 \end{align}
 Consider a singular value decomposition $\mtx{\Sigma} = \mtx{U}\mtx{S}\mtx{U}^T$ with $\mtx{S} = \diag{s_1,\dotsc, s_d}$, and the first column of $\mtx{U}$ being $\tmu$ and $s_{1} = a^2$ (Recall that $\tmu$ is a singular value of $\mtx{\Sigma}$ with eigenvalue $a^2$.) Also let $\tilde{\vct{h}}:= \mtx{U}^T\vct{h}\sim\normal(0,\mtx{I}_d)$. Continuing from~\eqref{eq:ass-q02} we write
 \begin{align}
 &\lim_{n\to\infty}  e_{2,\mtx{I}+b_0\mtx{\Sigma}}\left((\mtx{I}+ b_0\mtx{\Sigma})^{-1} \left\{\frac{c_0}{2\sqrt{n}} \mtx{\Sigma}^{1/2} \pproj_{\vct{\mu}}\vct{h} - 
 \frac{c_1}{2} \tmu\right\};b_1 \twonorm{\vct{\mu}}^2\right)\nn\\
 &=\lim_{n\to\infty} \frac{1}{n} b_1\twonorm{\vct{\mu}}^2
 \twonorm{\mtx{U}\left((1+2b_1\twonorm{\vct{\mu}}^2)\mtx{I} +b_0\mtx{S}\right)^{-1/2}(\mtx{I}+b_0\mtx{S})^{-1/2} 
 \mtx{U}^T \left(\frac{c_0}{2} \mtx{\Sigma}^{1/2} \pproj_{\vct{\mu}}\vct{h} - 
 \frac{c_1}{2} \sqrt{n}\tmu\right) }^2\label{eq:ass-equil0}
 \end{align}
 Write $\mtx{U} = [\tmu, \tilde{\mtx{U}}]$ and $\tilde{\mtx{S}} = \diag{s_2,\dotsc, s_d}$. In addition, define $\tilde{\vct{h}}:= \tilde{\mtx{U}}^T\vct{h}\sim\normal(0,\mtx{I}_{d-1})$. We then have
 \[
\mtx{U}^T \mtx{\Sigma}^{1/2} \pproj_{\vct{\mu}}\vct{h} = \begin{pmatrix} 0 \\ \tilde{\mtx{S}^{1/2}} \tilde{\vct{h}}\end{pmatrix},
\quad \mtx{U}^T\tmu = \vct{e}_1\,,\quad \lim_{n\to\infty} \twonorm{\vct{\mu}} = \sigma_{M,2}\,,
 \]
 in probability, with the last limit following from item $(i)$ in the statement Remark~\ref{rem:p-q-2-Sigma}. Using the above identities in~\eqref{eq:ass-equil0} we get
  \begin{align}
 &\lim_{n\to\infty}  e_{2,\mtx{I}+b_0\mtx{\Sigma}}\left((\mtx{I}+ b_0\mtx{\Sigma})^{-1} \left\{\frac{c_0}{2\sqrt{n}} \mtx{\Sigma}^{1/2} \pproj_{\vct{\mu}}\vct{h} - 
 \frac{c_1}{2} \tmu\right\};b_1 \twonorm{\vct{\mu}}^2\right)\nn\\
  &= \lim_{n\to\infty} \frac{b_1\sigma_{M,2}^2}{n} \left\{\frac{c_1^2 n}{4} \cdot \frac{1}{(1+2b_1\sigma_{M,2}^2+b_0 a^2)(1+b_0a^2)} +
  \frac{c_0^2}{4}\sum_{i=2}^d \frac{ s_i \tilde{h}_i^2}{(1+b_0s_i)(1+2b_1\sigma_{M,2}^2+b_0 s_i)} \right\}\nn\\
  &=  {b_1\sigma_{M,2}^2} \left\{  \frac{c_1^2}{4(1+b_0a^2)(1+2b_1\sigma_{M,2}^2+b_0 a^2)} + \frac{c_0^2}{4\delta} \lim_{d\to\infty}
  \frac{1}{d}\sum_{i=2}^d \frac{s_i \tilde{h}_i^2}{(1+b_0s_i) (1+2b_1\sigma_{M,2}^2+b_0 s_i)} \right\}\label{eq:ass-equil1}
 \end{align}
 Define $\nu_i: = s_i(1+b_0s_i)^{-1} (1+2b_1\sigma_{M,2}^2+b_0 s_i)^{-1}$. Then the last sum reads as $\frac{1}{d}\sum_{i=2}^d \nu_i \tilde{h}_i^2$. Recall that $\tilde{\vct{h}}\sim\normal(0,\mtx{I}_{d-1})$. Therefore, by applying the Kolmogorov’s criterion of SLLN the above limit exists (almost surely and so in probability as well) provided that $\frac{1}{d^2}\sum_{i=2}^d \nu_i^2\Var(\tilde{h}_i^2) <\infty$. We note that $Var(\tilde{h}_i^2) = 2$ and since $\nu_i\ge0$, we have
 \[
 \frac{1}{d^2}\sum_{i=2}^d \nu_i^2 \le \left(\frac{1}{d} \sum_{i=2}^d \nu_i\right)^2\,.
 \]
 Hence it suffices to show that $\frac{1}{d} \sum_{i=2}^d \nu_i<\infty$.
 Now by item $(ii)$ of Remark~\ref{rem:p-q-2-Sigma}, the empirical distribution of eigenvalues of $\mtx{\Sigma}$ converges weakly to a distribution $\rho$ with Stieltjes transform $S_{\rho}(z): = \int \frac{\rho(t)}{z-t}\de t$. We write
 \begin{align}
 & \frac{ s_i }{(1+b_0s_i)(1+2b_1\sigma_{M,2}^2+b_0 s_i)}\nn\\ 
&=\frac{1}{2b_0b_1\sigma_{M,2}^2} \left\{-\frac{1}{1+b_0 s_i} + \frac{1+2b_1\sigma_{M,2}^2}{1+2b_1\sigma_{M,2}^2+b_0 s_i}\right\}\nn\\
 &=\frac{1}{2b_0^2b_1\sigma_{M,2}^2} \left\{-\frac{1}{\frac{1}{b_0}+ s_i} + \frac{1+2b_1\sigma_{M,2}^2}{\frac{1+2b_1\sigma_{M,2}^2}{b_0}+ s_i}\right\}
 \end{align}
 Therefore, 
  \begin{align}
  \lim_{d\to\infty} \frac{1}{d}\sum_{i=2}^d \nu_i&= \lim_{d\to\infty} \frac{1}{d}\sum_{i=2}^d \frac{ s_i }{(1+b_0s_i)(1+2b_1\sigma_{M,2}^2+b_0 s_i)}\nn\\
  &=\frac{1}{2b_0^2b_1\sigma_{M,2}^2} \lim_{d\to\infty} \frac{1}{d}\sum_{i=2}^d \left\{-\frac{1}{\frac{1}{b_0}+ s_i} + \frac{1+2b_1\sigma_{M,2}^2}{\frac{1+2b_1\sigma_{M,2}^2}{b_0}+ s_i}\right\}\nn\\
  &= \frac{1}{2b_0^2b_1\sigma_{M,2}^2} \left\{S_{\rho}\left(-\frac{1}{b_0}\right) - (1+2b_1\sigma_{M,2}^2)
  S_\rho\left(-\frac{1+2b_1\sigma_{M,2}^2}{b_0} \right)\right\}\,.\label{eq:ass-equil2}
   \end{align}
   It is worth noting that although the sum is over $2\le i\le d$, the term for $i=1$ is $O(1/d)$ and is negligible in the limit. Therefore, we can include that in our calculation above. By using Equation~\eqref{eq:ass-equil2} in~\eqref{eq:ass-equil1} we get that Assumption~\ref{ass:converging2} holds with
   \begin{align}
   \sF(c_0,c_1;b_0,b_1) =&\frac{b_1\sigma_{M,2}^2 c_1^2}{4(1+b_0a^2)(1+2b_1\sigma_{M,2}^2+b_0 a^2)} \nn\\
   &+
   \frac{b_1\sigma_{M,2}^2 c_0^2}{8\delta b_0^2b_1\sigma_{M,2}^2 }  \left\{S_{\rho}\left(-\frac{1}{b_0}\right) - (1+2b_1\sigma_{M,2}^2)
  S_\rho\left(-\frac{1+2b_1\sigma_{M,2}^2}{b_0} \right)\right\}\,.
   \end{align}
\end{proof} 

\begin{proof}[Verification of Assumption~\ref{ass:converging}]
To check Assumption~\ref{ass:converging} we use Lemma~\ref{lem:q=2} and write
\begin{align}
&\lim_{n\to\infty}  e_{2,\mtx{\Sigma}}\left(\frac{c_0}{\sqrt{n}}\mtx{\Sigma}^{-1/2}\vct{h}- c_1 \tmu; \lambda_0 \twonorm{\vct{\mu}}^2\right)\nn\\
&=\lim_{n\to\infty} \lambda_0 \twonorm{\vct{\mu}}^2 \twonorm{(\mtx{\Sigma} +2\lambda_0 \twonorm{\vct{\mu}}^2 \mtx{I})^{-1/2}\mtx{\Sigma}^{1/2}\left(\frac{c_0}{\sqrt{n}}\mtx{\Sigma}^{-1/2}\vct{h}- c_1 \tmu\right)}^2\nn\\
&=\lim_{n\to\infty} \lambda_0 \twonorm{\vct{\mu}}^2 \twonorm{(\mtx{\Sigma} +2\lambda_0 \twonorm{\vct{\mu}}^2 \mtx{I})^{-1/2}\left(\frac{c_0}{\sqrt{n}}\vct{h}- c_1a \tmu\right)}^2\label{eq:ass-q2}
\end{align}
Consider a singular value decomposition $\mtx{\Sigma} = \mtx{U}\mtx{S}\mtx{U}^T$ with $\mtx{S} = \diag{s_1,\dotsc, s_d}$, and the first column of $\mtx{U}$ being $\tmu$ and $s_{1} = a^2$ (Recall that $\tmu$ is a singular value of $\mtx{\Sigma}$ with eigenvalue $a^2$.) Also let $\tilde{\vct{h}}:= \mtx{U}^T\vct{h}\sim\normal(0,\mtx{I}_d)$. Continuing from~\eqref{eq:ass-q2} we write
\begin{align}
&\lim_{n\to\infty}  e_{2,\mtx{\Sigma}}\left(\frac{c_0}{\sqrt{n}}\mtx{\Sigma}^{-1/2}\vct{h}- c_1 \tmu; \lambda_0 \twonorm{\vct{\mu}}^2\right)\nn\\
&=\lim_{n\to\infty} \frac{1}{n} \lambda_0 \twonorm{\vct{\mu}}^2 \twonorm{\mtx{U}(\mtx{S} +2\lambda_0 \twonorm{\vct{\mu}}^2\mtx{I})^{-1/2} \left(c_0\tilde{\vct{h}}- c_1\sqrt{n}a\vct{e}_1\right)}^2\nn\\
&=\lim_{n\to\infty} \frac{\lambda_0\sigma^2_{M,2}}{n} \left\{\frac{(c_0\tilde{h}_1-c_1\sqrt{n}a)^2}{a^2+2\lambda_0 \sigma^2_{M,2}} 
+\sum_{i=2}^d \frac{c_0^2\tilde{h}_i^2}{s_i+2\lambda_0\sigma^2_{M,2}} \right\}\nn\\
&=\lambda_0\sigma^2_{M,2} \left\{\frac{ c_1^2a^2}{a^2+2\lambda_0\sigma^2_{M,2}} 
+ \frac{1}{\delta}\lim_{d\to\infty} \frac{1}{d}\sum_{i=2}^d \frac{c_0^2\tilde{h}_i^2}{s_i+2\lambda_0\sigma^2_{M,2}}\right\} \,.\label{eq:FFFF}
\end{align}
By applying the Kolmogorov’s criterion of SLLN the above limit exists (almost surely and so in probability as well) provided that $ \frac{1}{d^2}\sum_{i=2}^d \frac{1}{(s_i+2\lambda_0\sigma^2_{M,2})^2}<\infty$. Note that since $\lambda_0, s_i\ge0$, we have
\[
\frac{1}{d^2}\sum_{i=2}^d \frac{1}{(s_i+2\lambda_0\sigma^2_{M,2})^2}
\le \frac{1}{d^2}\sum_{i=2}^d \frac{1}{4\lambda_0^2\sigma^4_{M,2}} \to 0\,. 
\]
By using the LLN we obtain that the summation in~\eqref{eq:FFFF} converges (almost surely) to its expectation.
Now recalling item $(ii)$ in Remark~\ref{rem:p-q-2-Sigma}, we know that the empirical distribution of eigenvalues of $\mtx{\Sigma}$ converges weakly to a distribution $\rho$ with Stieltjes transform $S_{\rho}(z): = \int \frac{\rho(t)}{z-t}\de t$, and therefore we have
\begin{align}
\sE(c_0,c_1;\lambda_0)&:= \lim_{n\to\infty} e_{2,\mtx{\Sigma}}\left(\frac{c_0}{\sqrt{n}}\mtx{\Sigma}^{-1/2}\vct{h}- c_1 \tmu; \lambda_0 \twonorm{\vct{\mu}}^2\right)\nn\\
& = \lambda_0\sigma^2_{M,2}\left\{ \frac{c_1^2a^2}{a^2+2\lambda_0\sigma^2_{M,2}} 
- \frac{c_0^2}{\delta} S_{\rho}(-2\lambda_0 \sigma^2_{M,2})\right\}\,.
\end{align}
\end{proof} 

 \section{Proofs for isotropic Gaussian model (Section~4)}
This section is devoted to the proof of our theorems for the isotropic Gaussian model. We discuss how these theorems can be derived as special cases of our results for the anisotropic model, after some algebraic simplifications.

The claim of Theorem~\ref{thm:sep-thresh} on the separability threshold is an immediate corollary of Theorem~\ref{thm:sep-thresh-aniso}, with $\mtx{\Sigma}=\mtx{I}_{p\times p}$ and $a = 1$. We next move to the two other theorems on precise characterization of standard and robust accuracy in the separable and non-separable regimes. 
  
\subsection{Proof of Theorem~\ref{thm:isotropic-separable}}  
 Suppose that Assumption~\ref{ass:alternative} in the statement of Theorem~\ref{thm:isotropic-separable} holds. We first show that this assumption implies Assumption~\ref{ass:converging2}, required by Theorem~\ref{thm:isotropic-separable-ansio}, in case of $\mtx{\Sigma} =\mtx{I}$ and then show how Theorem~\ref{thm:isotropic-separable} can be derived as a special case of Theorem~\ref{thm:isotropic-separable-ansio}.
 
 To prove Assumption~\ref{ass:converging2}(b) for isotropic case, we use the following two properties of the weighted Moreau envelop that holds for all $q\ge0$:
 \begin{eqnarray}
 e_{q,\alpha\mtx{I}}(\vct{x};\lambda) &=& \alpha e_{q,\mtx{I}}\left(\vct{x},\frac{\lambda}{\alpha}\right)\,,\\
 \frac{1}{b^2}e_{q,\mtx{I}}\left(b\vct{x};\frac{\lambda}{b^{q-2}} \right)&=& e_{q,\mtx{I}}(\vct{x};\lambda)\,.\label{eq:evn-prop2}
 \end{eqnarray}
 Combining the above two identities we get
 \[
 e_{q,\alpha\mtx{I}}(\alpha^{-1}\vct{x};\lambda) = \alpha e_{q,\mtx{I}}\left(\frac{\vct{x}}{\alpha};\frac{\lambda}{\alpha}\right)
 =  \frac{\alpha}{b^2} e_{q,\mtx{I}}\left(b\frac{\vct{x}}{\alpha};\frac{\lambda}{\alpha b^{q-2}} \right)\,.
 \]
Using the above identity with $\alpha = 1+b_0$ and $b = \alpha\sqrt{d}$ we have
 \begin{align}
 &e_{q,(1+b_0)\mtx{I}}\left((1+ b_0)^{-1} \left\{\frac{c_0}{2\sqrt{n}}  \pproj_{\vct{\mu}}\vct{h} - 
 \frac{c_1}{2} \tmu\right\};b_1 \pnorm{\vct{\mu}}^q\right)\nn\\
 &=\frac{1}{(1+ b_0) d} e_{q,\mtx{I}}\left( \frac{\sqrt{d}c_0}{2\sqrt{n}}  \pproj_{\vct{\mu}}\vct{h} - 
 \frac{c_1\sqrt{d}}{2} \tmu;\frac{b_1 \pnorm{\vct{\mu}}^q}{(1+b_0)^{q-1} d^{\frac{q}{2}-1}}\right)
 \end{align}
We next proceed to take the limit of the above expression as $n\to\infty$. By Assumption~\ref{ass:alternative} we have 
\begin{align}\label{eq:2-p-moment-2}
\pnorm{\vct{\mu}}^q \to \sigma_{M,p}^q  d^{\frac{q}{2}-1}\,,\quad
\twonorm{\vct{\mu}} \to \sigma_{M,2} \,,
\end{align}
with high probability. Also by Assumption~\ref{ass:asymptotic}, we have $n/d\to\delta$. Therefore,
 \begin{align}
 &\lim_{n\to\infty} e_{q,(1+b_0)\mtx{I}}\left((1+ b_0)^{-1} \left\{\frac{c_0}{2\sqrt{n}}  \pproj_{\vct{\mu}}\vct{h} - 
 \frac{c_1}{2} \tmu\right\};b_1 \pnorm{\vct{\mu}}^q\right)\nn\\
 &=\lim_{n\to\infty} \frac{1}{(1+ b_0) d} e_{q,\mtx{I}}\left( \frac{\sqrt{d}c_0}{2\sqrt{n}}  \pproj_{\vct{\mu}}\vct{h} - 
 \frac{c_1\sqrt{d}}{2} \tmu;\frac{b_1 \pnorm{\vct{\mu}}^q}{(1+b_0)^{q-1} d^{\frac{q}{2}-1}}\right)\nn\\
  &=\lim_{n\to\infty} \frac{1}{(1+ b_0) d} e_{q,\mtx{I}}\left( \frac{\sqrt{d}c_0}{2\sqrt{n}}  \pproj_{\vct{\mu}}\vct{h} - 
 \frac{c_1}{2} \frac{\sqrt{d}\vct{\mu}}{\twonorm{\vct{\mu}}};\frac{b_1 \pnorm{\vct{\mu}}^q}{(1+b_0)^{q-1} d^{\frac{q}{2}-1}}\right)\nn\\
 &=\lim_{n\to\infty} \frac{1}{(1+ b_0) d} e_{q,\mtx{I}}\left( \frac{c_0}{2\sqrt{\delta}}  \pproj_{\vct{\mu}}\vct{h} - 
 \frac{c_1}{2} \frac{\sqrt{d}\vct{\mu}}{\sigma_{M,2}};b_1 (1+b_0)^{1-q} \sigma_{M,p}^q\right)\label{eq:ass-ass1-1}
 \end{align}
We next note that by definition of the weighted Moreau envelop we have $e_{q,\mtx{I}}(\vct{x};\lambda) = \sum_{i=1}^d J_q(x_i;\lambda)$. Also,  by Assumption~\ref{ass:converging} the empirical distribution of entries of $\sqrt{d}\vct{\mu}$ converges weakly to distribution $\prob_M$.  Therefore, continuing from~\eqref{eq:ass-ass1-1} we can write
\begin{align}
&\lim_{n\to\infty} e_{q,(1+b_0)\mtx{I}}\left((1+ b_0)^{-1} \left\{\frac{c_0}{2\sqrt{n}}  \pproj_{\vct{\mu}}\vct{h} - 
 \frac{c_1}{2} \tmu\right\};b_1 \pnorm{\vct{\mu}}^q\right)\nn\\
 &\stackrel{(a)}{=} \lim_{n\to\infty} e_{q,(1+b_0)\mtx{I}}\left((1+ b_0)^{-1} \left\{\frac{c_0}{2\sqrt{n}} \vct{h} - 
 \frac{c_1}{2} \tmu\right\};b_1 \pnorm{\vct{\mu}}^q\right)\nn\\
 &=\lim_{n\to\infty} \frac{1}{(1+ b_0) d} \sum_{i=1}^d J_q\left( \frac{c_0 h_i}{2\sqrt{\delta}} - 
 \frac{c_1}{2} \frac{\sqrt{d}\mu_i}{\sigma_{M,2}};b_1 (1+b_0)^{1-q} \sigma_{M,p}^q\right)\nn\\
 &\stackrel{(b)}{=} \frac{1}{1+b_0}\E\left[J_q\left( \frac{c_0 h}{2\sqrt{\delta}} - 
 \frac{c_1M}{2\sigma_{M,2}};b_1 (1+b_0)^{1-q} \sigma_{M,p}^q\right) \right] = (1+b_0)^{-1}\EJ\left(\frac{c_0}{2},\frac{c_1}{2}; b_1 (1+b_0)^{1-q}\right)\,,\label{eq:ass-ass1-2}
\end{align}
where the expectation in $(b)$ is taken with respect to the independent random variables $h\sim\normal(0,1)$ and $M\sim\prob_M$. The last equality follows by definition of function $\EJ$ given by~\eqref{eq:EJ} and by deploying Assumption~\ref{ass:alternative}.
 
Here, $(a)$ follows by writing
\[
\frac{c_0}{2\sqrt{\delta}}\; \pproj_{\vct{\mu}} \vct{h}-  \frac{c_1\sqrt{d}\vct{\mu}}{2\sigma_{M,2}} = \frac{c_0}{2\sqrt{\delta}}\; \vct{h}-   
\left(\frac{c_1\sqrt{d}}{2\sigma_{M,2}} + \frac{c_0}{2\sqrt{\delta}\twonorm{\vct{\mu}}} \vct{h}^T \tmu\right)
\vct{\mu}
\]
and noting that $\tmu^T\vct{h} \sim \normal(0,1)$ since $\twonorm{\tmu}=1$, and $\twonorm{\vct{\mu}} \to \sigma_{M,2}$ which implies that the last term in the right-hand side is dominated by the second term therein that is of order $\sqrt{d}$.

The chain of equalities in~\eqref{eq:ass-ass1-2} shows that Assumption~\ref{ass:converging2}(a) is satisfied by
$F(c_0,c_1;b_0,b_1) = \EJ\left(\frac{c_0}{2},\frac{c_1}{2}; b_1 (1+b_0)^{1-q}\right)$ for the isotropic model.

Assumption~\ref{ass:converging2}(b) also clearly holds for isotropic model ($\mtx{\Sigma} = \mtx{I}$) with 
$S_{\rho}(z) = \frac{1}{z-1}$. 

Now that Assumption~\ref{ass:converging2} holds we can use  the result of Theorem~\ref{thm:isotropic-separable-ansio} for the special case of $\mtx{\Sigma} = \mtx{I}$. As we showed above for this case, we have the following identities
\begin{align}
F(c_0,c_1;b_0,b_1) = \EJ\left(\frac{c_0}{2},\frac{c_1}{2}; b_1 (1+b_0)^{1-q}\right)\,, \quad
S_{\rho}(z)= \frac{1}{z-1}\,.
\end{align}

Now by using these identities in the AO problem~\eqref{eq:sep-AO92} and after some simple algebraic manipulation we obtain the AO problem~\eqref{eq:sep-AO9}.
\subsection{Proof of Theorem~\ref{thm:isotropic-nonseparable}}
We prove Theorem~\ref{thm:isotropic-nonseparable} as a special case of Theorem~\ref{thm:isotropic-nonseparable-aniso}.
We first show that in the isotropic case, Assumption~\ref{ass:alternative} implies Assumption~\ref{ass:converging}, required by Theorem~\ref{thm:isotropic-nonseparable-aniso}.

Note that by Assumption~\ref{ass:alternative} we have
\begin{align}\label{eq:2-p-moment-2}
\pnorm{\vct{\mu}}^q \to \sigma_{M,p}^q d^{\frac{q}{2}-1}\,,\quad
\twonorm{\vct{\mu}} \to \sigma_{M,2} \,,
\end{align}
with high probability. 

We then write
\begin{align}
\lim_{n\to\infty}  e_{q,\mtx{I}}\left(\frac{c_0}{\sqrt{n}}\vct{h}- c_1 \tmu; \lambda_0 \pnorm{\vct{\mu}}^q\right)
&\stackrel{(a)}{=} \lim_{n\to\infty} \frac{1}{d} e_{q,\mtx{I}}\left(\frac{c_0\sqrt{d}}{\sqrt{n}}\vct{h}- c_1\sqrt{d} \tmu;\frac{\lambda_0\pnorm{\vct{\mu}}^q}{d^{\frac{q}{2}-1}}\right)\nn\\
&= \lim_{n\to\infty} \frac{1}{d} e_{q,\mtx{I}}\left(\frac{c_0\sqrt{d}}{\sqrt{n}}\vct{h}- c_1 \frac{\sqrt{d}\vct{\mu}}{\twonorm{\vct{\mu}}};\frac{\lambda_0\pnorm{\vct{\mu}}^q}{d^{\frac{q}{2}-1}}\right)\nn\\
&\stackrel{(b)}{=} \lim_{n\to\infty} \frac{1}{d} e_{q,\mtx{I}}\left(\frac{c_0}{\sqrt{\delta}}\vct{h}- c_1 \frac{\sqrt{d}\vct{\mu}}{\sigma_{M,2}};\lambda_0\sigma_{M,p}^q\right)\nn\\
&\stackrel{(c)}{=} \lim_{n\to\infty} \frac{1}{d} \sum_{i=1}^d J_q\left(\frac{c_0 h_i}{\sqrt{\delta}}  - c_1 \frac{\sqrt{d}\mu_i}{\sigma_{M,2}};\lambda_0 \sigma_{M,p}^q\right)\nn\\
&\stackrel{(d)}{=}  \E\left[ J_q\left(\frac{c_0}{\sqrt{\delta}} h- c_1  \frac{M}{\sigma_{M,2}};\lambda_0\sigma_{M,p}^q\right)\right]\nn\\
&= \EJ\left(c_0, c_1;\lambda_0\right)\,.\label{eq:Sigma-I-Exp}
\end{align}
Here $(a)$ follows from \eqref{eq:evn-prop2} with $b = \sqrt{d}$; $(b)$ follows from \eqref{eq:2-p-moment-2}; $(c)$ holds due to the identity $e_{q,\mtx{I}}(\vct{x};\lambda) = \sum_{i=1}^d J_q(x_i;\lambda)$, which follows readily from the definition of weighted Moreau envelop $e_{q,\mtx{I}}$ and the function $J_q$ given by~\eqref{eq:Jq}. Finally, $(d)$ holds due to Assumption~\ref{ass:alternative}. The series of equalities~\eqref{eq:Sigma-I-Exp} implies that Assumption~\ref{ass:converging} holds in isotropic case with 
\begin{align}\label{eq:E-J}
\sE(c_0,c_1;\lambda_0) = \EJ\left(c_0, c_1;\lambda_0\right)\,.
\end{align} 
Having Assumption~\ref{ass:converging2} in place, we can specialize the result of Theorem~\ref{thm:isotropic-nonseparable-aniso} to isotropic model. Substituting for $\sE(c_0,c_1;\lambda_0)$ from~\eqref{eq:E-J} in the AO problem~\eqref{thm:isotropic-nonseparable-aniso} yields the AO problem~\eqref{eq:main-thm}. 
 \section{Proofs for special cases of $p$ (Section~5)}\label{proof:sec5}
 \subsection{Proof of Corollary~\ref{cor:p2-sep}}
 Part (a) is already proved in Example 2, cf.~\eqref{eq:thresh-L2}.
 
 \begin{proof}[Part (b)] We start by an explicit characterization of $\EJ$ function for case of $p=q=2$.
 \begin{lemma}\label{lem:q2-2}
Recall function $\EJ(c_0,c_1;\lambda_0)$ given by
\begin{align}
\EJ(c_0,c_1;\lambda_0) = \E\left[J_q\left(\frac{c_0}{\sqrt{\delta}} h - c_1 \frac{M}{\sigma_{M,2}}; \lambda_0 \sigma_{M,p}^q\right)\right]\,,
\end{align}
Then the following identity holds for case of $p=q=2$:
\[
\EJ(c_0,c_1;\lambda_0) = \frac{\lambda}{\alpha^2+2\lambda\alpha}\twonorm{\vct{x}}^2
\]
\end{lemma}

\begin{proof}
It is straightforward to see that
\[
J_2(x;\lambda) = \frac{\lambda}{1+2\lambda} x^2\,.
\]
Therefore,
\[
\EJ(c_0,c_1;\lambda_0) = \frac{\lambda_0 \sigma_{M,2}^2}{1+2\lambda_0 \sigma_{M,2}^2} \E\left[\left(\frac{c_0}{\sqrt{\delta}} h - c_1 \frac{M}{\sigma_{M,2}}\right)^2\right] = \frac{\lambda_0 \sigma_{M,2}^2}{1+2\lambda_0 \sigma_{M,2}^2}  \left(\frac{c_0^2}{\delta}+c_1^2\right)\,.
\]
\end{proof}
%
Using Lemma~\ref{lem:q2-2} in AO problem~\eqref{eq:sep-AO9} for $q=2$, we have
\begin{align}
D_{\rm s}(\alpha,\gamma_{0},\theta,\beta,\lambda_{0},\eta,\tilde{\eta}) &= 2  \left(1+\frac{\eta}{2\alpha}\right)^{-1} 
\EJ\left(\frac{\beta}{2}, \frac{\tilde{\eta}}{2}; \frac{\lambda_0}{2\gamma_0}\left(1+\frac{\eta}{2\alpha}\right)^{-1}\right)\nn\\
&\;\;-\left(\frac{\beta^2}{\delta}+\tilde{\eta}^2\right) \frac{1}{4(1+\frac{\eta}{2\alpha})}-{\lambda_0} \gamma_0 -\frac{\eta\alpha}{2} -\tilde{\eta}\theta \nn\\
& \; \;+\beta\sqrt{\E\Bigg[\left(\left(1+\eps_0\gamma_0-\theta \sigma_{M,2} \right)+\alpha g\right)_{+}^2\Bigg]}\nn\\
&=\frac{\frac{\lambda_0}{4\gamma_0} \sigma_{M,2}^2}{(1+\frac{\eta}{2\alpha})^2 + (1+\frac{\eta}{2\alpha})\frac{\lambda_0}{\gamma_0} \sigma_{M,2}^2}  \left(\frac{\beta^2}{\delta}+ \tilde{\eta}^2\right)\nn\\
&\;\;-\left(\frac{\beta^2}{\delta}+\tilde{\eta}^2\right) \frac{1}{4(1+\frac{\eta}{2\alpha})}-{\lambda_0}  \gamma_0 -\frac{\eta\alpha}{2} -\tilde{\eta}\theta \nn\\
& \;\;+\beta\sqrt{\E\Bigg[\left(\left(1+\eps_0\gamma_0-\theta\sigma_{M,2}\right)+\alpha g\right)_{+}^2\Bigg]}\nn\\
&=
-\left(\frac{\beta^2}{\delta}+\tilde{\eta}^2\right) \frac{1}{4(1+\frac{\eta}{2\alpha} + \frac{\lambda_0}{\gamma_0}\sigma_{M,2}^2)}-{\lambda_0}  \gamma_0 -\frac{\eta\alpha}{2} -\tilde{\eta}\theta \nn\\
& \;\; +\beta\sqrt{\E\Bigg[\left(\left(1+\eps_0\gamma_0-\theta\sigma_{M,2} \right)+\alpha g\right)_{+}^2\Bigg]}\,.
\end{align}
Setting $\frac{\partial D_{{\rm s}}}{\partial\beta}$ to zero we conclude that
\begin{align*}
\widehat{\beta}=2\delta\left(1+\frac{\eta}{2\alpha} + \frac{\lambda_0}{\gamma_0}\sigma_{M,2}^2\right)\sqrt{\E\Bigg[\left(\left(1+\eps_0\gamma_0-\theta\sigma_{M,2} \right)+\alpha g\right)_{+}^2\Bigg]}\,.
\end{align*}
Thus the AO problem reduces to
\begin{align}\label{eq:sep-AO100}
&\min_{\alpha,\gamma_0\ge 0,\theta} \max_{\lambda_0,\eta\ge0, \tilde{\eta}} 
-\tilde{\eta}^2 \frac{1}{4(1+\frac{\eta}{2\alpha} + \frac{\lambda_0}{\gamma_0}\sigma_{M,2}^2)}-{\lambda_0}  \gamma_0 -\frac{\eta\alpha}{2} -\tilde{\eta}\theta \nn\\
&\quad\quad\quad\quad\quad\quad  +\delta\left(1+\frac{\eta}{2\alpha} + \frac{\lambda_0}{\gamma_0}\sigma_{M,2}^2\right)\E\Bigg[\left(\left(1+\eps_0\gamma_0-\theta\sigma_{M,2}\right)+\alpha g\right)_{+}^2\Bigg]\,.
\end{align}
Setting the derivative with respect to $\widetilde{\eta}$ to zero we arrive at
\begin{align*}
\widehat{\widetilde{\eta}}=-2\theta \left(1+\frac{\eta}{2\alpha} + \frac{\lambda_0}{\gamma_0}\sigma_{M,2}^2\right)\,,
\end{align*}
which further simplifies the AO problem to 
\begin{align}\label{eq:sep-AO1000}
&\min_{\alpha,\gamma_0\ge 0,\theta} \max_{\lambda_0,\eta\ge0} 
-{\lambda_0}  \gamma_0 -\frac{\eta\alpha}{2} +\theta^2\left(1+\frac{\eta}{2\alpha} + \frac{\lambda_0}{\gamma_0}\sigma_{M,2}^2\right) \nn\\
&\quad\quad\quad\quad\quad\quad  +\delta\left(1+\frac{\eta}{2\alpha} + \frac{\lambda_0}{\gamma_0}\sigma_{M,2}^2\right)\E\Bigg[\left(\left(1+\eps_0\gamma_0-\theta\sigma_{M,2} \right)+\alpha g\right)_{+}^2\Bigg]\,.
\end{align}
Note that if $\alpha^2< \theta^2+\delta\E\Bigg[\left(\left(1+\eps_0\gamma_0-\theta\sigma_{M,2} \right)+\alpha g\right)_{+}^2\Bigg]$ then the maximum over $\eta$ is $+\infty$. Furthermore, when $\alpha^2\ge \theta^2+\delta\E\Bigg[\left(\left(1+\eps_0\gamma_0-\theta\sigma_{M,2} \right)+\alpha g\right)_{+}^2\Bigg]$ then the optimal $\eta=0$. Thus the above AO is equivalent to
\begin{align}\label{eq:sep-AO1001}
&\min_{\alpha,\gamma_0\ge 0,\theta} \max_{\lambda_0\ge 0} 
\quad -{\lambda_0}  \gamma_0 +\theta^2\left(1 + \frac{\lambda_0}{\gamma_0}\sigma_{M,2}^2\right) \nn\\
&\quad\quad\quad\quad\quad\quad  +\delta\left(1+ \frac{\lambda_0}{\gamma_0}\sigma_{M,2}^2\right)\E\Bigg[\left(\left(1+\eps_0\gamma_0-\theta\sigma_{M,2} \right)+\alpha g\right)_{+}^2\Bigg]\nn\\
&\text{subject to}\quad \alpha^2\ge \theta^2+\delta\E\Bigg[\left(\left(1+\eps_0\gamma_0-\theta\sigma_{M,2}\right)+\alpha g\right)_{+}^2\Bigg]
\end{align}
Using a similar argument for optimization over $\lambda_0$, it is straightforward to see that the above optimization is equivalent to
\begin{align}\label{eq:sep-AO1002}
&\min_{\alpha,\gamma_0\ge 0,\theta}  
 \theta^2+\delta\E\Bigg[\left(\left(1+\eps_0\gamma_0-\theta\sigma_{M,2}\right)+\alpha g\right)_{+}^2\Bigg]\nn\\
&\text{subject to}\quad \alpha^2\ge \theta^2+\delta\E\Bigg[\left(\left(1+\eps_0\gamma_0-\theta\sigma_{M,2}\right)+\alpha g\right)_{+}^2\Bigg]\,,\nn\\
&\quad\text{and}\quad\frac{\gamma_0^2}{\sigma_{M,2}^2}\ge \theta^2+\delta\E\Bigg[\left(\left(1+\eps_0\gamma_0-\theta\sigma_{M,2} \right)+\alpha g\right)_{+}^2\Bigg]\,.
\end{align}

Since the objective function is increasing in $\alpha$ and $\gamma$, then the optimal $\alpha$ and $\gamma$ should make the inequality constraints equality and therefore $\gamma_0 = \alpha \sigma_{M,2}$. This brings us to the following problem:
\begin{align}
&\min_{\alpha\ge 0,u}  
 \alpha^2\nn\\
&\text{subject to}\quad \alpha^2 \ge \theta^2+\delta\E\left[\left(1+(\eps_0\alpha-\theta)\sigma_{M,2}+\alpha g\right)_{+}^2\right]
\,.
\end{align}

By change of variable $u = \frac{\theta}{\alpha}$ we have
\begin{align}
&\min_{\alpha\ge 0,u}  
 \alpha^2\nn\\
&\text{subject to}\quad 1\ge u^2+\delta\E\left[\left(\frac{1}{\alpha}+(\eps_0- u)\sigma_{M,2}+ g\right)_{+}^2\right]
\end{align}

By another change of variable $\tilde{\alpha} = \left(\frac{1}{\alpha}+\eps_0 \sigma_{M,2}\right)^{-1}$ we have
\begin{align}
&\min_{\frac{1}{\eps_0 \sigma_{M,2}}\ge \tilde{\alpha} \ge 0,u}  \left(\frac{1}{\tilde{\alpha}} - \eps_0 \sigma_{M,2}\right)^{-2}
\nn\\
&\text{subject to}\quad 1\ge u^2+\delta\E\left[\left(\frac{1}{\tilde{\alpha}}- u\sigma_{M,2}+ g\right)_{+}^2\right]
\end{align}
Since objective is increasing in $\tilde{\alpha}$ this is equivalent to
\begin{align}
&\min_{\frac{1}{\eps_0 \sigma_{M,2}}\ge \tilde{\alpha} \ge 0,u}  \tilde{\alpha}^2
\nn\\
&\text{subject to}\quad 1\ge u^2+\delta\E\left[\left(\frac{1}{\tilde{\alpha}}- u\sigma_{M,2}+ g\right)_{+}^2\right]
\end{align}
Note that we can drop the constraint $\frac{1}{\eps_0 \sigma_{M,2}}\ge \tilde{\alpha}$ because for $\frac{1}{\eps_0 \sigma_{M,2}} = \tilde{\alpha}$ one can already find $u$ that satisfies the inequality constraint. As such the optimal $\tilde{\alpha}$ should be less than $\frac{1}{\eps_0 \sigma_{M,2}}$. To see why, by letting $u = \frac{\theta}{\sqrt{1+\theta^2}}$ with $\theta$ the minimizer in separability condition~\eqref{eq:thresh-L2} we have
\begin{align*}
 u^2+\delta\E\left[\left(\frac{1}{\tilde{\alpha}}- u\sigma_{M,2}+ g\right)_{+}^2\right] &= \frac{\theta^2}{1+\theta^2}+\delta\E\left[\left((\eps_0- \frac{\theta}{\sqrt{1+\theta^2}})\sigma_{M,2}+ g\right)_{+}^2\right]\nn\\
 &=  \frac{\theta^2}{1+\theta^2}+\frac{\delta}{1+\theta^2}\E\left[\left((\sqrt{1+\theta^2}\eps_0- \theta)\sigma_{M,2}+ \sqrt{1+\theta^2}g\right)_{+}^2\right]\nn\\
 &\le \frac{\theta^2}{1+\theta^2}+ \frac{1}{1+\theta^2} = 1\,.
\end{align*}
This brings us to the following AO problem:
\begin{align}
&\min_{\tilde{\alpha} \ge 0,u}  \tilde{\alpha}^2
\nn\\
&\text{subject to}\quad 1\ge u^2+\delta\E\left[\left(\frac{1}{\tilde{\alpha}}- u\sigma_{M,2}+ g\right)_{+}^2\right]
\end{align}
Denoting by $\tilde{\alpha}_*$ the solution of the above problem, it is clear that by our change of variable we have
\begin{align}
\alpha_* = \left(\tilde{\alpha}_*^{-1} - \eps_0 \sigma_{M,2}\right)^{-1}\,, \quad \theta_* = u_*\alpha_*, \quad \gamma_{0*} = \alpha_* \sigma_{M,2}\,.
\end{align}
This concludes the proof of part (b).
\end{proof}
\smallskip

\begin{proof}[Part (c)]
We focus on part of the AO problem~\eqref{eq:main-thm} that involves the variables $\lambda_0, \nu, \tau_h$ and specialize it to the case of $q=2$:
\[
\min_{\lambda_0\ge0,\nu} \left[\frac{\alpha}{\tau_h}\left\{\frac{\beta^2}{2\delta} + \lambda_0 \left(\frac{\gamma_0\tau_h}{\alpha}\right)^2  + \frac{\nu^2}{2}
-\EJ\left(\beta,\left(\frac{\tau_h\theta}{\alpha} + {\nu}\right) ;\lambda_0\right)\right\} +\frac{\alpha\tau_h}{2} \right]\,.
\]
We next plug in for $\EJ(c_0,c_1;\lambda_0)$ from Lemma~\ref{lem:q2-2} which results in
\[
\min_{\lambda_0\ge0,\nu} \left[\frac{\alpha}{\tau_h}\left\{\frac{\beta^2}{2\delta} + \lambda_0 \left(\frac{\gamma_0\tau_h}{\alpha}\right)^2  + \frac{\nu^2}{2}
-\frac{\lambda_0 \sigma_{M,2}^2}{1+2\lambda_0 \sigma_{M,2}^2}  \left(\frac{\beta^2}{\delta}+\left(\frac{\tau_h\theta}{\alpha} + {\nu}\right)^2\right) \right\} +\frac{\alpha\tau_h}{2} \right]\,.
\]
Writing the first order optimality for $\lambda_0$, $\nu$, $\tau_h$ we get a set of equations that admits a solution only if $\gamma_0 = \sigma_{M,2}\sqrt{\alpha^2+\theta^2}$. Then, 
\[
\nu = 2\lambda_0\sigma_{M,2}^2\frac{\tau_h\theta}{\alpha}\,,\quad \tau_h = \frac{1}{1+2\lambda_0\sigma_{M,2}^2} \frac{\beta}{\sqrt{\delta}}\,.\]
In this case, the value of $\lambda_0$ does not matter and the above part of the AO simplifies to $\alpha\beta/\sqrt{\delta}$.

This simplifies the AO problem~\eqref{eq:main-thm} to 
\begin{align}\label{AO:p2-F}
&\max_{0\le\beta}\;\; \min_{\theta, 0\le \alpha, \tau_g}\;\; \frac{\beta \tau_g}{2} +L\left(\sqrt{\alpha^2+\theta^2},\sigma_{M,2} \theta - \eps_0 \sqrt{\alpha^2+\theta^2},\frac{\tau_g}{\beta}\right) 
- \frac{\alpha\beta}{\sqrt{\delta}}\,.
\end{align}
We next further simplifies the AO problem by solving for $\tau_g$. We use the shorthand $L_3'(a,b;\mu) = \frac{\partial L}{\partial \mu}L(a,b;\mu)$ to denote the derivative of the expected Moreau envelop with respect to its third argument. 
 Writing the first order optimality condition for $\beta$ and $\tau_g$ in optimization~\eqref{AO:p2-F}, we get
 \begin{align}
 &\frac{\beta}{2} +\frac{1}{\beta} L_3'\left(\sqrt{\alpha^2+\theta^2},\sigma_{M,2} \theta - \eps_0 \sqrt{\alpha^2+\theta^2},\frac{\tau_g}{\beta}\right) = 0\,,\nn\\
 &\frac{\tau_g}{2} -\frac{\tau_g}{\beta^2} L_3'\left(\sqrt{\alpha^2+\theta^2},\sigma_{M,2} \theta - \eps_0 \sqrt{\alpha^2+\theta^2},\frac{\tau_g}{\beta}\right) - \frac{\alpha}{\sqrt{\delta}}=0\,.
 \end{align}
 Combining the above two equations, we obtain $\tau_g = \frac{\alpha}{\sqrt{\delta}}$. Substituting for $\tau_g$ in~\eqref{AO:p2-F}, the AO problem for case of $q=2$ simplifies to 
 \begin{align}
&\max_{0\le\beta}\;\; \min_{\theta, 0\le \alpha}\;\; L\left(\sqrt{\alpha^2+\theta^2}, \sigma_{M,2} \theta - \eps_0 \sqrt{\alpha^2+\theta^2},\frac{\alpha}{\beta\sqrt{\delta}}\right) 
- \frac{\alpha\beta}{2\sqrt{\delta}}\,.
\end{align}
This completes the proof of part (c).
\end{proof}
\subsection{Proof of Theorem~\ref{thm:isotropic-qinf}}
\begin{proof}[Part (a)] The first part of the theorem is on precise characterization of the separability threshold. 
We use the result of Theorem~\ref{thm:sep-thresh} that holds for any choice of $(p,q)$, in particular $(p=1,q=\infty)$, and relates the
separability threshold to the spherical width. What is remaining to prove is the characterization of the spherical width given by~\eqref{eq:omega-p1}. To this end, we follow a similar argument as in Lemma~\ref{lem:justify}. However, since $\qnorm{\cdot}^q$ is not well defined for $q=\infty$ (recall that $\ell_q$ is the dual norm of $\ell_p$ and $p=1$), it requires a slightly different analysis. Specifically, in the Lagrangian we write the constraint $\infnorm{\vct{u}} \le \frac{1}{\eps_0\infnorm{\vct{\mu}}}$ as the term $\infnorm{\vct{u}} - \frac{1}{\eps_0\pnorm{\vct{\mu}}}$, as compared to the case of finite $q$ where we raised the both sides to power $q$ to use the separability property of function $\qnorm{\cdot}^q$. Then, by following a similar derivation as in~\eqref{J-lem-1}, we obtain
\begin{align}
&\min_{\bz\in\cS(\alpha,\theta,\eps_0,\vct{\mu})}\;\; -\frac{1}{\sqrt{n}}  \vct{h}^T\bz\nn\\
&=\sup_{\lambda,\eta\ge0, \nu}  \;\; \eta  J_\infty\left(\frac{\vct{h}}{\eta\sqrt{n}} - \left(\frac{\nu}{\eta}-\theta\right)\tmu;\frac{\lambda}{\eta} \right)  - \frac{\nu^2}{2\eta}-\frac{1}{2\eta n}\twonorm{\vct{h}}^2 + \frac{\nu}{\eta\sqrt{n}} \tmu^T\vct{h} - \frac{\eta}{2}\alpha^2 - \frac{\lambda}{\eps_0\onenorm{\vct{\mu}}}\,,\label{J-lem-1-e}
\end{align} 
where $J_\infty(\vct{x};\lambda)$ is defined as
\begin{align}\label{eq:Jinf}
J_\infty(\vct{x};\lambda) = \min_{\vct{v}} \frac{1}{2} \twonorm{\vct{x}-\vct{v}}^2+\lambda\infnorm{\vct{v}}\,.
\end{align}
Our next step is scalarization of the optimization~\eqref{J-lem-1-e} in the large sample limit (as $n\to\infty$), and a challenge along this way is that the function $J(\vct{x};\lambda) $ is not a separable function over the entries of $\vct{x}$. To cope with this problem, we propose an alternative representation of this function that involves an additional variable $t_0$.

We write
\begin{align}\label{eq:Morea-p1}
J_\infty(\vct{x};\lambda) &= \min_{\bv} \frac{1}{2} \twonorm{\vct{x}-\bv}^2+\lambda\infnorm{\bv}\nn\\
&= \min_{\bv, t\ge0} \frac{1}{2} \twonorm{\vct{x}-\bv}^2+\lambda t\quad \text{subject to }\quad \infnorm{\bv}\le t\nn\\
&= \min_{t\ge0} \frac{1}{2} \twonorm{\ST(\vct{x};t)}^2+\lambda t\,.
\end{align}
Let $t_0 = \onenorm{\vct{\mu}} t$ and $\lambda_ 0  = \frac{\lambda}{\onenorm{\vct{\mu}}}$.  Similar to the trick of `artificial' boundedness that we used in applying the CGMT framework (e.g., cf. explanation after~\eqref{generalAO1} and Appendix~\ref{sec:artifical}), we continue by the ansatz that the optimal value of $\lambda_0$ and $t_0$ remain bounded as $n\to \infty$. After we take the limit of the Lagrangian to obtain a scalar auxiliary optimization (AO) problem, this ansatz is verified by the boundedness of solutions of the AO problem.

For $\vct{x} = \frac{c_0}{\sqrt{n}}\vct{h}- c_1 \tmu$  we have
\begin{align}
&\lim_{n\to\infty}  \frac{1}{2} \twonorm{\ST(\vct{x};t)}^2 +\lambda t \nn\\
&= \lim_{n\to\infty}  \frac{1}{2} \sum_{i=1}^d \ST\left(x_i;\frac{t_0}{\onenorm{\vct{\mu}}}\right)^2 +\lambda_0 t_0\nn\\
&= \lim_{n\to\infty}  \frac{1}{2} \sum_{i=1}^d \ST\left(\frac{c_0}{\sqrt{n}} h_i- c_1 \tmu_i ;\frac{t_0}{\onenorm{\vct{\mu}}}\right)^2 +\lambda_0 t_0 \nn\\
&\stackrel{(a)}{=} \lim_{n\to\infty}  \frac{1}{2d} \sum_{i=1}^d \ST\left(\frac{c_0}{\sqrt{\delta}} h_i- c_1 \frac{\sqrt{d}\mu_i}{\sigma_{M,2}} ;\frac{t_0}{\sigma_{M,1}}\right)^2 +\lambda_0 t_0 \nn\\
&=   \frac{1}{2} \E\left[\ST\left(\frac{c_0}{\sqrt{\delta}} h- c_1 \frac{M}{\sigma_{M,2}} ;\frac{t_0}{\sigma_{M,1}}\right)^2\right] +\lambda_0 t_0 \,,\nn\\
&= f(c_0,c_1;t_0) +\lambda_0 t_0\,,
\end{align}
with high probability. In $(a)$ we used the fact that as $n\to \infty$, we have $\onenorm{\vct{\mu}}\to \sqrt{d}\sigma_{M,1}$ along with the identity $\frac{1}{a^2}\ST(a\vct{x};a\lambda) = \ST(\vct{x};\lambda)$. 

Note that this is a pointwise convergence. However, the left hand side is a convex function of $t$ and it's minimizer satisfies $t\le \|\vct{x}\|_\infty \le c_0 + c_1$ and so belongs to a compact set. Therefore, by applying the convexity lemma, see e.g, \cite[Lemma 7.75]{StatDecision}, \cite[Lemma B1]{thrampoulidis2018precise}, we can change the order of limit and minimization, and  get that  for $\vct{x} = \frac{c_0}{\sqrt{n}}\vct{h}- c_1 \tmu$ ,
\begin{align}\label{eq:convex-ST}
&\lim_{n\to\infty} J_\infty(\vct{x};\lambda)\nn\\
&=\lim_{n\to\infty} \min_{t\ge0} \left\{ \frac{1}{2} \twonorm{\ST(\vct{x};t)}^2 +\lambda t\right\} \nn\\
&= \min_{t_0\ge0}\left\{f(c_0,c_1;t_0) +\lambda_0 t_0\right\}\,,
\end{align}
in probability. In addition, as $n\to \infty$ we have
\begin{align}
\frac{1}{n}\twonorm{\vct{h}}^2 \to \frac{1}{\delta},\quad \frac{1}{\sqrt{n}} \tmu^T\vct{h}\to 0\,,
\quad \onenorm{\vct{\mu}} \to \sigma_{M,1} \sqrt{d}\,,
\end{align}
with high probability.

Using the above limits, we see that the objective function~\eqref{J-lem-1-e} converges pointwise to the following function:
\begin{align}\label{eq:J-lem-2-e}
\min_{t_0\ge0}\; \eta  \left\{f\left(\frac{1}{\eta},\frac{\nu}{\eta}-\theta;t_0\right) + \frac{\lambda_0}{\eta} t_0\right\} - \frac{\nu^2}{2\eta}-\frac{1}{2\eta \delta}  - \frac{\eta}{2}\alpha^2 - \frac{\lambda_0}{\eps_0} 
\end{align} 
Note that \eqref{J-lem-1} is the dual optimization and hence is a concave problem. We apply the convexity lemma~\cite[Lemma B.2]{thrampoulidis2018precise} to conclude that the objective value in \eqref{J-lem-1} also converges to the supremum of function~\eqref{eq:J-lem-2} over $\lambda_0,\eta\ge0, \nu$. Therefore the solution of optimization~\eqref{J-lem-1-e} converges to the solution of the following optimization problem:
\begin{align}\label{eq:J-lem-3-e}
\sup_{\lambda_0,\eta\ge0, \nu} \min_{t_0\ge0}\; \eta  f\left(\frac{1}{\eta},\frac{\nu}{\eta}-\theta;t_0\right) + \lambda_0 t_0  - \frac{\nu^2}{2\eta}-\frac{1}{2\eta \delta}  - \frac{\eta}{2}\alpha^2 - \frac{\lambda_0}{\eps_0} 
\end{align} 

Note that the above objective is linear in $\lambda_0$. Therefore the optimal $t_0^*$ should satisfy $t_0^*\le \frac{1}{\eps_0}$. Otherwise $\lambda_0^* = \infty$ which makes the above max-min value unbounded, and this is a contradiction because the above problem involves minimization over $t_0$ and it is easy to see that by choosing $t_0 = 0$ the optimal objective value over $\{\lambda_0, \eta\ge0, \nu\}$ becomes zero.

Therefore, we can assume $t_0\le \frac{1}{\eps_0}$ which yields $\lambda_0^*(t_0 - \frac{1}{\eps_0}) = 0$. This simplifies the problem~\eqref{eq:J-lem-3-e} to 
\begin{align}\label{eq:J-lem-4-e}
\sup_{\eta\ge0, \nu} \min_{0\le t_0\le\frac{1}{\eps_0}}\; \eta  f\left(\frac{1}{\eta},\frac{\nu}{\eta}-\theta;t_0\right)   - \frac{\nu^2}{2\eta}-\frac{1}{2\eta \delta}  - \frac{\eta}{2}\alpha^2 
\end{align} 
Since $f(c_0,c_1;t)$ is decreasing in $t$, the optimal value $t_0^*$ is given by $t_0^* = \frac{1}{\eps_0}$ and the problem is further simplified and along with \eqref{J-lem-1-e} implies that
\begin{align}\label{eq:J-lem-5-e}
&\min_{\bz\in\cS(\alpha,\theta,\eps_0,\vct{\mu})}\;\; -\frac{1}{\sqrt{n}}  \vct{h}^T\bz\nn\\
&=\sup_{\eta\ge0, \nu} \; \eta  f\left(\frac{1}{\eta},\frac{\nu}{\eta}-\theta;\frac{1}{\eps_0}\right)   - \frac{\nu^2}{2\eta}-\frac{1}{2\eta \delta}  - \frac{\eta}{2}\alpha^2 
\end{align} 
Now similar to the proof of Lemma~\ref{lem:justify} we use Equation~\eqref{eq:htz} to write 
\begin{align}
&\omega(\alpha, \theta, \eps_0) \nn\\
&= \lim_{n\to\infty} \omega_s\left(\cS(\alpha, \theta, \eps_0, \vct{\mu})\right)\nn\\
&= - \sup_{\eta\ge0, \nu} \; \sqrt{\delta}\left\{\eta  f\left(\frac{1}{\eta},\frac{\nu}{\eta}-\theta;\frac{1}{\eps_0}\right)   - \frac{\nu^2}{2\eta}-\frac{1}{2\eta \delta}  - \frac{\eta}{2}\alpha^2\right\}\nn\\
&= \min_{\eta\ge0, \nu}  \;\; \sqrt{\delta} \left\{\frac{\nu^2}{2\eta}+\frac{1}{2\eta \delta} + \frac{\eta}{2}\alpha^2 -\eta f\bigg(\frac{1}{\eta},
 \frac{\nu}{\eta}-\theta;\frac{1}{\eps_0}\bigg)\right\} \,.
\end{align}
This completes the proof.

\end{proof}
\smallskip

\begin{proof}[Part (b)] The proof of this parts proceeds along the same lines of Theorem~\ref{thm:isotropic-separable} for the special case of $p=1$, $q=\infty$. However, it requires a slightly different treatment as in part (a) because the function $\qnorm{\cdot}^q$ and therefore $J_q$ given by~\eqref{eq:Jq} are not well defined in this case.

Here we only highlight the modifications that are needed to the proof of Theorem~\ref{thm:isotropic-separable} to apply it for case of $q=\infty$.

We proceed the exact same derivation that yields \eqref{eq:sep-AO3}, repeated here for convenience:
\begin{align}\label{eq:FF1}
&\min_{\theta,\bth, \gamma\ge 0, \alpha\ge 0}\text{ }\max_{\beta,\lambda,\eta\ge0, \widetilde{\eta}}\quad
\twonorm{\vct{\theta}}^2 +2\lambda \left(\qnorm{\bth}- \gamma\right)+\frac{\beta}{\sqrt{n}}  \vct{h}^T  \pproj_{\vct{\mu}}\mtx{\Sigma}^{1/2} \bth\nn\\ 
&\quad\quad\quad\quad+\beta\sqrt{\E\Bigg[\left(\left(1+\eps\gamma-\theta\twonorm{\vct{\mu}}\right)+\alpha g\right)_{+}^2\Bigg]}
+\eta \left(\twonorm{\mtx{\Sigma}^{\frac{1}{2}}\vct{\theta}}- \alpha\right) +\widetilde{\eta} \left(\tmu^T\bth - \theta\right)
\end{align}
We substitute for $\twonorm{\mtx{\Sigma}^{\frac{1}{2}}\vct{\theta}}$ using the identity
$\twonorm{\mtx{\Sigma}^{\frac{1}{2}}\vct{\theta}}=\min_{\tau\ge 0} \frac{\twonorm{\mtx{\Sigma}^{\frac{1}{2}}\vct{\theta}}^2}{2\tau}+\frac{\tau}{2}$ 
to get
\begin{align}
&\min_{\theta,\bth, \gamma\ge 0, \alpha\ge 0}\text{ }\max_{\beta,\lambda,\eta\ge0, \widetilde{\eta}}\text{ }\min_{\tau\ge 0}\quad
\twonorm{\vct{\theta}}^2 +{2\lambda} \infnorm{\bth}- 2\lambda\gamma+\frac{\beta}{\sqrt{n}}  \vct{h}^T  \pproj_{\vct{\mu}}\mtx{\Sigma}^{1/2} \bth \nn\\
&\quad \quad\quad\quad\quad\quad+\beta\sqrt{\E\Bigg[\left(\left(1+\eps\gamma-\theta\twonorm{\vct{\mu}}\right)+\alpha g\right)_{+}^2\Bigg]}\nn\\
&\quad \quad\quad\quad\quad\quad+ \frac{\eta}{2\tau}\twonorm{\mtx{\Sigma}^{\frac{1}{2}}\vct{\theta}}^2+\frac{\eta\tau}{2}- \eta\alpha+\widetilde{\eta} \left(\tmu^T\bth - \theta\right)\,.
\end{align}
Specializing it to $\mtx{\Sigma} = \mtx{I}$ and $q=\infty$, the optimization over $\bth$ takes the form
\begin{align}
&\min_{\bth} \quad \left(1+\frac{\eta}{2\tau}\right)\twonorm{\vct{\theta}}^2 + 2\lambda \infnorm{\bth}+\frac{\beta}{\sqrt{n}}  \vct{h}^T  \pproj_{\vct{\mu}} \bth+\widetilde{\eta}\tmu^T\bth\,.
\end{align}
Note that
\begin{align*}
\frac{\beta}{\sqrt{n}}\; \pproj_{\vct{\mu}} \vct{h}-  \frac{\teta\vct{\mu}}{\twonorm{\vct{\mu}}} &= 
\frac{\beta}{\sqrt{n}}\; \vct{h}-  \frac{\teta\vct{\mu}}{\twonorm{\vct{\mu}}} -\frac{\beta}{\sqrt{n}}\; \proj_{\vct{\mu}} \vct{h}\nn\\
&=\frac{\beta}{\sqrt{n}}\; \vct{h}-   
\left(\frac{\teta}{\twonorm{\vct{\mu}}} + \frac{\beta}{\sqrt{n}\twonorm{\vct{\mu}}} \vct{h}^T \tmu\right)
\vct{\mu}\,,
\end{align*}
where $\tmu^T\vct{h} \sim \normal(0,1)$ since $\twonorm{\tmu}=1$, and $\twonorm{\vct{\mu}} \to \sigma_{M,2}$ which implies that the last term in the right-hand side is dominated by the second term. Therefore in the asymptotic regime $n\to\infty$, we can equivalently work replace $\pproj_{\vct{\mu}}\vct{h}$ by $\vct{h}$ and by using the symmetry of the Gaussian distribution work with
\begin{align}
\min_{\bth} \quad \left(1+\frac{\eta}{2\tau}\right)\twonorm{\vct{\theta}}^2 + 2\lambda \infnorm{\bth}-\frac{\beta}{\sqrt{n}}  \vct{h}^T  \bth+\widetilde{\eta}\tmu^T\bth \,.
\end{align}

We let $\vct{x} = \frac{\beta}{\sqrt{n}}\vct{h} - \teta\tmu$ and consider the change of variable $\vct{u}:= 2(1+\frac{\eta}{2\tau})\bth$ and write 
\begin{align}\label{eq:FF2}
&\min_{\bth} \quad \left(1+\frac{\eta}{2\tau}\right)\twonorm{\vct{\theta}}^2 + 2\lambda \infnorm{\bth}-\frac{\beta}{\sqrt{n}}  \vct{h}^T  \bth+\teta \tmu^T\bth\nn\\
&=\min_{\bu} \frac{1}{2}\left(1+\frac{\eta}{2\tau}\right)^{-1} \left\{\frac{1}{2}\twonorm{\bu}^2+ 2\lambda\infnorm{\bu} -\vct{x}^T\bu \right\}\nn\\
&=\min_{\bu} \frac{1}{2}\left(1+\frac{\eta}{2\tau}\right)^{-1} \left\{\frac{1}{2}\twonorm{\bu -\vct{x}}^2+ 2\lambda\infnorm{\bu} -\frac{\twonorm{\vct{x}}^2}{2} \right\}\nn\\
&=\frac{1}{2}\left(1+\frac{\eta}{2\tau}\right)^{-1} J_\infty(\vct{x};2\lambda) - \frac{1}{4}\left(1+\frac{\eta}{2\tau}\right)^{-1} \twonorm{\vct{x}}^2
\end{align}
Using~\eqref{eq:FF2} and substituting for $\vct{x}$ in~\eqref{eq:FF1}, our AO problem becomes 
\begin{align}\label{eq:FF3}
\min_{\theta, \gamma\ge 0, \alpha\ge 0}\text{ }\max_{\beta,\lambda,\eta\ge0, \widetilde{\eta}}\text{ }\min_{\tau\ge 0}\quad&
\frac{1}{2}\left(1+\frac{\eta}{2\tau}\right)^{-1} J_\infty\left(\frac{\beta}{\sqrt{n}}\vct{h} - \teta\tmu ;2\lambda\right) - \frac{1}{4}\left(1+\frac{\eta}{2\tau}\right)^{-1} \twonorm{\frac{\beta}{\sqrt{n}}\vct{h} - \teta\tmu }^2 \nn\\
&- 2\lambda\gamma
+\beta\sqrt{\E\Bigg[\left(\left(1+\eps\gamma-\theta\twonorm{\vct{\mu}}\right)+\alpha g\right)_{+}^2\Bigg]}
+\frac{\eta\tau}{2}- \eta\alpha -\widetilde{\eta}  \theta
\end{align}
Our next step is to scalarize the AO problem by taking the asymptotic limit of the objective.

We have 
\[
\lim_{n\to\infty }\twonorm{\frac{\beta}{\sqrt{n}}\vct{h} - \teta\tmu }^2 = \frac{\beta^2}{\delta} + \teta^2\,, \quad
\lim_{n\to\infty} \twonorm{\vct{\mu}} = \sigma_{M,2}\,.
\]
in probability. Also by using~\eqref{eq:convex-ST} we have
\[
\lim_{n\to\infty} J_\infty\left(\frac{\beta}{\sqrt{n}}\vct{h} - \teta\tmu ;2\lambda\right) = \min_{t_0\ge0}\left\{f(\beta,\teta;t_0) +2\lambda_0 t_0\right\}\,,
\]
with $\lambda_0 = \frac{\lambda}{\onenorm{\vct{\mu}}}$.
Using these limits in the AO problem~\eqref{eq:FF3} we obtain the following scalar AO problem
\begin{align}\label{eq:FF4}
\min_{\theta, \gamma\ge 0, \alpha\ge 0}\text{ }\max_{\beta,\lambda,\eta\ge0, \widetilde{\eta}}\text{ }\min_{\tau, t_0\ge 0}\quad&
\frac{1}{2}\left(1+\frac{\eta}{2\tau}\right)^{-1} (f(\beta,\teta;t_0) +2\lambda_0 t_0)- \frac{1}{4}\left(1+\frac{\eta}{2\tau}\right)^{-1} \left(\frac{\beta^2}{\delta} + \teta^2\right) \nn\\
&- 2\lambda_0\gamma_0
+\beta\sqrt{\E\Bigg[\left(\left(1+\eps_0\gamma_0-\theta\sigma_{M,2}\right)+\alpha g\right)_{+}^2\Bigg]}
+\frac{\eta\tau}{2}- \eta\alpha -\widetilde{\eta}  \theta\,,
\end{align}
where we recall our notation $ \eps_0 =\frac{\eps }{\onenorm{\vct{\mu}}}$ and $\gamma_0 = \gamma \onenorm{\vct{\mu}}$.

Now note that objective function~\eqref{eq:FF4} is linear in $\lambda_0$ and therefore the optimal $t_0^*$ should satisfy $t_0^*\le 2\gamma_0 \left(1+\frac{\eta}{2\tau}\right)$, otherwise $\lambda_0^* = \infty$ which makes the above max-min value unbounded. As such, we also have $\lambda_0^* = 0$ which further simplifies the problem as follows:
\begin{align}
\min_{\theta, \gamma\ge 0, \alpha\ge 0}\text{ }\max_{\beta,\eta\ge0, \widetilde{\eta}}\text{ }\min_{0\le\tau, 0\le t_0\le 2\gamma_0 (1+\frac{\eta}{2\tau})}\quad&
\frac{1}{2}\left(1+\frac{\eta}{2\tau}\right)^{-1} f(\beta,\teta;t_0) - \frac{1}{4}\left(1+\frac{\eta}{2\tau}\right)^{-1} \left(\frac{\beta^2}{\delta} + \teta^2\right) \nn\\
&+\beta\sqrt{\E\Bigg[\left(\left(1+\eps_0\gamma_0-\theta\sigma_{M,2} \right)+\alpha g\right)_{+}^2\Bigg]}
+\frac{\eta\tau}{2}- \eta\alpha -\widetilde{\eta}  \theta\,,
\end{align}
Since $f(c_0,c_1;t_0)$ is decreasing in $t_0$, the optimal value of $t_0$ is given by $t_0^* = 2\gamma_0 (1+\frac{\eta}{2\tau})$ which results in the following AO problem:
\begin{align}
\min_{\theta, \gamma\ge 0, \alpha\ge 0}\text{ }\max_{\beta,\eta\ge0, \widetilde{\eta}}\text{ }\min_{\tau\ge 0}\quad&
\frac{1}{2}\left(1+\frac{\eta}{2\tau}\right)^{-1} f\left(\beta,\teta;2\gamma_0 (1+\frac{\eta}{2\tau})\right) - \frac{1}{4}\left(1+\frac{\eta}{2\tau}\right)^{-1} \left(\frac{\beta^2}{\delta} + \teta^2\right) \nn\\
&+\beta\sqrt{\E\Bigg[\left(\left(1+\eps_0\gamma_0-\theta\sigma_{M,2} \right)+\alpha g\right)_{+}^2\Bigg]}
+\frac{\eta\tau}{2}- \eta\alpha -\widetilde{\eta}  \theta\,.
\end{align}
Our final step of simplification is to solve for $\tau$. To this end, we define the function
\[
R\left(\frac{\eta}{\tau}\right):=\frac{1}{2}\left(1+\frac{\eta}{2\tau}\right)^{-1} f\left(\beta,\teta;2\gamma_0 (1+\frac{\eta}{2\tau})\right) - \frac{1}{4}\left(1+\frac{\eta}{2\tau}\right)^{-1} \left(\frac{\beta^2}{\delta} + \teta^2\right)\,,
\]
where we make the dependence on $\frac{\eta}{\tau}$ explicit in the notation. Setting derivative of the AO objective with respect to $\eta$, to zero we obtain
\[
\frac{1}{\tau} R'\left(\frac{\eta}{\tau}\right) + \frac{\tau}{2} -\alpha = 0\,.
\]
Setting derivative with respect to $\tau$ to zero gives
\[
-\frac{\eta}{\tau^2} R'\left(\frac{\eta}{\tau}\right) +\frac{\eta}{2} = 0\,.
\]
Combining the above two optimality condition implies that $\eta(1-\frac{\alpha}{\tau}) = 0$. So either $\alpha = \tau$ or $\eta = 0$. If $\eta=0$, then it is clear that the terms involving $\tau$ in the AO problem would vanish and therefore the value of $\tau$ does not matter. So in this case, we can as well assume $\tau = \alpha$. Substituting for $\tau$ the AO problem further simplifies to 
\begin{align}
\min_{\theta, \gamma\ge 0, \alpha\ge 0}\text{ }\max_{\beta,\eta\ge0, \widetilde{\eta}}\text{ }\quad&
\frac{1}{2}\left(1+\frac{\eta}{2\alpha}\right)^{-1} f\left(\beta,\teta;2\gamma_0 (1+\frac{\eta}{2\alpha})\right) - \frac{1}{4}\left(1+\frac{\eta}{2\alpha}\right)^{-1} \left(\frac{\beta^2}{\delta} + \teta^2\right) \nn\\
&+\beta\sqrt{\E\Bigg[\left(\left(1+\eps_0\gamma_0-\theta\sigma_{M,2}\right)+\alpha g\right)_{+}^2\Bigg]}
-\frac{\eta\tau}{2} -\widetilde{\eta}  \theta\,.
\end{align}
This concludes the proof of part (b).
\end{proof}

\begin{proof}[Part (c)] 
The proof of part (c) follows along the same lines of the proof of Theorem~\ref{thm:isotropic-nonseparable-aniso} (and Theorem~\ref{thm:isotropic-nonseparable} for isotropic case). But similar to previous parts, we need to make slight modifications to the proof.

Note that in our derivation of the AO problem~\eqref{eq:AO-gen-ns}, we replaced the constraint $\qnorm{\vct{u}}\le \frac{\tau_h}{\alpha}\frac{\gamma_0}{\pnorm{\vct{\mu}}}$ with the equivalent constraint $\qnorm{\vct{u}}^q\le (\frac{\tau_h}{\alpha}\frac{\gamma_0}{\pnorm{\vct{\mu}}})^q$, see~\eqref{eq:OPT1} for more details. The benefit of this alternative representation is that it results in the Moreau-envelope $e_{q,\mtx{\Sigma}}$ of the $\qnorm{\cdot}^q$ function, see~\eqref{eq:OPT4}, which is separable over the samples. As a result, in the isotropic case the expected Moreau envelope reduces to the expected of the one-dimensional function $J_q$, given by~\eqref{eq:Jq}, \eqref{eq:Sigma-I-Exp}. 

However, for $q=\infty$ the function $\qnorm{\cdot}^q$ is not well-defined and requires a slightly different treatment. In this case we stay with the original constraint $\qnorm{\vct{u}}\le \frac{\tau_h}{\alpha}\frac{\gamma_0}{\pnorm{\vct{\mu}}}$. Proceeding along the same derivations of AO problem~\eqref{eq:main-thm}, it is straightforward to see that this results in the following AO problem for the non-separable regime: 
\begin{align}\label{eq:FF5}
&\max_{0\le\beta, \tau_h}\;\; \min_{\theta, 0\le \alpha, \gamma_0, \tau_g}\;\; D_{{\rm ns}}(\alpha, \gamma_0, \theta, \tau_g, \beta,\tau_h)\nn\\
& D_{{\rm ns}}(\alpha, \gamma_0, \theta, \tau_g, \beta,\tau_h) = \frac{\beta \tau_g}{2} +L\left(\sqrt{\alpha^2+\theta^2},\sigma_{M,2} \theta-\eps_0 \gamma_0,\frac{\tau_g}{\beta}\right) \nn \\
&\quad\quad\quad\quad\quad\quad\quad\quad\quad- \min_{\lambda_0\ge0,\nu} \left[\frac{\alpha}{\tau_h}\left\{\frac{\beta^2}{2\delta} + \lambda_0 \frac{\gamma_0\tau_h}{\alpha}  + \frac{\nu^2}{2}
-\tilde{\sE}\left(\beta,\frac{\tau_h\theta}{\alpha} + {\nu} ;\lambda_0\right)\right\} +\frac{\alpha\tau_h}{2} \right]
\end{align}
with
\begin{align}\label{eq:E-p1}
\tilde{\sE}(c_0,c_1;\lambda_0): = \lim_{n\to\infty}  J_\infty\left(\frac{c_0}{\sqrt{n}}\vct{h}- c_1 \tmu; \lambda_0 \onenorm{\vct{\mu}}\right)\,,
\end{align}
and $J_{\infty}(\vct{x};\lambda)$ given by~\eqref{eq:Jinf}. Using \eqref{eq:convex-ST}, we have
\begin{align}
\tilde{\sE}(c_0,c_1;\lambda_0)
= \min_{t_0\ge0}\left\{f(c_0,c_1;t_0) +\lambda_0 t_0\right\}\,,
\end{align}
Substituting for $\tilde{\sE}$ function in the AO problem~\eqref{eq:FF5} results in
\begin{align}
&\max_{0\le\beta, \tau_h}\;\; \min_{\theta, 0\le \alpha, \gamma_0, \tau_g}\;\; D_{{\rm ns}}(\alpha, \gamma_0, \theta, \tau_g, \beta,\tau_h)\nn\\
& D_{{\rm ns}}(\alpha, \gamma_0, \theta, \tau_g, \beta,\tau_h) = \frac{\beta \tau_g}{2} +L\left(\sqrt{\alpha^2+\theta^2},\sigma_{M,2} \theta-\eps_0 \gamma_0,\frac{\tau_g}{\beta}\right) \nn \\
&\quad\quad\quad\quad\quad\quad\quad\quad\quad- \min_{\lambda_0\ge0,\nu} \sup_{t_0\ge0} \left[\frac{\alpha}{\tau_h}\left\{\frac{\beta^2}{2\delta} + \lambda_0 \left(\frac{\gamma_0\tau_h}{\alpha} -t_0\right)  + \frac{\nu^2}{2}
-f\left(\beta,\frac{\tau_h\theta}{\alpha} + {\nu};t_0\right)  \right\} +\frac{\alpha\tau_h}{2} \right]
\end{align}
Note that the above objective is linear in $\lambda_0$. Clearly, the optimal value $t_0^*$ should satisfy $t_0^*\le \frac{\gamma_0\tau_h}{\alpha}$; otherwise $\lambda_0^* = \infty$ which makes the objective value unbounded. For $t_0\le \frac{\gamma_0\tau_h}{\alpha}$, we have $\lambda_0^* = 0$. Therefore, $t_0$ only appears in the term $f(\beta,\frac{\tau_h\theta}{\alpha} + {\nu};t_0)$.  Given that $f(c_0,c_1;t_0)$ is decreasing in $t_0$, we have $t_0^* = \frac{\gamma_0\tau_h}{\alpha}$. 

We substitute for $t_0^*$ in the AO problem to obtain
\begin{align}
&\max_{0\le\beta, \tau_h}\;\; \min_{\theta, 0\le \alpha, \gamma_0, \tau_g}\;\; D_{{\rm ns}}(\alpha, \gamma_0, \theta, \tau_g, \beta,\tau_h)\nn\\
& D_{{\rm ns}}(\alpha, \gamma_0, \theta, \tau_g, \beta,\tau_h) = \frac{\beta \tau_g}{2} +L\left(\sqrt{\alpha^2+\theta^2},\sigma_{M,2} \theta-\eps_0 \gamma_0,\frac{\tau_g}{\beta}\right) \nn \\
&\quad\quad\quad\quad\quad\quad\quad\quad\quad- \min_{\nu}  \left[\frac{\alpha}{\tau_h}\left\{\frac{\beta^2}{2\delta}  + \frac{\nu^2}{2}
-f\left(\beta,\frac{\tau_h\theta}{\alpha} + {\nu};\frac{\gamma_0\tau_h}{\alpha}\right)  \right\} +\frac{\alpha\tau_h}{2} \right]\,.
\end{align}
This completes the proof of part (c).
\end{proof}
\section{Proof of technical lemmas}
\subsection{Proof of Lemma~\ref{lem:SR-AR}}\label{proof:SR-AR}
By definition, we have
\begin{align}
\SR(\hth) &:= \E[\ind(\hat{y} =  y)]  = \prob(y \<\x,\hth\> > 0)\nn\\
&= \prob\left(y\<y\vct{\mu}+ \Sigma^{1/2} \bz, \hth\> > 0 \right)\nn\\
&= \prob\left(\<\vct{\mu}+ \Sigma^{1/2} \bz, \hth\> > 0 \right)\nn\\
&=  \prob\left(\<\vct{\mu},\hth\>+ \twonorm{\Sigma^{1/2}\hth} Z > 0 \right)\nn\\
&= \Phi\left(\frac{\<\vct{\mu},\hth\>}{\twonorm{\mtx{\Sigma}^{1/2}\hth}}\right)\,,\label{eq:SR-char}
\end{align}
where $\z\sim\normal(0,\mtx{I}_d)$ and $Z\sim\normal(0,1)$. 

Likewise for the adversarial risk we have
\begin{align}
\AR(\hth) &: = \E\Bigg[\min_{\pnorm{\bdelta}\le \eps} \ind(y \<\x+\bdelta,\hth\> \ge 0) \Bigg]\nn\\
&\stackrel{(a)}{=} \E\Bigg[\ind(y \<\x,\hth\> - \eps \qnorm{\hth} \ge 0)\Bigg]\nn\\
&= \prob\left(y \<\x,\hth\> - \eps \qnorm{\hth} \ge 0\right)\nn\\
&= \prob\left(y\<y\vct{\mu}+ \mtx{\Sigma}^{1/2} \bz, \hth\> - \eps \qnorm{\hth} \ge 0 \right)\nn\\
&\stackrel{(b)}{=}  \prob\left(\<\vct{\mu},\hth\>+ \twonorm{\mtx{\Sigma}^{1/2}\hth} Z -\eps \qnorm{\hth} \ge 0 \right)\nn\\
&=  \Phi\left(\frac{\<\vct{\mu},\hth\>- \eps \qnorm{\hth}}{\twonorm{\mtx{\Sigma}^{1/2}\hth}}\right)\,, \label{eq:AR-char}
\end{align}
where $(a)$ we used that $\<\delta,\hth\> \ge -\pnorm{\bdelta} \qnorm{\hth} \ge - \eps\qnorm{\hth}$, using H\"older inequality (with $\frac{1}{p} + \frac{1}{q} = 1$) and that $\pnorm{\bdelta}\le \eps$, with equality achieving for some $\bdelta$ in this set. In $(b)$, we used the symmetry of Gaussian distribution.


\subsection{Proof of Lemma~\ref{lem:justify}}\label{proof:lem:justify}
We first note that
\begin{align}\label{eq:htz}
\min_{\bz\in\cS(\alpha,\theta,\eps_0,\vct{\mu})}\;\; -\frac{1}{\sqrt{n}}  \vct{h}^T\bz=-\frac{1}{\sqrt{n/d}}\sup_{\bz\in\cS(\alpha,\theta,\eps_0,\vct{\mu}) }\;\; \frac{1}{\sqrt{d}}  \vct{h}^T\bz\rightarrow -\frac{1}{\sqrt{\delta}} \omega(\alpha,\theta,\eps_0)\,, 
\end{align}
in probability, using the fact that $\vct{h}/\sqrt{d}$ is asymptotically uniform on the unit sphere, and for $\mathcal{S}\in\mathbb{S}^{d-1}$ the function $f(\vct{u})=\sup_{\vct{z}\in\mathcal{S}} \vct{z}^T\vct{u}$ is Lipschitz. Therefore, using the concentration of Lipschitz functions on the sphere (see e.g. \cite[Theorem 5.2.2]{vershynin2018high}), $f(\vct{u})$ concentrates around its mean $\E f(\vct{u})= \omega_s(\cS(\alpha,\theta,\eps_0,\vct{\mu}))$. More precisely, 
\begin{align*}
\prob\left\{\Big|\sup_{\vct{z}\in\mathcal{S}} \frac{1}{\sqrt{d}}\vct{h}^T\bz - \omega_s(\cS(\alpha,\theta,\eps_0,\vct{\mu})) \Big| \right\}\le 2e^{-cdt^2}\,,
\end{align*}
for an absolute constant $c>0$ and for every $t\ge0$. Therefore, by invoking the assumption on the convergence of spherical width, cf. Assumption~\ref{ass:omegas}, we arrive at
\begin{align*}
\lim_{d\to\infty}\prob\left\{\Big|\sup_{\vct{z}\in\mathcal{S}} \frac{1}{\sqrt{d}}\vct{h}^T\bz - \omega(\alpha,\theta,\eps_0) \Big| \ge \eta \right\} =0\,,\quad \forall \eta>0\,.
\end{align*}
Therefore, $\sup_{\bz\in\cS(\alpha,\theta,\eps_0) }\;\; \frac{1}{\sqrt{d}}  \vct{h}^T\bz\rightarrow  \omega(\alpha,\theta,\eps_0)$, in probability.


To evaluate the left hand side, we form the Lagrangian corresponding to the set $\cS$. Let $\tmu: = \frac{\vct{\mu}}{\twonorm{\vct{\mu}}}$ and consider the change of variable $\vct{u}:= \vct{z} +\theta \tmu$. We then have
\begin{align}
&\min_{\bz\in\cS(\alpha,\theta,\eps_0,\vct{\mu})}\;\; -\frac{1}{\sqrt{n}}  \vct{h}^T\bz\nn\\
&= \sup_{\lambda,\eta\ge0, \nu} \min_{\bu} \;\; -\frac{1}{\sqrt{n}}  \vct{h}^T (\bu -\theta \tmu )  +\lambda\left(\qnorm{\vct{u}}^q-\left(\frac{1}{\eps_0 \pnorm{\vct{\mu}}}\right)^q\right)+\frac{\eta}{2}\left(\twonorm{\bu- \theta\tmu}^2- \alpha^2 \right) +\nu \tmu^T \left(\vct{u} - \theta \tmu\right)\nn\\
&=\sup_{\lambda,\eta\ge0, \nu} \min_{\bu} \;\; \frac{\eta}{2}\twonorm{\bu -\theta \tmu+ \frac{\nu}{\eta}\tmu -\frac{\vct{h}}{\eta\sqrt{n}} }^2 + \lambda \qnorm{\vct{u}}^q - \frac{1}{2\eta}\twonorm{\nu\tmu - \frac{\vct{h}}{\sqrt{n}}}^2 - \frac{\eta}{2}\alpha^2 - \frac{\lambda}{(\eps_0\pnorm{\vct{\mu}})^q}\nn\\
&=\sup_{\lambda,\eta\ge0, \nu} \min_{\bu} \;\; \eta\left[\frac{1}{2}\twonorm{\bu + \left(\frac{\nu}{\eta}-\theta\right)\tmu -\frac{\vct{h}}{\eta\sqrt{n}} }^2 + \frac{\lambda}{\eta} \qnorm{\vct{u}}^q\right] - \frac{1}{2\eta}\twonorm{\nu\tmu - \frac{\vct{h}}{\sqrt{n}}}^2 - \frac{\eta}{2}\alpha^2 - \frac{\lambda}{(\eps_0\pnorm{\vct{\mu}})^q}\nn\\
&=\sup_{\lambda,\eta\ge0, \nu}  \;\; \eta \sum_{i=1}^d J_q\left(\frac{{h_i}}{\eta\sqrt{n}} - \left(\frac{\nu}{\eta}-\theta\right)\tilde{\mu}_i;\frac{\lambda}{\eta} \right)  - \frac{1}{2\eta}\twonorm{\nu\tmu - \frac{\vct{h}}{\sqrt{n}}}^2 - \frac{\eta}{2}\alpha^2 - \frac{\lambda}{(\eps_0\pnorm{\vct{\mu}})^q}\nn\\
&=\sup_{\lambda,\eta\ge0, \nu}  \;\; \eta  \sum_{i=1}^d J_q\left(\frac{{h_i}}{\eta\sqrt{n}} - \left(\frac{\nu}{\eta}-\theta\right)\tilde{\mu}_i;\frac{\lambda}{\eta} \right)  - \frac{\nu^2}{2\eta}-\frac{1}{2\eta n}\twonorm{\vct{h}}^2 + \frac{\nu}{\eta\sqrt{n}} \tmu^T\vct{h} - \frac{\eta}{2}\alpha^2 - \frac{\lambda}{(\eps_0\pnorm{\vct{\mu}})^q}\label{J-lem-1}
\end{align}

Recall that $\vct{h}\sim\normal(0,\mtx{I}_d)$. As $n\to \infty$ and $n/d\to\delta$, we have
\begin{align}
\frac{1}{n}\twonorm{\vct{h}}^2 \to \frac{1}{\delta},\quad \frac{1}{\sqrt{n}} \tmu^T\vct{h}\to 0\,,
\end{align}
in probability. In addition, 
\begin{align}
\pnorm{\vct{\mu}} \to \sigma_{M,p} d^{\frac{1}{p}-\frac{1}{2}} = \sigma_{M,p} d^{\frac{1}{2}-\frac{1}{q}}\,,
\end{align}
in probability. Using the identity $J_q(x;\lambda) = c^2 J_q(x/c;\lambda c^{q-2})$ and letting $\lambda_0:= \lambda d^{1-\frac{q}{2}}$ we have
\begin{align}
\sum_{i=1}^d J_q\left(\frac{{h_i}}{\eta\sqrt{n}} - \left(\frac{\nu}{\eta}-\theta\right)\tilde{\mu}_i;\frac{\lambda}{\eta} \right)
&= \frac{1}{d} \sum_{i=1}^d J_q\left(\frac{\sqrt{d}{h_i}}{\eta\sqrt{n}} - \left(\frac{\nu}{\eta}-\theta\right)\frac{\sqrt{d}\mu_i}{\sigma_{M,2}};\frac{\lambda}{\eta} d^{1-q/2} \right)\nn\\
&= \frac{1}{d} \sum_{i=1}^d J_q\left(\frac{h_i}{\eta\sqrt{\delta}} - \left(\frac{\nu}{\eta}-\theta\right)\frac{\sqrt{d}\mu_i}{\sigma_{M,2}};\frac{\lambda_0}{\eta}\right)\,.\nn
\end{align} 
Since $J_q(x;\lambda) \le \frac{1}{2} x^2$, the function $J_q$ is pseudo-lipschitz of order 2 and by an application of \cite[Lemma 5]{bayati2011dynamics}, we have
\begin{align}
\lim_{n\to \infty} \frac{1}{d} \sum_{i=1}^d J_q\left(\frac{h_i}{\eta\sqrt{\delta}} - \left(\frac{\nu}{\eta}-\theta\right)\frac{\sqrt{d}\mu_i}{\sigma_{M,2}};\frac{\lambda_0}{\eta}\right) = \E\left[J_q\left(\frac{h}{\eta\sqrt{\delta}} - \left(\frac{\nu}{\eta}-\theta\right)\frac{M}{\sigma_{M,2}};\frac{\lambda_0}{\eta}\right) \right]\,,
\end{align}
almost surely, where the expectation in the last line is taken with respect to the independent random variables $h\sim\normal(0,1)$ and $M\sim\prob_M$.

Using the above limits, we see that the objective function~\eqref{J-lem-1} converges pointwise to the following function:
\begin{align}\label{eq:J-lem-2}
\eta \E\left[J_q\left(\frac{h}{\eta\sqrt{\delta}} - \left(\frac{\nu}{\eta}-\theta\right)\frac{M}{\sigma_{M,2}};\frac{\lambda_0}{\eta}\right) \right]  - \frac{\nu^2}{2\eta}-\frac{1}{2\eta \delta} - \frac{\eta}{2}\alpha^2 - {\lambda_0}(\eps_0\sigma_{M,p})^{-q} 
\end{align} 
Note that \eqref{J-lem-1} is the dual optimization and hence is a concave problem. We apply the convexity lemma~\cite[Lemma B.2]{thrampoulidis2018precise} to conclude that the objective value in \eqref{J-lem-1} also converges to the supremum of function~\eqref{eq:J-lem-2} over $\lambda_0,\eta\ge0, \nu$. 

Using this observation along with Equation~\eqref{eq:htz} and \eqref{J-lem-1} we obtain
\begin{align}
&\omega(\alpha, \theta, \eps_0) \nn\\
&= \lim_{n\to\infty} \omega_s\left(\cS(\alpha, \theta, \eps_0, \vct{\mu})\right)\nn\\
&= - \sup_{\lambda_0,\eta\ge0, \nu}  \;\; \eta\sqrt{\delta} \E\left[J_q\left(\frac{h}{\eta\sqrt{\delta}} - \left(\frac{\nu}{\eta}-\theta\right)\frac{M}{\sigma_{M,2}};\frac{\lambda_0}{\eta}\right) \right]  - \sqrt{\delta} \left\{\frac{\nu^2}{2\eta}+\frac{1}{2\eta \delta} + \frac{\eta}{2}\alpha^2 + {\lambda_0}(\eps_0\sigma_{M,p})^{-q}\right\}\nn\\
&= \min_{\lambda_0,\eta\ge0, \nu}  \;\; \sqrt{\delta} \left\{\frac{\nu^2}{2\eta}+\frac{1}{2\eta \delta} + \frac{\eta}{2}\alpha^2 + {\lambda_0}(\eps_0\sigma_{M,p})^{-q}\right\}-\eta\sqrt{\delta} \E\left[J_q\left(\frac{h}{\eta\sqrt{\delta}} - \left(\frac{\nu}{\eta}-\theta\right)\frac{M}{\sigma_{M,2}};\frac{\lambda_0}{\eta}\right) \right] \,, 
\end{align}
which completes the proof.
\subsection{Proofs that the minimization and maximization primal problems can be restricted to a compact set} \label{sec:artifical}
In this section we demonstrate how the minimization and maximization problems can be restricted to compacts sets. 
\subsubsection{Bounded domains in optimization~\eqref{eq:Lag-sep}}\label{setres}
We start with the restriction on $\bth$. Note that one of the claims of Theorem~\ref{thm:isotropic-separable-ansio}, part (b), is to show that $\twonorm{\tth^\eps} \to \alpha_*$ as $n\to \infty$, in probability, for some $\alpha_*$ by the solution of minimax problem~\eqref{eq:sep-AO92}. We define $\cS_{\bth} = \{\bth:\; \twonorm{\bth}\le K_{\alpha}\}$ with $K_{\alpha} = \alpha_* +\xi$ for a constant $\xi>0$. We start by the ansatz that $\bth\in \cS_{\bth}$ and add this `artificial constraint' in the minimax optimization. In addition, by the stationary condition for $\gamma$ in optimization~\eqref{eq:Lag-sep} we have
\[
\frac{1}{n} \eps \bu^T\mathbf{1} = 2\lambda\,.
\]
Therefore 
$$\frac{1}{n} \bu^T\mathbf{1} = \frac{2\lambda}{\eps} = \frac{2\lambda}{\eps_0 \pnorm{\vct{\mu}}}\,.$$
Let $\lambda_0:= \frac{\lambda}{\pnorm{\vct{\mu}}}$. Assuming the ansatz that $\lambda_0 = O(1)$, we also use this `artificial constraint' in the minimax optimization.

With these compact constraints in place, we then deploy the CGMT framework to prove Theorem~\ref{thm:isotropic-separable-ansio}. This theorem implies that $\twonorm{\tth^\eps} \to \alpha_*$ as $n\to \infty$ and so our initial ansatz on the boundedness of $\bth$ is verified. Further, as it can be seen from the proof of Theorem~\ref{thm:isotropic-separable-ansio} (see the line following Equation~\eqref{eq:sep-AO7}), we have $\lambda_0=\frac{\lambda}{\pnorm{\vct{\mu}}} \to \lambda_{0*}$ as $n\to \infty$, in probability, for some $\lambda_{0*}$ that is determined  by the solution of minimax problem~\eqref{eq:sep-AO92}. This also verifies our ansatz that $\lambda_0 = O(1)$, which in turn implies that $\bu\in \cS_{\bu} = \{\bu:\; 0\le u_i, \; \frac{1}{n}\mathbf{1}^T\bu \le K_{\bu}\}$ for some sufficiently large constant $K_{\bu}>0$.

\subsubsection{Bounded domains in optimization~\eqref{eq:min-max-1}}
Similar to previous subsection, we start by the ansatz that $\bth\in \cS_{\bth}$ where $\cS_{\bth} = \{\bth:\; \twonorm{\bth}\le K_{\alpha}\}$ with $K_{\alpha} = \alpha_* +\xi$ for a constant $\xi>0$, and add this `artificial constraint' in the minimax optimization~\eqref{eq:min-max-1}. Also by stationarity condition for $\bv$ in \eqref{eq:min-max-1} we have
$u_i =  \ell' \left(v_i - \eps \qnorm{\bth}\right)$
and hence $|u_i| \le K_{\bu}$
for some large enough constant $K_{\bu} > 0$, using our assumption on the loss function $\ell$.

\subsection{Proof of Lemma~\ref{lem:tbth}}\label{proof:lem:tbth}
First note that by Cauchy–Schwarz inequality we have
\[
 \<\pproj_{\vct{\mu}}\vct{r},\left(\pproj_{\vct{\mu}}\mtx{\Sigma}^{1/2}\pproj_{\vct{\mu}}\right)\tbth\> \ge - \twonorm{\pproj_{\vct{\mu}}\vct{r}}
 \twonorm{\left(\pproj_{\vct{\mu}}\mtx{\Sigma}^{1/2}\pproj_{\vct{\mu}}\right)\tbth} = -\alpha \twonorm{\pproj_{\vct{\mu}}\vct{r}}.
\]
To achieve equality, note that similar to~\eqref{eq:aux1} we have
\[
\left(\pproj_{\vct{\mu}}\mtx{\Sigma}^{1/2}\pproj_{\vct{\mu}}\right) \left(\pproj_{\vct{\mu}}\mtx{\Sigma}^{-1/2}\pproj_{\vct{\mu}}\right) \vct{r} = \pproj_{\vct{\mu}} \vct{r}\,.
\]  
Therefore equality is achieved by choosing 
$\bth = \lambda \mtx{\Sigma}^{-1/2} \pproj_{\vct{\mu}} \vct{r}$ with $\lambda = \frac{\alpha}{\twonorm{\pproj_{\vct{\mu}} \vct{r}}}$.

\subsection{Proof of Lemma~\ref{lem:conj-L}}\label{proof:lem:conj-L}
By definition of the conjugate function we have
\begin{align*}
\widetilde{\ell}(\vct{v},\vct{w})=&\sup_{\vct{\theta}}\text{ }\vct{w}^T\bth-\ell(\vct{v},\vct{\theta})\\
=&\sup_{\vct{\theta}}\text{ }\vct{w}^T\bth-\frac{1}{n} \sum_{i=1}^n  \ell \left(v_i - \eps \qnorm{\bth}\right) 
\end{align*}
Now assume $\vct{\theta}=\gamma \vct{u}$ with $\qnorm{\vct{u}}=1$. We thus have,
\begin{align*}
\widetilde{\ell}(\vct{v},\vct{w})=&\sup_{\vct{u}: \qnorm{\vct{u}}=1,\gamma}\text{ }\gamma\vct{w}^T\vct{u}-\frac{1}{n} \sum_{i=1}^n  \ell \left(v_i - \eps \gamma \right) \\
=&\sup_{\gamma\ge 0}\quad \gamma\left(\sup_{\vct{u}: \qnorm{\vct{u}}=1}\text{ }\vct{w}^T\vct{u}\right)-\frac{1}{n} \sum_{i=1}^n  \ell \left(v_i - \eps \gamma \right)\\
=&\sup_{\gamma\ge 0}\quad \gamma \pnorm{\vct{w}}-\frac{1}{n} \sum_{i=1}^n  \ell \left(v_i - \eps \gamma \right)\,.
\end{align*}

\subsection{Proof of Lemma~\ref{lem:f-conjugate}}\label{proof:lem:f-conjugate}
By definition
\begin{align*}
f^*(\vct{u})&:= \sup_{\tw}\;\; \<\vct{u},\tw\>  - f(\tw)\nn\\
&= \sup_{\tw}\;\; \<\vct{u},\tw\>  + \<\tw, \mtx{\Sigma}^{1/2}\tmu\> \frac{\theta \tau_h}{\alpha} - \frac{\tau_h}{\alpha}\frac{\gamma_0}{\pnorm{\vct{\mu}}} \pnorm{\mtx{\Sigma}^{1/2}\tw} \nn\\
&= \sup_{\tw}\;\; \left\<\vct{u}+ \mtx{\Sigma}^{1/2}\tmu\frac{\theta \tau_h}{\alpha},\tw\right\>  - \frac{\tau_h}{\alpha} \frac{\gamma_0}{\pnorm{\vct{\mu}}}\pnorm{\mtx{\Sigma}^{1/2}\tw} \nn\\
&= \sup_{\tw}\;\; \left\<\mtx{\Sigma}^{-1/2}\vct{u}+ \tmu\frac{\theta \tau_h}{\alpha},\mtx{\Sigma}^{1/2}\tw\right\>  - \frac{\tau_h}{\alpha} \frac{\gamma_0}{\pnorm{\vct{\mu}}}\pnorm{\mtx{\Sigma}^{1/2}\tw} 
\end{align*}
By H\"older's inequality, 
\[
\left\<\mtx{\Sigma}^{-1/2}\vct{u}+ \tmu\frac{\theta \tau_h}{\alpha},\mtx{\Sigma}^{1/2}\tw\right\> \le \qnorm{\mtx{\Sigma}^{-1/2}\vct{u}+ \tmu\frac{\theta \tau_h}{\alpha}} \pnorm{\mtx{\Sigma}^{1/2}\tw}
\]
Therefore, if $\vct{u}\in S$ then the supremum is achieved by choosing $\tw = 0$. If $\vct{u}\notin S$, by scaling $\tw$ the supremum would be $+\infty$.
\subsection{Proof of Lemma~\ref{lem:proj}}\label{proof:lem:proj}
Fix arbitrary $\vct{u}$. By definition,
\begin{align}
\proj_{\cal{B}}(\vct{u}) &:= \arg\min_{\vct{z}\in \cal{B}} \twonorm{\vct{u} -\vct{z}} \nn\\
&= \arg\min_{\vct{z}\in \cal{B}} \twonorm{\pproj_{\vct{\mu}}(\vct{u} -\vct{z})}^2 + \twonorm{\proj_{\vct{\mu}}(\vct{u} -\vct{z})}^2 \nn\\
&\stackrel{(a)}{=} \arg\min_{\vct{z}\in \cal{B}} \twonorm{\pproj_{\vct{\mu}}\vct{u} -\vct{z}}^2 + \twonorm{\proj_{\vct{\mu}}\vct{u}}^2\nn\\
&= \arg\min_{\vct{z}\in \cal{B}} \twonorm{\pproj_{\vct{\mu}}\vct{u} -\vct{z}}^2 \nn\\
&= \proj_{\cal{B}}\pproj_{\vct{\mu}}\vct{u}\,,
\end{align}
where step $(a)$ follows from that fact that $\vct{z}\in \cal{B}$ and hence $\vct{z} = \pproj_{\vct{\mu}} \vct{z}$.
\subsection{Proof of Lemma~\ref{lem:proj-main}}\label{proof:lem:proj-main}
By definition of the set $\cal{B}$, the value of $\twonorm{\proj_{\cal{B}}\left(\vct{h}\right) - \vct{h}}^2$ is given by the optimal objective value of the following optimization:
\begin{align}
&\text{minimize}_{\vct{z}} \;\;\;\;\twonorm{\vct{z}-\vct{h}} \nn\\
&\text{subject to} \;\;\;\; \qnorm{\mtx{\Sigma}^{-1/2}\vct{z} -  \frac{\tau_h\theta}{\alpha}\tmu}\le \frac{\tau_h}{\alpha}\frac{\gamma_0}{\pnorm{\vct{\mu}}}\,, \quad \tmu^T\vct{z} = 0\,.
\label{eq:OPT1}
\end{align}
By the change of variable $\vct{u}:= \mtx{\Sigma}^{-1/2}\vct{z} -  \frac{\tau_h\theta}{\alpha}  \tmu$ and forming the Lagrangian, the optimal value of \eqref{eq:OPT1} is equal to
the optimal value of the following problem:
\begin{align}\label{eq:OPT2}
\sup_{\lambda\ge 0,\nu}\text{ }\min_{\vct{u}} \;\;\frac{1}{2}\twonorm{\mtx{\Sigma}^{1/2}\left(\vct{u}+ \frac{\tau_h\theta}{\alpha} \tmu \right) -\vct{h}}^2+\lambda\left(\qnorm{\vct{u}}^q-\left(\frac{\tau_h}{\alpha}\frac{\gamma_0}{\pnorm{\vct{\mu}}}\right)^q\right)+\nu \tmu^T \mtx{\Sigma}^{1/2}\left(\vct{u} + \frac{\tau_h \theta}{\alpha}\tmu\right)\,.
\end{align}
Rearranging the terms we get the next alternative representation 
\begin{align}\label{eq:OPT3}
\sup_{\lambda\ge 0,\nu}\text{ }\min_{\vct{u}} \;\;\frac{1}{2}\twonorm{\mtx{\Sigma}^{1/2}\left(\vct{u}+ \frac{\tau_h\theta}{\alpha} \tmu \right) -\vct{h}+\nu \frac{\vct{\mu}}{\twonorm{\vct{\mu}}}}^2+\lambda\left(\qnorm{\vct{u}}^q-\left(\frac{\tau_h}{\alpha}\frac{\gamma_0}{\pnorm{\vct{\mu}}}\right)^q\right)+\nu \tmu^T\vct{h} -  \frac{\nu^2}{2} \,,
\end{align}
Now adopting the notation $\Sigmanorm{\vct{v}}^2:=\vct{v}^T\mtx{\Sigma}\vct{v}$ and invoking the assumption $\mtx{\Sigma}^{1/2}\tmu = a\tmu$, we rewrite the optimization as follows:
\begin{align}\label{eq:OPT4}
\sup_{\lambda\ge 0,\nu}\text{ }\min_{\vct{u}} \;\;\frac{1}{2}\Sigmanorm{\vct{u}+ \left(\frac{\tau_h\theta}{\alpha} +\frac{\nu}{a} \right) \tmu -\mtx{\Sigma}^{-1/2}\vct{h}}^2+\lambda\left(\qnorm{\vct{u}}^q-\left(\frac{\tau_h}{\alpha}\frac{\gamma_0}{\pnorm{\vct{\mu}}}\right)^q\right)+\nu \tmu^T\vct{h} -  \frac{\nu^2}{2} \,,
\end{align}
Rearranging the terms further we obtain
\begin{align}\label{eq:OPT4}
&\sup_{\lambda\ge 0,\nu}\text{ } \;  \min_{\vct{u}} \left[\frac{1}{2}\Sigmanorm{\vct{u}+ \left(\frac{\tau_h\theta}{\alpha} +\frac{\nu}{a} \right) \tmu -\mtx{\Sigma}^{-1/2}\vct{h}}^2+ \lambda\qnorm{\vct{u}}^q\right] - \lambda \left(\frac{\tau_h}{\alpha} \frac{\gamma_0}{\pnorm{\vct{\mu}}}\right)^q +\nu \tmu^T\vct{h} - \frac{\nu^2}{2} \nn\\
&= \sup_{\lambda\ge 0,\nu}\text{ } \; e_{q,\mtx{\Sigma}}\left(\mtx{\Sigma}^{-1/2}\vct{h}- \left(\frac{\tau_h\theta}{\alpha}+\frac{\nu}{a}\right)\tmu;\lambda\right)
- \lambda \left(\frac{\tau_h}{\alpha} \frac{\gamma_0}{\pnorm{\vct{\mu}}}\right)^q +\nu \tmu^T\vct{h} -  \frac{\nu^2}{2} \,.
\end{align}
This concludes the proof.

\subsection{Proof of Lemma~\ref{lem:q=2}}\label{proof:lem:q=2}
We recall the definition of weighted Moreau envelope
\begin{align}\label{eq:e2-Sig}
e_{2,\mtx{\Sigma}}(\vct{x};\lambda) = \min_{\vct{v}} \text{} \frac{1}{2} \Sigmanorm{\vct{x}-\vct{v}}^2
+ \lambda \twonorm{\vct{v}}^2\,.
\end{align}
Setting derivative to zero we get
\[
- \mtx{\Sigma}(\vct{x}-\vct{v}^*) + 2\lambda \vct{v}^* = 0\,,
\]
which implies that $\vct{v}_* = (\mtx{\Sigma} +2\lambda\mtx{I})^{-1}\mtx{\Sigma}\vct{x}$. Now consider a singular value decomposition $\mtx{\Sigma} = \mtx{U}\mtx{S}\mtx{U}^T$. Then,
$\vct{v}_* = \mtx{U}(\mtx{S}+2\lambda\mtx{I})^{-1}\mtx{S} \mtx{U}^T\vct{x}$. Substituting for $\vct{v}_*$ in \eqref{eq:e2-Sig} we obtain
\begin{align*}
e_{2,\mtx{\Sigma}}(\vct{x};\lambda) &= 2\lambda^2 \twonorm{\mtx{U} (\mtx{S}+2\lambda\mtx{I})^{-1}
\mtx{S}^{1/2} \mtx{U}^T\vct{x}}^2 +\lambda \twonorm{\mtx{U}(\mtx{S}+2\lambda\mtx{I})^{-1}\mtx{S} \mtx{U}^T\vct{x}}^2\nn\\
&= \lambda\vct{x}^T\mtx{U}^T (\mtx{S}+2\lambda\mtx{I})^{-1}\mtx{S}\mtx{U}\vct{x}\nn\\
&= \lambda \twonorm{\mtx{U}(\mtx{S}+2\lambda\mtx{I})^{-1/2}\mtx{S}^{1/2}\mtx{U}\vct{x}}^2\nn\\
&= \lambda\twonorm{(\mtx{\Sigma} +2\lambda\mtx{I})^{-1/2}\mtx{\Sigma}^{1/2}\vct{x}}^2
\end{align*}
which yields the desired result.

\section{Proposition~3.2 (an extended statement)}\label{app:pro:SVM}
Consider the adversarial training loss
\[
\cL(\bth) := \frac{1}{n} \sum_{i=1}^n  \ell\left(y_i \<\x_i,\bth\> - \eps \qnorm{\bth}\right)\,,
\]
where the loss $\ell(t)$ can be expressed as $\ell(t) = e^{-f(q)}$ obeying the following technical assumptions:
\begin{itemize}
\item $f:\reals\to\reals$ is $C^2$-smooth.
\item $f'(q)>0$ for all $q\in\reals$.
\item There exists $b_f\ge0$ such that $qf'(q)$ is non-decreasing for $q\in (b_f,\infty)$ and $qf'(q)\to\infty$ as $q\to\infty$.
\item Let $g: [f(b_f),\infty) \to [b_f,\infty)$ be the inverse function of $f$ on the domain $[b_f,\infty)$. There exists $p\ge0$ such that
for all $x>f(b_f)$, $y>b_f$,
\[
\left|\frac{g''(x)}{g'(x)}\right|\le \frac{p}{x}, \quad \left|\frac{f''(y)}{f'(y)}\right| \le \frac{p}{y} \,.
\]
\end{itemize}
(It can be verified that the above assumptions are satisfied by exponential loss and logistic loss.) Then, the gradient descent iterates 
\begin{align*}
\bth_{\tau+1}=\bth_{\tau}-\mu\nabla\mathcal{L}(\bth_{\tau})
\end{align*}
with a sufficiently small step size $\mu$ obey
\begin{align}
\lim_{t\to\infty} \twonorm{\frac{\bth_t}{\twonorm{\bth_t}} - \frac{\tth^\eps}{\twonorm{\tth^\eps}}} = 0\,,
\end{align}
where $\tth^\eps$ is the solution to the following max-margin problem
\begin{align}
\tth^\eps = &\arg\min_{\bth\in\reals^d} \quad  \twonorm{\bth}^2\nn  \\
&{\rm subject}\;{\rm to}\;\; y_i\<\vct{x}_i,\vct{\theta}\> - \eps \qnorm{\bth} \ge1\,.\label{eq:MM}
\end{align}  

%

\end{document}